\documentclass[twoside]{article}

\usepackage[accepted]{aistats2026}
\usepackage[round]{natbib}

\usepackage[american]{babel}
\usepackage{natbib}
\usepackage{mathtools}
\usepackage{booktabs}
\usepackage{tikz}

\usepackage[T1]{fontenc}    
\usepackage{hyperref}       
\usepackage{url}            
\usepackage{amsfonts}       
\usepackage{nicefrac}       
\usepackage{microtype}      
\usepackage{svg}
\usepackage{amsmath}
\usepackage{multirow}
\usepackage{centernot}
\usepackage{amsthm}
\usepackage{thmtools}
\usepackage{thm-restate}
\usepackage{cleveref}
\usepackage{dirtytalk}
\usepackage{minted}
\usepackage{amssymb}

\usepackage{pgfplots}
\DeclareUnicodeCharacter{2212}{−}
\usepgfplotslibrary{groupplots,dateplot}
\usetikzlibrary{patterns,shapes.arrows}
\pgfplotsset{compat=newest}

\newtheorem{theorem}{Theorem}[section]
\newtheorem{corollary}{Corollary}[section]
\newtheorem{lemma}{Lemma}[section]

\newtheorem{definition}{Definition}[section]
\newtheorem{example}{Example}[section]

\usepackage[noend]{algpseudocode}
\usepackage{algorithm}

\usepackage{xcolor}
\usepackage[dvipsnames]{xcolor}
\usepackage{hyperref}
\usepackage{graphicx}
\usepackage{subcaption}
\usepackage{paralist}

\usepackage{comment}
\newcommand{\Step}[1]{{\color{ForestGreen} \Statex$\vartriangleright$ #1}}

\Crefname{section}{Sec.}{Secs.}
\Crefname{appendix}{App.}{Apps.}
\Crefname{algorithm}{Alg.}{Algs.}
\Crefname{figure}{Fig.}{Figs.}
\Crefname{table}{Tab.}{Tabs.}
\Crefname{definition}{Def.}{Defs.}
\Crefname{theorem}{Thm.}{Thms.}
\Crefname{lemma}{Lem.}{Lems.}
\Crefname{corollary}{Cor.}{Cors.}
\Crefname{conjecture}{Con.}{Cons.}
\Crefname{example}{Ex.}{Exs.}

\newcommand{\indep}{\perp\kern-6pt\perp}
\newcommand{\dep}{\centernot{\perp\kern-6pt\perp}}
\newcommand{\starleft}{*\kern-5pt}
\newcommand{\starright}{\kern-5pt*}

\usepackage{acronym}
\acrodef{LOAD}{Local Optimal Adjustments Discovery}
\acrodef{AID}{adjustment identification distance}
\acrodef{ATE}{average treatment effect}
\acrodef{CI}{conditional independence}
\acrodef{CPDAG}{completed partially directed acyclic graph}
\acrodef{DAG}{directed acyclic graph}
\acrodef{KCI}{Kernel-based conditional independence}
\acrodef{LCD}{local causal discovery}
\acrodef{MB}{Markov blanket}
\acrodef{MEC}{Markov equivalence class}
\acrodef{SHD}{structural Hamming distance}
\acrodef{SNAP}{Sequential Non-Ancestor Pruning}

\begin{document}

\twocolumn[
    \aistatstitle{Local Causal Discovery for Statistically Efficient Causal Inference}

    \aistatsauthor{ Mátyás Schubert \And Tom Claassen \And  Sara Magliacane }
    \aistatsaddress{University of Amsterdam \And Radboud University Nijmegen \And Saarland University \& \\ University of Amsterdam}
]

\begin{abstract}

Causal discovery methods can identify valid adjustment sets for causal effect estimation for a pair of \emph{target variables}, even when the underlying causal graph is unknown. Global causal discovery methods focus on learning the whole causal graph and therefore enable the recovery of \emph{optimal adjustment sets}, i.e., sets with the lowest asymptotic variance, but they quickly become computationally prohibitive as the number of variables grows. Local causal discovery methods offer a more scalable alternative by focusing on the local neighborhood of the target variables, but are restricted to statistically suboptimal adjustment sets. In this work, we propose Local Optimal Adjustments Discovery (LOAD), a sound and complete causal discovery approach that combines the computational efficiency of local methods with the statistical optimality of global methods. First, LOAD identifies the causal relation between the targets and tests if the causal effect is identifiable by using only local information. If it is identifiable, it finds the possible descendants of the treatment and infers the optimal adjustment set as the parents of the outcome in a modified forbidden projection. Otherwise, it returns the locally valid parent adjustment sets. In our experiments on synthetic and realistic data LOAD outperforms global methods in scalability, while providing more accurate effect estimation than local methods.

\end{abstract}

\section{INTRODUCTION}
\label{sec:introduction}

Estimating causal effects \citep{pearl2009causality} is an essential task in science and decision making.
A popular method for causal effect estimation is covariate adjustment based on the underlying causal graph.
However, often the causal graph is unknown.
In this case we can learn it using causal discovery \citep{glymour2019review}. We can then use the learned graph to read off valid adjustment sets \citep{perkovic2018complete}.
Global causal discovery methods, i.e., algorithms that learn the full graph, are generally not scalable to large number of variables due to computational demand \citep{mokhtarian2021recursive}, making them inapplicable in many practical scenarios.

On the other hand, if we only care about estimating causal effects between a pair of target variables, then recovering the full causal graph over all variables might be unnecessary and inefficient.
Under the assumption of causal sufficiency, i.e., no unobserved confounders or selection bias, local causal discovery methods estimate the causal effect between a pair of targets by discovering only local information about the causal graph around the treatment \citep{wang2014discovering, gupta2023local} and identifying the parent adjustment set. Recent extensions consider also the causally insufficient case \citep{xie2024local,ling2025local}.
Other approaches recover only coarse-grained information on ancestral relations that are enough to identify valid adjustment sets \citep{watson2022causal, maasch2024local}.

While these methods can effectively reduce computational demands for finding valid adjustment sets, these sets can be suboptimal in terms of asymptotic variance \citep{henckel2022graphical}. Moreover, they require additional assumptions about ancestral relationships.
\citet{schubert2025snap} show that definite non-ancestors of the target variables are not required to discover the optimal adjustment set. They propose SNAP, which progressively identifies these definite non-ancestors and prunes them from the causal discovery process, improving the computational efficiency of the method.
However, SNAP still requires discovering fine-grained causal relations over the remaining, potentially large, number of variables, making it potentially inefficient.

In this paper, we combine the advantages of global and local methods, and propose \ac{LOAD}, a method to identify the optimal adjustment set using only local information around variables.
\ac{LOAD} employs local causal discovery to cheaply identify information over the local neighborhoods of the target variables and their siblings.
Then, using only this local information, it identifies the type of causal relation between the targets and determines whether the causal effect is identifiable. If it is identifiable, it then finds the possible descendants of the treatment and infers the optimal adjustment set as the parents of the outcome in a variant of the forbidden projection \citep{witte2020efficient}. Otherwise, it returns the locally valid parent adjustment sets \citep{maathuis2009estimating}.
Our contributions are:
\begin{compactitem}
    \item We develop a method to determine the identifiability of a causal effect  from local information.
    \item We propose \acf{LOAD}, a sound and complete method to identify the optimal adjustment set, using only local information around variables.
    \item We evaluate \ac{LOAD} on both simulated and realistic data, showing that it can recover high-quality adjustment sets with low computational costs.
\end{compactitem}
\section{PRELIMINARIES}
\label{sec:preliminaries}

We consider graphs $G = (\mathbf{V}, \mathbf{E})$ with nodes $\mathbf{V}$ and edges $\mathbf{E} \subseteq \mathbf{V} \times \mathbf{V}$. We denote undirected edges as $X - Y$ and directed edges as $X \to Y$.  An undirected graph contains only undirected edges, while a directed graph contains only directed edges. A mixed graph can contain both directed and undirected edges.

If two nodes are connected by an edge, then we say that they are \emph{adjacent}, and denote the set of nodes adjacent to a node $X$ as $Adj_G(X)$. 
If two nodes are connected by an undirected edge $X - Y$ then we say that they are siblings and denote the set of siblings of $X$ as $Sib_G(X)$.
If $X \to Y$, we say that $X$ is a parent of $Y$ and $Y$ is a child of $X$, and denote the set of parents and children of $X$ as $Pa_G(X)$ and $Ch_G(X)$ respectively. 
In directed graphs, the Markov blanket $MB(X)$ of $X$ consists of its parents, children and the parents of its children.

A path between two nodes $X$ and $Y$ is a sequence of distinct adjacent nodes starting with $X$ and ending with $Y$.
If every edge on a path between $X$ and $Y$ is directed towards $Y$ then it is a directed path from $X$ to $Y$.
A \ac{DAG} is a graph with only directed edges and no cycles, i.e., no directed paths from a node to itself.
If there is a directed path from $X$ to $Y$ in a DAG $G$, then we say that $X$ is an ancestor of $Y$ and $Y$ is a descendant of $X$, and we denote the set of ancestors and descendants of $X$ as $An_G(X)$ and $De_G(X)$ respectively. By convention, we consider all nodes to be their own ancestors and descendants. We extend the definitions of siblings, parents, children, Markov blanket, ancestors, and descendants to sets of nodes by taking the union of the respective sets, e.g., the parents of variables $\mathbf{S}$ are $Pa_G(\mathbf{S}) = \cup_{X \in \mathbf{S}}Pa_G(X)$.

The data generating process of an observational distribution $p$ over variables $\mathbf{V}$ can be described by a causal DAG $D$ such that if a variable $X$ has a direct causal effect on another variable $Y$ then $X\to Y$.
As standard in causal discovery, we assume that the distribution $p$ is Markov and faithful to $D$ \citep{spirtes2000causation}, which implies that every conditional independence $X \indep Y | \mathbf{S}$ in $p$ is equivalent to a d-separation relation of the same form in $D$. Additionally, we assume causal sufficiency, i.e., no latent confounders or selection bias.

Constraint-based causal discovery methods leverage these assumptions to identify causal relations from conditional independence (CI) tests. In general, we cannot fully identify the true causal graph $D$ from CI tests, as multiple DAGs imply the same set of conditional independences.
This set of graphs is called the \ac{MEC} of $D$ and we can represent it by a mixed graph, the \emph{completed partially directed acyclic graph (CPDAG)}. A CPDAG will have a directed edge $X \to Y$ if all DAGs in the MEC have this directed edge, otherwise if some DAGs have $X \to Y$ and others $Y \to X$, it will have an undirected edge $X - Y$. 

We say that $X$ is a possible ancestor of $Y$, denoted as $X \in PossAn_G(Y)$, and that $Y$ is a possible descendant of $X$, denoted as $Y \in PossDe_G(X)$ in a CPDAG $G$ when $X$ is an ancestor of $Y$ in at least one \ac{DAG} in the \ac{MEC} represented by $G$. Equivalently, $X$ is a possible ancestor of $Y$, iff there exists a \emph{possibly directed path} from $X$ to $Y$, i.e., a path that does not contain directed edges from $Y$ to $X$. If there is no such path, then $X$ is a definite non-ancestor of $Y$ in $G$. If $X$ is an ancestor of $Y$ in all \acp{DAG} in a \ac{MEC}, then $X$ is a definite ancestor $Y$.
However, this does not necessarily mean that there is a directed path from $X$ to $Y$ in the corresponding \ac{CPDAG} \citep{roumpelaki2016marginal}.
If there is a directed path from $X$ to $Y$ in the \ac{CPDAG} then we say that $X$ is an explicit ancestor of $Y$. We denote the set of explicit ancestors of $X$ as $ExplAn_G(X)$. Reversely, we define the explicit descendants of $X$ as the nodes with a directed path starting from $X$ in the CPDAG $G$.

If all \acp{DAG} in the \ac{MEC} share a valid adjustment set, then we say that the causal effect is identifiable.
Amenability provides a sufficient and necessary graphical condition for the causal effect of $X$ on $Y$ to be identifiable from the \ac{CPDAG} $G$ \citep{perkovic2020identifying}.

\begin{definition}[Amenability \citep{perkovic2015complete}]
\label{def:amenability}
    For any two distinct nodes $X$ and $Y$ in a CPDAG $G$,
    $G$ is amenable relative to $(X,Y)$ if every possibly directed path from $X$ to $Y$ starts with a directed edge out of $X$.
\end{definition}

If the causal effect of a treatment $X$ on an outcome $Y$ is identifiable, there might be multiple valid adjustment sets to estimate it.
\citet{perkovic2018complete} show that no valid adjustments contain the \emph{forbidden nodes}, $Forb_G(X, Y)= PossDe_G(PossCn_G(X,Y)) \cup X$, where $PossCn_G(X,Y)$ are the nodes on possibly directed paths from $X$ to $Y$, excluding $X$.
\citet{henckel2022graphical} and \citet{rotnitzky2020efficient} define the optimal adjustment set as the set with the lowest asymptotic variance across all valid adjustment sets, and provide graphical criteria to identify it. 

We take inspiration from \citep{witte2020efficient}, who formulate the optimal adjustment sets as the parents of the outcome in the \emph{forbidden projection} in a general class of graphs, maxPDAGs, which extends CPDAGs to settings with additional background knowledge.
In our case, we focus only on CPDAGs and consider a single treatment and outcome for which the graph is amenable, so Def.~17 in \citep{witte2020efficient} together with Prop.~19 and Lem.~20, simplifies to:
\begin{definition}[Simplified forbidden projection]
\label{def:simp_forbidden_proj}
    Let $G$ be a CPDAG with nodes $\mathbf{V}$, and let $X, Y \in \mathbf{V}$ such that $G$ is amenable relative to $(X, Y)$.    
    Define $\mathbf{F} = \text{Forb}_G(X,Y) \setminus \{X,Y\}$.
    The forbidden projection $\tilde{G}^{X,Y}$ of $G$ is a graph with nodes $\mathbf{V} \setminus \mathbf{F}$ and edges as follows. For distinct nodes $W_i, W_j \in \mathbf{V} \setminus \mathbf{F}$,
    \begin{compactenum}
        \item $\tilde{G}^{X,Y}$ contains a directed edge $W_i \to W_j$ if and only if $G$ contains a directed path $W_i \to \cdots \to W_j$ on which all non-endpoint nodes are in $\mathbf{F}$,
        \item $\tilde{G}^{X,Y}$ contains an undirected edge $W_i - W_j$ if and only if $G$ contains $W_i - W_j$.
    \end{compactenum}
\end{definition}

Based on this definition, we can now define the optimal adjustment set as the parents of the outcome in the forbidden projection, excluding the treatment.

\begin{definition}[Optimal adjustment set \citep{witte2020efficient}]
\label{def:oset}
    For treatment $X$ and outcome $Y$ in an amenable CPDAG $G$ with $\tilde{G}^{X,Y}$ as the forbidden projection,
    the optimal adjustment set $Oset_G(X,Y)$ relative to $X$ and $Y$ in $G$ is given by
    $$
    Oset_G(X,Y) = Pa_{\tilde{G}^{X,Y}}(Y) \setminus \{X\}.
    $$
\end{definition}
As we will see in Lem.~\ref{lemma:possde}, this definition is particularly useful for our method, since we can show that we can slightly modify the projection to include all possible descendants of the treatment, which we can easily estimate from local information, while also reducing the size of the projected graph.

In case the causal effect is not identifiable, we can still estimate a set of possible causal effects by considering all possible DAGs in the MEC of the discovered CPDAG.
While this is computationally infeasible for large graphs, \citet{maathuis2009estimating} showed that local information around the treatment is enough to identify the set of possible \emph{locally valid parent adjustment sets} that can be used to estimate this set of possible effects.
Intuitively, a candidate parent adjustment set of an outcome $Y$ contains its identified parents and a subset of its undirected siblings $\mathbf{S} \subseteq Sib_G(Y)$.
Then, this candidate parent set $Pa_G(Y) \cup \mathbf{S}$ is locally valid, if after orienting $S - Y$ as $S \to Y$ for all $S \in \mathbf{S}$, no new v-structures are created in $G$, which would imply additional orientations and contradict the fact that they are siblings in the CPDAG.
This is ensured when $S \in \mathbf{S}$ are adjacent to all other candidate parents $Pa_G(Y) \cup \mathbf{S}$, because then all new colliders are shielded.
\section{RELATED WORK}
\label{sec:related_work}

We focus on estimating the causal effect of a pair of target variables in a computationally and statistically efficient way, when the causal graph is unknown. A standard way to achieve this is to learn the causal graph over all available variables using standard causal discovery methods \citep{glymour2019review}, e.g., PC \citep{spirtes2000causation}, and then find the optimal adjustment set \citep{henckel2022graphical} through a graphical characterization. We call this approach \emph{global causal discovery}. The issue with this approach is that global causal discovery is an inherently computationally expensive task, and even methods that reduce the computational complexity, e.g., MARVEL \citep{mokhtarian2021recursive}, still struggle with hundreds of variables. Moreover, it is usually unnecessary to recover a complete graph only to estimate a causal effect between a pair of variables.

An alternative approach, \emph{local causal discovery}, instead estimates the graph only in the neighborhood of a target variable, which is inherently a computationally easier task, but it is limited to suboptimal adjustment sets, thus potentially providing a less accurate estimate of the causal effect between targets. 
 
Most work in local causal discovery is limited to a single target variable and the effects of its neighbors \citep{gao2015local}. The adjustment sets are also restricted to the neighborhood of the target, e.g., the locally valid parent adjustment sets. In particular, MB-by-MB \citep{wang2014discovering} employs sequential Markov blanket discovery starting from the target until the orientations of all edges adjacent to the target are determined.
Local Discovery using Eager Collider Checks (LDECC) \citep{gupta2023local} starts by identifying the Markov blanket of the target and then performs CI tests similarly to PC with additional checks to efficiently determine the relation of adjacent variables. LDECC provides complementary computational benefits to the other local causal discovery methods.
However, as we show in \Cref{sec:local_pc}, LDECC and other methods using LocalPC, e.g. \citep{xie2024local, wang2014discovering, li2025local}, are not sound, since they might misidentify a non-adjacent spouse as adjacent, due to only testing for separating sets among the neighbors. We show that this can be fixed by using the Markov Blanket to test for separating sets. Moreover, as we show in 
\Cref{sec:ldecc}, LDECC is not complete, since it may not orient all orientable edges around the treatment in some cases.

In contrast to previous local methods, Local Discovery by Partitioning (LDP) \citep{maasch2024local} discovers adjustment sets that are not restricted to the neighborhood of the target. LDP learns partitions of the nodes in terms of their relation to the target pair. While more scalable than global causal discovery methods, LDP still potentially learns sub-optimal adjustment sets and assumes that we know which target causes the other. Moreover, as we show in \Cref{sec:ldp}, when an identifiable effect between the target pair is zero, LDP might misreport the effect as not identifiable.

Two recent approaches focus on learning adjustment sets in the more general setting, when we do not know how the target variables are causally related, while learning only parts of the causal graph. The Confounder Blanket Learner (CBL) \citep{watson2022causal} learns ancestral relations and adjustment sets for the causal effect between the target variables,
but requires that all other variables are non-descendants of the targets.
SNAP \citep{schubert2025snap} does not have assumptions on the underlying causal graph and recovers the optimal adjustment set without discovering the full CPDAG by progressively identifying and removing
definite non-ancestors of the target variables. While more scalable than global causal discovery methods, SNAP still requires causal discovery over all possible ancestors of the targets, which is unnecessary for learning optimal adjustment sets and is potentially still a computational bottleneck to scaling to larger graphs. 

In this paper, we combine the best of both worlds in terms of global and local causal discovery: similar to global methods, we aim to learn optimal adjustment sets, which allow us to have statistically efficient causal effect estimation, while in the spirit of local methods, we only want to estimate the necessary causal information for this task, thus being computationally efficient.
\section{METHOD}
\label{sec:method}

We first discuss how to extend local causal discovery algorithms to the case in which we do not know the causal relations between two targets $X$ and $Y$. Then, we introduce \acf{LOAD}, a sound and complete method that identifies the optimal adjustment set for the causal effect between the pair of targets in a computationally efficient way.

\subsection{Identifying relations between targets}
\label{sec:extend}

While local causal discovery methods usually assume that we known the causal relation between a pair of targets $X$ and $Y$, in this section we show how we can extend them to the setting in which this relation is unknown. We leverage results by \citet{fang2022local} on how to test explicit and possible ancestry (and conversely, definite non-ancestry) with local information:
\begin{theorem}[Thm. 3 in \citep{fang2022local}]
\label{thm:explicit_ancestor}
    For any two distinct nodes $X$ and $Y$ in a CPDAG $G$, $X \in ExplAn_G(Y)$ iff $X \dep Y | Pa_G(X) \cup Sib_G(X)$.
\end{theorem}
\begin{theorem}[Thm. 2 in \citep{fang2022local}]
\label{thm:possible_ancestor}
    For any two distinct nodes $X$ and $Y$ in a CPDAG $G$, $X$ is a definite non-ancestor of $Y$ iff $X \indep Y | Pa_G(X)$. Otherwise, $X$ is a possible ancestor of $Y$.
\end{theorem}

The local information returned by methods such as MB-by-MB \citep{wang2014discovering} is the set of parents, children and siblings, thus we can use it to identify the relation between targets. We show an implementation in \Cref{alg:relate}, which first runs MB-by-MB on each target and then uses  \Cref{thm:explicit_ancestor} and  \Cref{thm:possible_ancestor} implemented by \Cref{alg:is_explicit_ancestor} and \Cref{alg:is_possible_ancestor} with a few optimizations to test explicit or possible ancestry using only local information. 

The output of the Algorithm are $Relation(X,Y)$ and $Relation(Y, X)$, which describe if the relation between $X$ and $Y$ is definite non-ancestral ($DefNonAn$), i.e., they do not cause each other in any DAG in the MEC, if one is an explicit ancestor of the other ($ExplAn$) or if either is a possible ancestor of the other ($PossAn$). In this last case, the possible ancestry is potentially true in both directions. We show that our method is sound and complete in recovering these relations.

\begin{algorithm}[t]
    \centering
    \caption{LocalRelate}
    \label{alg:relate}
    \begin{algorithmic}[1]
    \Require{Targets $X,Y \in \mathbf{V}$ and variables $\mathbf{V}$}
    \Ensure{Causal relation between $X$ and $Y$, $G_X$, $G_Y$}
    \State $Relation(X,Y) \gets Relation(Y,X) \gets DefNonAn$\label{line:localrelate:initialize}
    \State $G_X \gets \text{MB-by-MB}(X, \mathbf{V})$\label{line:localrelate:G_x}
    \State $G_Y \gets \text{MB-by-MB}(Y, \mathbf{V})$\label{line:localrelate:G_y}
    \If{IsExplAn($X,Y,G_X$)}\Comment{\Cref{alg:is_explicit_ancestor}}\label{line:localrelate:explan_start}
        \State $Relation(X,Y) \gets ExplAn$
    \ElsIf{IsExplAn($Y, X, G_Y$)}\Comment{\Cref{alg:is_explicit_ancestor}}
        \State $Relation(Y,X) \gets ExplAn$\label{line:localrelate:explan_end}
    \Else\Comment{Cannot determine causal relation}\label{line:localrelate:possan_start}
        \If{IsPossAn($X, Y, G_X$)}\Comment{\Cref{alg:is_possible_ancestor}}
            \State $Relation(X,Y) \gets PossAn$
        \EndIf
        \If{IsPossAn($Y, X, G_Y$)}\Comment{\Cref{alg:is_possible_ancestor}}
            \State $Relation(Y,X) \gets PossAn$
        \EndIf
    \EndIf\label{line:localrelate:possan_end}
    \State\Return{$Relation, G_X, G_Y$}
    \end{algorithmic}
\end{algorithm}

\begin{restatable}[]{lemma}{relate}
\label{lem:relate}
    \Cref{alg:relate} is sound and complete in finding definite non-ancestral, possible ancestral and explicit ancestral relations between a pair of targets.
\end{restatable}

We provide a proof in \Cref{proof:relate}. Besides identifying the causal relations, the algorithm also provides local structures around both targets that can be used to enumerate the locally valid parent adjustment sets \citep{maathuis2009estimating}, as described in \Cref{alg:locally_valid_sets}.

\subsection{Local Optimal Adjustments Discovery}
\label{sec:load}

\acf{LOAD} is a sound and complete method to identify the optimal adjustment set for a target pair using only local information.
\ac{LOAD} leverages standard local causal discovery methods to collect subgraphs $G_V$ around nodes $V$ identifying the parents, children and siblings of $V$.
In our implementation we use MB-by-MB \citep{wang2014discovering} extended with caching, as described in \Cref{alg:mb-by-mb}, but LOAD is agnostic to the local causal discovery method and we evaluate using an alternative method, CMB \citep{gao2015local} in \Cref{sec:lcd_algorithm}.
We describe the pseudocode of LOAD in \Cref{alg:load}.

LOAD starts in Step 1 by determining the causal relation between the two targets by employing \Cref{alg:relate}.
If a variable is a definite non-ancestor of the other, then its causal effect on the other is trivially identifiable as zero.
If one of the variables is a possible ancestor of the other, we prove a necessary but not sufficient condition for the causal effect to be identifiable.

\begin{restatable}[]{lemma}{identifiabilityimpliesexplan}
\label{lem:identifiability_implies_expl_an}
    For any two distinct nodes $X$ and $Y$ such that $X \in PossAn_G(Y)$ in a CPDAG $G$, the causal effect of $X$ on $Y$ is identifiable only if $X$ is an explicit ancestor of $Y$ in $G$.
\end{restatable}

We provide the proof in \Cref{proof:identifiability_implies_expl_an}.
If $X$ is an explicit ancestor of $Y$, then we call $X$ the ``treatment'' $T$ and $Y$ the ``outcome'' $O$, or viceversa if $Y$ is an explicit ancestor of $X$.
If $X$ is not an explicit ancestor of $Y$, then its causal effect on $Y$ is not identifiable and LOAD returns the locally valid parent adjustment sets for $X$.

Even if $X$ is an explicit ancestor of $Y$, its causal effect on $Y$
might still not be identifiable. We can test this by checking amenability \citep{perkovic2015complete}, but that requires checking all possibly directed paths between $X$ and $Y$, which  potentially requires knowing the whole CPDAG. Instead, we develop a test for amenability that uses only local information around $X$ and its siblings.

\begin{restatable}[]{lemma}{localamenability}
\label{lem:local_amenability}
    For any two distinct nodes $X$ and $Y$
    such that $X \in PossAn_G(Y)$ in a CPDAG $G$, $G$ is adjustment amenable relative to $(X,Y)$ iff
    $$
    \forall V \in Sib_G(X):V \indep Y | Pa_G(V) \cup \{X\}.
    $$
\end{restatable}
We provide a proof in \Cref{proof:local_amenability}.
We implement this lemma through \Cref{alg:is_amenable} and use it in Step 2 of LOAD. This test requires only knowing a local subgraph of the CPDAG, in particular the parents and adjacencies of each sibling of the treatment. We show that with \Cref{lem:identifiability_implies_expl_an}
and \Cref{lem:local_amenability}, LOAD correctly determines the identifiability of the causal effects between the targets.

\begin{restatable}[]{corollary}{loadidentifiability}
\label{lem:load_identifiability}
 LOAD is sound and complete in determining the identifiability of the causal effect between a pair of targets $X$ and $Y$.
\end{restatable}

\begin{algorithm}[t]
    \centering
    \caption{LocalAmenTest}
    \label{alg:is_amenable}
    \begin{algorithmic}[1]
        \Require Targets $X$, $Y$, Sibling $V$, Graph $G$
        \If{$V \in Adj_G(Y)$}
            \State\Return{False}
        \ElsIf{$V \indep Y | Pa_G(V) \cup \{X\}$}
            \State\Return{True}
        \Else
            \State\Return{False}
        \EndIf
    \end{algorithmic}
\end{algorithm}

We provide a proof in \Cref{proof:load_identifiability}.
Similarly to the previous step, if the causal effect of the treatment on the outcome is not identifiable, then LOAD returns the locally valid parent adjustment sets.
If the causal effect is identifiable, then LOAD continues in Step 3 to identify the possible descendants of the treatment using \Cref{alg:is_possible_ancestor}, which we use in a modified forbidden projection \citep{witte2020efficient} that instead of marginalizing the forbidden nodes out, projects over the nodes $\mathbf{V} \setminus PossDe_G(T) \cup \{T, O\}$. We show that the optimal adjustment set is still the parents of the outcome in this smaller projection, as in the original projection:
\begin{restatable}[]{lemma}{possde}
\label{lemma:possde}
    For treatment $T$ and outcome $O$ in an amenable CPDAG $G$, let 
    $\mathbf{D} = PossDe_G(T)) \setminus \{T,O\}$ and
    $G^{T,O}$ be the modified forbidden projection with nodes $\mathbf{V}\setminus \mathbf{D}$. 
    We show that $G^{T,O}$ is a CPDAG and the  optimal adjustment set $Oset_G(X,Y)$ is given by
    $$
    Oset_G(T,O) = Pa_{G^{T,O}}(O) \setminus \{T\}.
    $$
\end{restatable}
We provide a proof in App.~\ref{proof:possde}. In Step 4, LOAD performs local causal discovery on the outcome in this modified forbidden projection by running MB-by-MB only on $\mathbf{V} \setminus PossDe_G(T) \cup \{T, O\}$. The parents of the outcome identified by local discovery are the optimal adjustment set.
We show that LOAD is sound and complete in discovering optimal adjustment sets.

\begin{algorithm}[t!]
\footnotesize{
\caption{\ac{LOAD}}
\label{alg:load}
\begin{algorithmic}[1]
    \Require{Targets $X,Y \in \mathbf{V}$ and variables $\mathbf{V}$}
    \Ensure{Causal relation of $X$ and $Y$ in $Relation$, if the effects of $X$ on $Y$ and $Y$ on $X$ are identifiable in $IsIdent$, and either optimal or locally valid parent adjustment set in $AdjSet$ based on identifiability.}
    \State $IsIdent_{X \to Y} \gets \text{False}$, $IsIdent_{Y \to X} \gets \text{False}$    
    \State $AdjSets_{X \to Y} \gets \emptyset, AdjSets_{Y \to X} \gets \emptyset$
    \Step{\textbf{Step 1:} Determine causal relations between targets}
    \State $Relation, G_X, G_Y \gets$ LocalRelate($X,Y, \mathbf{V}$) \Comment{\Cref{alg:relate}}\label{line:load:relate}
    \If{$Relation(X,Y) = ExplAn$}
        \State T, O $\gets X, Y$
    \ElsIf{$Relation(Y,X) = ExplAn$}
        \State T, O $\gets Y, X$
    \Else\Comment{Cannot determine treatment and outcome}\label{line:load:no_explan_start}
        \If{$Relation(X,Y) = PossAn$}
            \State $AdjSets_{X \to Y} \gets \text{LocalValidSets}(X, Y, G_X)$
        \EndIf
        \If{$Relation(Y,X) = PossAn$}
            \State $AdjSets_{Y \to X} \gets \text{LocalValidSets}(Y, X, G_Y)$
        \EndIf
        \If{$Relation(X,Y) = DefNonAn$}
         \State $IsIdent_{X \to Y} \gets \text{True}$
         \EndIf
        \If{$Relation(Y,X) = DefNonAn$}
         \State $IsIdent_{Y \to X} \gets \text{True}$
        \EndIf
        \State\Return{$Relation, IsIdent, AdjSets$}
    \EndIf\label{line:load:no_explan_end}
    \State $IsIdent_{O \to T} \gets \text{True}$ \Comment{Zero effect in reverse dir.}
    \Step{\textbf{Step 2:} Test identifiability of treatment on outcome}
    \For{$V \in Sib_{G_T}(T)$}\label{line:load:amenability_start}
        \State $G_V \gets \text{MB-by-MB}(V, \mathbf{V})$
        \If{\textbf{not} LocalAmenTest($T,O,V,G_V$)}
            \State $AdjSets_{T \to O} \gets \text{LocalValidSets}(T, O, G_T)$
            \State\Return{$Relation, IsIdent, AdjSets$}
        \EndIf
    \EndFor\label{line:load:amenability_end}
    \State $IsIdent_{T \to O} \gets \text{True}$
    \Step{\textbf{Step 3:} Find possible descendants of treatment}
    \State $PossDe(T) \gets \{T,O\}$\label{line:load:possde_start}
    \For{$V \in \mathbf{V} \setminus \{T,O\}$}
       \If{IsPossAn($T, V, G_T$)}
            \State $PossDe(T) \gets PossDe(T) \cup \{V\}$
        \EndIf
    \EndFor\label{line:load:possde_end}
    \Step{\textbf{Step 4:} Identify $Oset$ via mod. forbidden projection}
    \State $G^{T,O}_O \gets \text{MB-by-MB}(O, \mathbf{V} \setminus PossDe(T) \cup \{T, O\})$\label{line:load:forb_proj}
    \State $Oset(T,O) \gets Pa_{G^{T,O}_O}(O) \setminus \{T\}$\label{line:load:oset}
    \State $AdjSets_{T \to O} \gets \{Oset(T,O)\}$
    \State\Return{$Relation, IsIdent, AdjSets$}
\end{algorithmic}
}
\end{algorithm}

\begin{restatable}[]{theorem}{loadoset}
\label{thm:load_oset}
    LOAD is sound and complete in finding optimal adjustment sets for a pair of target variables.
\end{restatable}

We provide the proof in \Cref{proof:load_oset}.
We show our implementation of \ac{LOAD} in \Cref{alg:load}.
LOAD returns three sets of variables.
$Relation$ variable holds the causal relation between the targets, while $IsIdent$ indicates whether a causal effect is identifiable or not.
If the causal effect is identifiable, then $AdjSets$ contains the optimal adjustment set, which is trivially the empty set for a zero causal effect.
Otherwise, it contains the locally valid parent adjustment sets for targets that are explicit or possible ancestors of the other.
We discuss further optimizations in \Cref{sec:optimizations}, including substituting MB-by-MB in line~\ref{line:load:forb_proj} with just Markov blanket discovery, which did not significantly improve performance.

\paragraph{Computational Complexity}
The computational complexity of \ac{LOAD} in terms of CI tests is mainly determined by the local causal discovery algorithm it uses.
Given a local causal discovery algorithm with complexity $\mathcal{O}(L)$, the worst-case complexity of LOAD is given by applying the local algorithm on all variables, adding up to $\mathcal{O}(L \times |\mathbf{V}|)$, plus performing $\mathcal{O}(|\mathbf{V}|)$ tests to determine causal relations using \Cref{alg:is_explicit_ancestor} and \Cref{alg:is_possible_ancestor}, dominated by the previous term. 
However, we can re-implement most local methods, so that they cache and re-use Markov blankets and local structures from their earlier runs, which allows us to only run a Markov blanket discovery and local structure learning once per variable, reducing the worst-case complexity to $\mathcal{O}(L)$.

In our implementation, we use the MB-by-MB algorithm modified to cache and reuse Markov blankets and local structures from earlier runs in  \Cref{alg:mb-by-mb}.
The complexity of MB-by-MB depends on the complexity of the algorithm for finding Markov blankets, which is a well-known task in literature. 
The worst-case complexity for forward search in Markov blanket discovery is $\mathcal{O}(|MB_{\max}| \times |\mathbf{V}|)$, where $|MB_{\max}|$ is the size of the largest Markov blanket \citep{tsamardinos2003algorithms}.
For learning local structures over the Markov Blankets, we then use PC  \citep{spirtes2000causation} which has a worst-case complexity of $\mathcal{O}(|MB_{\max}|^{d_{\max}+2})$, where $d_{\max}$ is the maximum degree over all nodes.

In the worst-case, MB-by-MB runs both subroutines for each node in the graph, resulting in a total complexity for our implementation of $\mathcal{O}(|MB_{\max}| \times |\mathbf{V}|^2 + |MB_{\max}|^{d_{\max}+2}\times |\mathbf{V}|))$.
On the other hand, when applied to the whole graph, the PC algorithm has a worst-case complexity of $\mathcal{O}(|\mathbf{V}|^{d_{\max}+2})$. Since typically $|MB_{\max}| \ll |\mathbf{V}|$ this showcases the benefits of local causal discovery and our method.
This means that LOAD typically has a much lower complexity compared to running a global method such as PC to obtain a full CPDAG and then collecting adjustments sets using IDA-like methods \citep{maathuis2009estimating}, which we also demonstrate empirically.

\section{EMPIRICAL EVALUATION}
\label{sec:experiments}

\begin{figure*}[ht!]
    \centering
    \includegraphics[width=.9\linewidth]{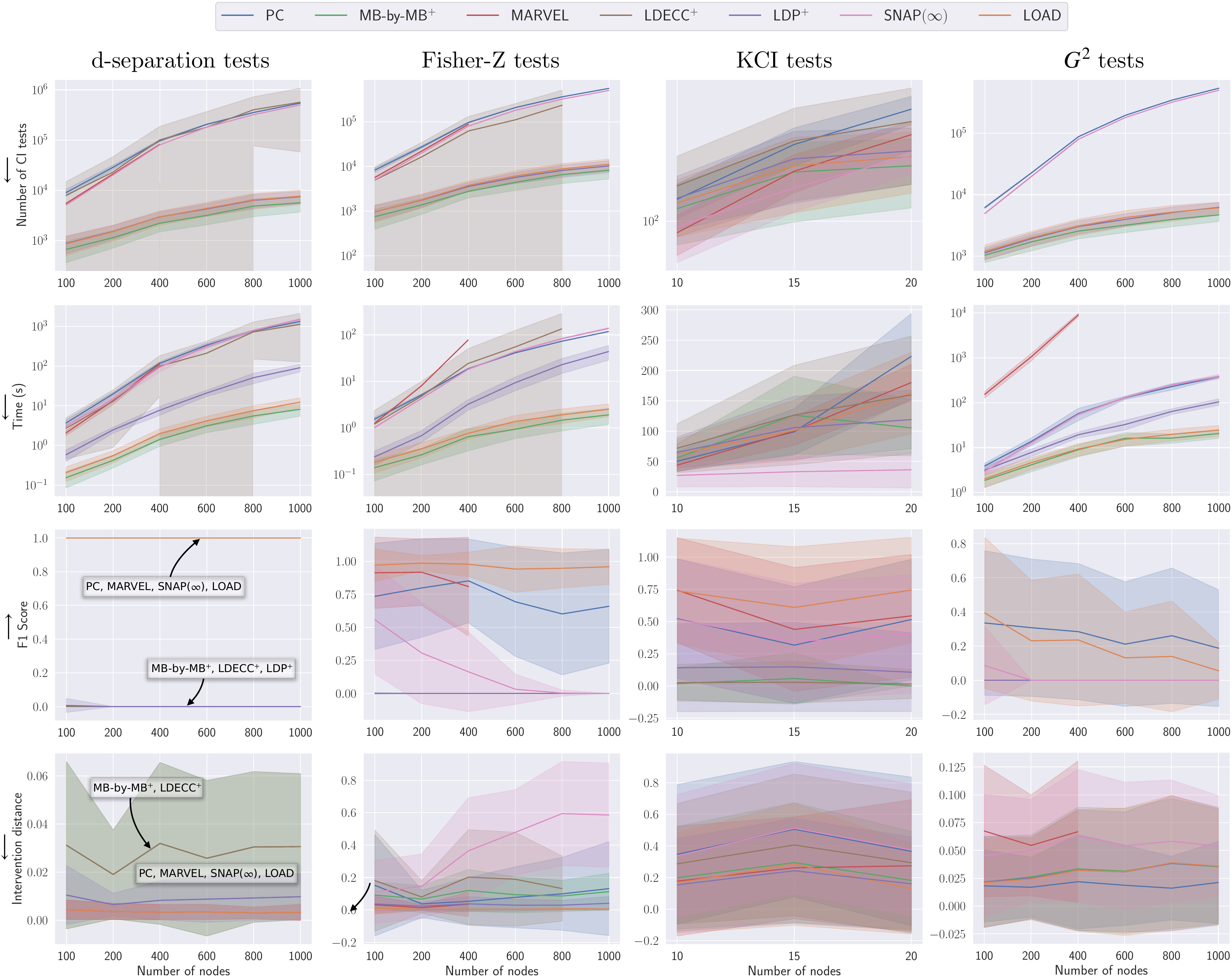}
    \caption{Results over number of nodes with $n_{\mathbf{D}} = 10000$ ($n_{\mathbf{D}} = 1000$ for KCI), $\overline{d} = 2$, $d_{\max} = 10$ and targets such that one is an explicit ancestor of the other.
    The shadow area denotes the the standard deviation.
    We could not run MARVEL on more than 400 nodes in any setting, LDECC$^+$ on 1000 nodes with Fisher-Z tests and LDECC$^+$ on any number of nodes with $G^2$ tests due to memory issues.
    We report all numbers in detail in \Cref{sec:extended_results}.
    }
    \label{fig:main_results}
\end{figure*}

We evaluate the performance of LOAD on synthetic and realistic semi-synthetic data from bnlearn \citep{scutari2010bnlearn}. We compare with global algorithms PC \citep{spirtes2000causation} and MARVEL \citep{mokhtarian2021recursive}, the targeted causal discovery method SNAP($\infty$) \citep{schubert2025snap}, as well as the extended versions of local algorithms MB-by-MB \citep{wang2014discovering}, LDECC \citep{gupta2023local} and LDP \citep{maasch2024local} based on our \Cref{alg:relate}. We describe the extension in \Cref{app:baselines_plus} and refer to the extended methods as MB-by-MB$^+$, LDECC$^+$ and LDP$^+$. 
In this section, we focus on the setting in which the causal relation between the two targets is not known.
In \Cref{sec:bk_experiments}, we present results for the scenario where the treatment-outcome relation is known, and evaluate the original versions of MB-by-MB, LDECC and LDP, showing similar results.
We publish our code at \url{https://github.com/matyasch/load}.

\paragraph{Synthetic data.}
We evaluate \ac{LOAD} using  $n_{\mathbf{D}} = 10000$ data samples generated according to randomly generated Erdős–Rényi graphs with varying number of nodes $n_{\mathbf{V}}$, expected degree of $\overline{d} = 2$ and maximum degree of $d_{\max} = 10$.
We sample pairs of targets such that one is an explicit ancestor of the other, but we do not provide this background knowledge about the causal relation between the targets to the algorithms.
We evaluate algorithms using three types of data and corresponding \ac{CI} tests: oracle d-separation tests, Fisher-Z tests on linear Gaussian data, and $G^2$ tests on binary data.
We also evaluate all algorithms using the non-parametric KCI tests \citep{zhang2011kernel} using $n_{\mathbf{D}} = 1000$ data samples of linear Gaussian data.
For all CI tests we use a significance level $\alpha = 0.01$.
For linear Gaussian data, we sample edge weights from $[-3, -0.5] \cup [0.5, 3]$ with standard Gaussian noises.
For the binary data, we generate a conditional probability table according to the graph structure with probabilities sampled uniformly.
We report the average results over 100 seeds with the best 5 and worst 5 results removed for each method to better show the general trends.
We provide further details of the experimental setup in \Cref{sec:experimental_details}.

\paragraph{Realistic data.}
We evaluate all baselines on realistic data generated according to the MAGIC-NIAB and ANDES networks from bnlearn \citep{scutari2010bnlearn} with ground truth graphs in \Cref{app:real_networks}. 
Similarly to the synthetic setting, we sample target pairs such that one is an explicit ancestor of the other.
For MAGIC-NIAB ($n_{\mathbf{V}}=44$ and  $\overline{d}=3$) we simulate linear Gaussian data and use Fisher-Z tests with significance level $\alpha = 0.01$.
For ANDES ($n_{\mathbf{V}}=223$ and $\overline{d}=3.03$) we generate binary data and use $G^2$ tests.

\paragraph{Metrics.}
We compare methods over three metrics: number of CI tests (which is a hardware and implementation independent measure of computational efficiency), F1 score on the optimal adjustment sets and the intervention distance.
If both the output of an algorithm and the true CPDAG indicates that the optimal adjustment set does not exist, then we consider the F1 score to be 1.
If the output of an algorithm and the true CPDAG does not agree whether the optimal adjustment set should exist, then we consider the F1 score to be 0.
For MB-by-MB$^+$, LDECC$^+$ and LDP$^+$ we use the highest F1 score reached by the adjustment sets they return.
As the recovered adjustment sets should be ultimately used for causal effect estimation, we also report the intervention distance \citep{schubert2025snap}, defined for a single target pair as
$$
    \frac{1}{2} \sum_{T,T' \in \{(X,Y), (Y,X)\}}\frac{1}{|\hat{\Theta}_{T \to T'}|} \sum_{\hat{\theta} \in \hat{\Theta}_{T \to T'}} \left|\theta^*_{T \to T'}-\hat{\theta} \right|,
$$
where $\theta^*_{T \to T'}$ is the true causal effect of $T$ on $T'$ and $\hat{\Theta}_{T \to T'}$ is the set of estimated causal effects.
For all methods, if their output indicates that $T$ is a definite non-ancestor of $T'$, either based on the output CPDAG, the returned relation or missing adjustment sets, we consider its estimated causal effect $\hat{\Theta}_{T \to T'}$ to be $\{0\}$.
For MB-by-MB$^+$, LDECC$^+$, LDP$^+$ and LOAD, we use the adjustment sets they return to estimate the causal effects $\hat{\Theta}_{T \to T'}$.
For PC, MARVEL and SNAP, if their output CPDAG indicates that the causal effect is identifiable, then we use the optimal adjustment set according to the CPDAG for the causal effect estimation. If the output CPDAG indicates that the causal effect is not identifiable, then we use the locally valid parent adjustments, which results in a set of possible causal effects $\hat{\Theta}_{T \to T'}$.
Following \citet{gradu2022valid}, we use 10000 newly generated synthetic samples for causal effect estimation to avoid \emph{double dipping}.

\paragraph{Simulated Data Results.}
Our experimental results in \Cref{fig:main_results} show that LOAD combines the best of both worlds in terms of local and global causal discovery. In particular, it provides a comparable accuracy in causal effect estimation to the slow but accurate global causal discovery methods, while being only slightly more computationally expensive than the fast but less accurate local causal discovery methods.
We could not compare some of the methods in some settings, because they had running time or memory issues. In particular, we could not run MARVEL on more than 400 nodes, LDECC$^+$ on 1000 nodes with Fisher-Z tests and LDECC$^+$ on any number of nodes with $G^2$ tests.

The first row of \Cref{fig:main_results} describes the number of CI tests for each  method in each of the three settings: d-separation, Fisher-Z tests on linear Gaussian data and $G^2$ tests on binary data.
In all settings, the number of CI performed by LOAD is consistently well below global methods, LDECC$^+$ and SNAP($\infty$) and slightly above MB-by-MB$^+$ and LDP$^+$.
This is expected, since LOAD requires fewer tests than global methods due to not recovering the whole graph, but it requires more CI tests than local methods, since it focuses on a more difficult task, optimal adjustment set discovery.

We report the computation time in the second row of \Cref{fig:main_results}.
Our results show similar trends as the number of CI tests, but with LOAD and MB-by-MB$^+$ performing the best, requiring even less computation time than LDP, over all data settings except with KCI tests where SNAP($\infty$) performs best.

The third row of \Cref{fig:main_results} describes the F1 score for learning the optimal adjustment set in each setting. This is the main quality metric our method is targeting. Our results using oracle d-separation tests empirically verify that LOAD is sound and complete in this task the same way as global methods and SNAP($\infty$) are, providing an F1 score of 1. On the other hand, local methods (MB-by-MB, LDECC and LDP) have an F1 score of 0 in terms of recovering the optimal adjustment set, which is expected since this is not their goal.
On linear Gaussian data with Fisher-Z tests, LOAD outperforms all baselines and can consistently recover the optimal adjustment set almost perfectly. PC, MARVEL and SNAP($\infty$) still recover some of the optimal adjustment sets, while as expected local methods are still not able to learn any of the optimal adjustment sets.
LOAD similarly outperforms all baselines with KCI tests.
All methods struggle on binary data using $G^2$ tests, as it is inherently more challenging for CI tests than linear Gaussian data, leading to more CI errors, as we discuss in \Cref{sec:ci_errors}.
Here, PC performs best followed by LOAD, while other methods, including MARVEL and SNAP($\infty$), cannot recover the optimal adjustment set.

The fourth row of \Cref{fig:main_results} shows the intervention distance results.
Intervention distance provides a complementary metric for the quality of the optimal adjustment set discovery, since it focuses on the downstream task of causal effect estimation using the discovered (optimal) adjustment sets.
The theoretical advantages of the optimal adjustment set are showcased when measuring the intervention distance on results for oracle d-separation CI tests.
Here, methods that can recover the optimal adjustment set, including LOAD, achieve not only the lowest intervention distance, but also the lowest variance.
On linear Gaussian data, LOAD outperforms all baselines in intervention distance, with LDP following closely.
On binary data, PC has the lowest intervention distance, while LOAD is the second best with local methods MB-by-MB$^+$ and LDP$^+$.

\paragraph{Ablations.}
The accuracy of LOAD in estimating the optimal adjustment set depends on the quality of the local neighborhoods estimated by the local causal discovery subroutine.
In \Cref{sec:nb_error} we show that the F1 score of LOAD decreases slightly with more structural errors in the estimated local neighborhoods due to lower sample sizes.
Notably, even with an average of almost one erroneous node in every local neighborhood for 500 samples of linear Gaussian data with Fisher-Z CI tests, the F1 score and intervention distance of LOAD is still better than most baselines.
Additionally, we also study how errors in earlier steps of LOAD propagate to later steps in \Cref{sec:steps}.
As expected, the precision and recall scores of the first two steps are generally higher than the
later steps.
Interestingly, LOAD has a very high precision for Step 2 and better recall than for Step 1, meaning that it almost never predicts that the causal effect is identifiable when it is not.

Our implementation of LOAD utilizes MB-by-MB for local causal discovery, which employs the Grow-Shrink algorithm for Markov blanket discovery.
We perform an ablation where we implement LOAD with alternative modules for local causal discovery in \Cref{sec:lcd_algorithm} and Markov blanket discovery in \Cref{sec:mb_algorithm}.
Our results for replacing MB-by-MB with the CMB \citep{gao2015local} algorithm are presented in \Cref{fig:cmb} and show that MB-by-MB outperforms CMB in every metric except for computation time with both d-separation and Fisher-Z CI tests.
We show results for replacing the Grow-Shrink algorithm with Total-Conditioning \citep{pellet2008using} in \Cref{fig:total_conditioning} and with scored-based S$^2$TMB algorithm \citep{GAO2017277} in \Cref{fig:s2tmb}.
While Total-Conditioning performs fewer CI tests and achieves comparable intervention distances, it requires much more computation time.
Additionally, it fails to recover the optimal adjustment set on binary data.
S$^2$TMB takes a prohibitively long time for more than a few dozen nodes and achieves worse F1 scores and intervention distances than Grow-Shrink.
MB-by-MB, LDECC and LDP assume background knowledge of the treatment-outcome relation.
Thus, we also consider the scenario where this is provided to them in \Cref{sec:bk_experiments}.
We compare with LOAD$^*$, a variant of LOAD that can also leverage this information by skipping Step 1 of the algorithm.
Our results show the same trends as in \Cref{fig:main_results}, with LDP improving on its computational performance and LOAD$^*$ surpassing PC on binary data in F1 score and intervention distances.

We also report results on identifiable target pairs for $n_{\mathbf{V}} \leq 500$ (\Cref{sec:identifiable_experiments}), where the overall patterns remain consistent. The main distinction is again in the binary setting, where the F1 score and intervention distance of LOAD is on par with PC.

We evaluate the robustness of LOAD to different parameters of the synthetic data generation and experiment $n_{\mathbf{V}} = 200 $ with various numbers of data samples in \Cref{sec:samples}, and expected degrees for the synthetic graphs in \Cref{sec:degrees}.
As expected, with larger numbers of samples, PC, MARVEL, SNAP($\infty$) and LOAD improve in F1 score and all methods improve in intervention distance.
The number of samples does not seem to affect the computation time for Fisher Z. For $G^2$ the computation time and tests increase with the number of samples, reducing the gap between local methods and LOAD with PC and SNAP($\infty$). As expected, a larger average degree requires more CI tests and computation time, especially for PC and LDECC$^+$ , while the intervention distance and the F1 score worsen.

\paragraph{Realistic Data Results.}

\begin{table}
\caption{Results for realistic data. The best result for each metric is indicated in \textbf{bold}.
We did not run LDECC$^+$ due to memory issues.}
\label{tab:magic-niab}
\centering
\scriptsize{
\begin{tabular}{lrrrr}
\toprule
\multicolumn{1}{c}{} & \multicolumn{1}{c}{CI tests $\times 10^3$} & \multicolumn{1}{c}{F1 of Oset} & \multicolumn{1}{c}{Int. Dist.} \\
\midrule
\multicolumn{4}{c}{MAGIC-NIAB} \\
PC             & $14.94 \pm 0.34$ & $0.32 \pm 0.43$ & $0.016 \pm 0.035$ \\
MB-by-MB$^+$   & $2.36 \pm 1.65$ & $0.03 \pm 0.07$ & $0.007 \pm 0.007$ \\
MARVEL         & $9.66 \pm 4.87$ & $0.60 \pm 0.42$ & $0.011 \pm 0.015$ \\
LDP$^+$        & $2.49 \pm 1.64$ & $0.14 \pm 0.35$ & $0.009 \pm 0.010$ \\
SNAP($\infty$) & $\mathbf{2.08 \pm 1.64}$ & $0.30 \pm 0.29$ & $0.010 \pm 0.014$ \\
LOAD           & $5.31 \pm 6.79$ & $\mathbf{0.62 \pm 0.43}$ & $\mathbf{0.006 \pm 0.006}$ \\
\midrule
\multicolumn{4}{c}{ANDES} \\
PC             & $69.29 \pm 2.80$ & $0.00 \pm 0.00$ & $0.005 \pm 0.005$ \\
MB-by-MB$^+$   & $\mathbf{2.77 \pm 1.13}$ & $0.00 \pm 0.00$ & $0.004 \pm 0.003$ \\
MARVEL         & $24.75 \pm 0.00$ & $0.00 \pm 0.00$ & $0.013 \pm 0.032$ \\
LDP$^+$        & $2.86 \pm 1.12$ & $0.00 \pm 0.00$ & $0.004 \pm 0.003$ \\
SNAP($\infty$) & $24.75 \pm 0.00$ & $0.00 \pm 0.00$ & $0.013 \pm 0.032$ \\
LOAD           & $3.01 \pm 1.20$ & $\mathbf{0.05 \pm 0.15}$ & $\mathbf{0.002 \pm 0.001}$ \\
\bottomrule
\end{tabular}}
\end{table}

We report our results for the MAGIC-NIAB and ANDES networks in \Cref{tab:magic-niab}.
Similarly to synthetic data, the number of CI tests performed by LOAD is between global and local methods.
The quality of the recovered adjustment sets by LOAD is consistently as one of the best among all methods. 
Notably, LOAD achieves the only non-zero F1 score on the ANDES network, which poses a significant challenge for all algorithms due to generating binary data, as discussed in \Cref{sec:ci_errors}.
Furthermore, LOAD achieves the lowest intervention distance in both settings.
\section{CONCLUSIONS}
\label{sec:conclusion}
We propose LOAD, a sound and complete local method to identify optimal adjustment sets to estimate the causal effect of a pair of target variables.
Unlike previous local discovery methods, LOAD can determine the treatment-outcome relation of the targets, and if the causal effect is identifiable.
We show that LOAD is more computationally efficient than global methods, while retaining their statistical efficiency for causal effect estimation.
As future work, we plan to extend LOAD to the causally insufficient setting, for which there are no unique optimal adjustment sets \citep{smucler2022efficient}, which will require defining graphical criteria to identify candidate adjustments.
\section*{Acknowledgements}

We thank SURF (\url{www.surf.nl}) for the support in using the National Supercomputer Snellius. We would like to thank Roel Hulsman for valuable discussions.

\bibliography{references}
\clearpage
\section*{Checklist}

\begin{enumerate}

  \item For all models and algorithms presented, check if you include:
  \begin{enumerate}
    \item A clear description of the mathematical setting, assumptions, algorithm, and/or model. [Yes]
    \item An analysis of the properties and complexity (time, space, sample size) of any algorithm. [Yes]
    \item (Optional) Anonymized source code, with specification of all dependencies, including external libraries. [Yes]
  \end{enumerate}

  \item For any theoretical claim, check if you include:
  \begin{enumerate}
    \item Statements of the full set of assumptions of all theoretical results. [Yes]
    \item Complete proofs of all theoretical results. [Yes]
    \item Clear explanations of any assumptions. [Yes]     
  \end{enumerate}

  \item For all figures and tables that present empirical results, check if you include:
  \begin{enumerate}
    \item The code, data, and instructions needed to reproduce the main experimental results (either in the supplemental material or as a URL). [Yes, we publish our code at \url{https://github.com/matyasch/load}]
    \item All the training details (e.g., data splits, hyperparameters, how they were chosen). [Yes]
    \item A clear definition of the specific measure or statistics and error bars (e.g., with respect to the random seed after running experiments multiple times). [Yes]
    \item A description of the computing infrastructure used. (e.g., type of GPUs, internal cluster, or cloud provider). [Yes]
  \end{enumerate}

  \item If you are using existing assets (e.g., code, data, models) or curating/releasing new assets, check if you include:
  \begin{enumerate}
    \item Citations of the creator if your work uses existing assets. [Yes]
    \item The license information of the assets, if applicable. [Yes]
    \item New assets either in the supplemental material or as a URL, if applicable. [Not Applicable]
    \item Information about consent from data providers/curators. [Not Applicable]
    \item Discussion of sensible content if applicable, e.g., personally identifiable information or offensive content. [Not Applicable]
  \end{enumerate}

  \item If you used crowdsourcing or conducted research with human subjects, check if you include:
  \begin{enumerate}
    \item The full text of instructions given to participants and screenshots. [Not Applicable]
    \item Descriptions of potential participant risks, with links to Institutional Review Board (IRB) approvals if applicable. [Not Applicable]
    \item The estimated hourly wage paid to participants and the total amount spent on participant compensation. [Not Applicable]
  \end{enumerate}

\end{enumerate}

\newpage
\appendix
\clearpage

\onecolumn
\aistatstitle{Local Causal Discovery for Statistically Efficient Causal Inference \\
Supplementary Materials}

\section{ISSUES IN LOCAL CAUSAL DISCOVERY}
\label{sec:issues}

In this section we detail the issues we found in various local causal discovery methods.
In \Cref{sec:local_pc} we discuss LocalPC, a subroutine implemented by several local discovery algorithm, and show that LocalPC is not sound in terms of the edges it returns.
Then, in \Cref{sec:ldecc} and \Cref{sec:ldp} we consider issues specific to the LDECC and LDP algorithms.

\subsection{Issues with LocalPC}
\label{sec:local_pc}
A key component of several local discovery algorithms is to distinguish adjacent variables from spouses (or co-parents) in the Markov blanket of a variable.
\citet{gupta2023local} implement this subroutine for LDECC by following the logic of the PC algorithm as shown in \Cref{alg:localpc}.
\citet{xie2024local} also \say{apply the logic of the PC algorithm to identify the adjacent edges} over the Markov blanket in their MMB-by-MMB algorithm, an extension of the MB-by-MB algorithm \citep{wang2014discovering} to causally insufficient settings the same way.
Similarly, \citet{li2025local} \say{apply the logic of the
PC algorithm to identify the skeleton} over the Markov blanket in their LocICR algorithm, which aims to identify causal relationships from local information.

\begin{algorithm}
\caption{LocalPC}
\label{alg:localpc}
\footnotesize{
\begin{algorithmic}[1]
    \Require Target $X$, Markov blanket $MB(X)$
    \State $Adj(X) \gets MB(X)$
    \State $s \gets 0$
    \While{$Adj(X) > s$}\label{line:localpc:while}
        \For{$V \in Adj(X)$}
            \For{$\mathbf{S} \subseteq Adj(X) \setminus \{V\}$ s.t. $|\mathbf{S}| = s$}\label{line:localpc:candsets}
                \If{$X \perp\kern-6pt\perp V | \mathbf{S}$}
                    \State $Adj(X) \gets Adj(X) \setminus \{V\}$
                    \State \textbf{break}
                \EndIf
            \EndFor
        \EndFor
        \State $s \gets s + 1$
    \EndWhile
\State\Return{$Adj(X)$}
\end{algorithmic}
}
\end{algorithm}

However, the LocalPC algorithm is not sound, as it can \textbf{incorrectly} identify non-adjacent spouses in the Markov blanket as adjacent.
\Cref{example:localpc_counter_example} shows a CPDAG/PAG an example such a scenario.

\begin{example}
\label{example:localpc_counter_example}
Consider the CPDAG and PAG on \Cref{fig:localpc_counter_example} with treatment $X$ and outcome $Y$.
In both graphs, $X$ and $Y$ are non-adjacent.
However, the LocalPC algorithm in \Cref{alg:localpc} with input $X$ incorrectly identifies $Y$ as adjacent, as shown below.
\begin{compactenum}
    \item LocalPC correctly finds $X \indep V_2$ at $s = 0$.
    \item LocalPC removes $V_2$ from $Adj(X)$.
    \item After this, $V_2$ is never considered for conditioning at line~\ref{line:localpc:candsets}.
    \item Thus, LocalPC never finds $X \indep Y | \{V_1, V_2\}$.
    \item LocalPC \textbf{incorrectly} leaves $Y \in Adj(X)$.
\end{compactenum}

\end{example}

\begin{figure}[ht]
\begin{subfigure}{.49\linewidth}
    \centering
    \includegraphics[width=.8\linewidth]{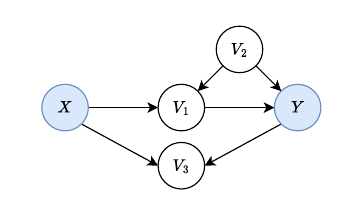}
    \caption{Example CPDAG}
    \label{fig:localpc_cpdag_counter_example}
\end{subfigure}
\begin{subfigure}{.49\linewidth}
    \centering
    \includegraphics[width=.8\linewidth]{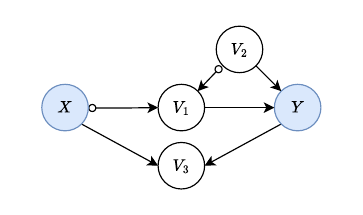}
    \caption{Example PAG}
    \label{fig:localpc_pag_counter_example}
\end{subfigure}
\caption{An example CPDAG and PAG for \Cref{example:localpc_counter_example} where the LocalPC algorithm in \Cref{alg:localpc} incorrectly identifies $Y$ as adjacent to $X$ because it never tests $X \indep Y |\{V_1, V_2\}$}
\label{fig:localpc_counter_example}
\end{figure}

The core issue of the LocalPC algorithm in \Cref{alg:localpc} is that when it tries to separate $X$ from some $V$, it only considers the still remaining adjacencies of $X$ for separating sets, and not the adjacencies of $V$.
This is in contrast with the true logic of the PC algorithm, which would also consider subsets of the adjacencies of $V$ if the adjacencies of $X$ could not separate the two.
Finding separating sets among subsets of the adjacencies, which are a superset of the parents, of both variables is possible due to the Markov condition and Lemma 3.3.9 in \citep{spirtes2000causation}, which we restate in \Cref{lem:pc_skel}.

\begin{lemma}[Lemma 3.3.9 in \citep{spirtes2000causation}]
\label{lem:pc_skel}
    In a DAG $G$, if $Y$ is not a descendant of $X$, and $X$ and $Y$ are not adjacent, then $X$ and $Y$ are d-separated by $Pa(X)$.
\end{lemma}

Crucially, \Cref{lem:pc_skel} does not apply when $Y$ is a descendant of $X$, which is exactly the case in \Cref{example:localpc_counter_example}.
Instead of LocalPC, local algorithms should implement \Cref{lem:nb_in_mb}, which appeared as part of the Grow-Shrink algorithm in \citep{margaritis1999bayesian}, discussed in \citep{pellet2008finding}, stated as a consequence of Lemma~1 in \citep{wang2014discovering}, and also stated as Theorem~1 in \citep{xie2024local} and as Proposition~1 in \citep{li2025local}.

\begin{lemma}
\label{lem:nb_in_mb}
    In a DAG/MAG $G$, for $X$ and $Y \in MB(X)$, $Y$ is adjacent to $X$ if and only if
    $$
    \forall \mathbf{S} \subsetneq MB(X) \setminus \{Y\} : X \dep Y | \mathbf{S}.
    $$
\end{lemma}

\subsubsection{Solution}
In practice, LocalPC in \Cref{alg:localpc} can be fixed by replacing $Adj(X)$ with $MB(X)$ at lines \ref{line:localpc:while} and \ref{line:localpc:candsets}.

\subsubsection{Code to Reproduce Ex.~\ref{example:localpc_counter_example}}
The following code snippet reproduces \Cref{example:localpc_counter_example} for LDECC using the public implementation from \citet{gupta2023local}.
The assertion in the final line fails because node 4 ($Y$ in the example) is returned as a child of node 0 ($X$ in the example).

\begin{minted}{python}
import networkx as nx
import pandas as pd
import numpy as np

from causal_discovery.ldecc import LDECCAlgorithm

dag = nx.DiGraph([(0, 1), (0, 3), (1, 4), (2, 1), (2, 4), (4, 3)])
data = np.zeros((2,len(dag.nodes)))
ldecc = LDECCAlgorithm(
    treatment_node=0,
    outcome_node=4,
    alpha=0.05,
    use_ci_oracle=True,
    graph_true=dag
)
result = ldecc.run(pd.DataFrame(data))
assert result['tmt_children'] == {1,3}
\end{minted}

The following code snippet implements \Cref{example:localpc_counter_example} for MMB-by-MMB using the public implementation\footnote{GitHub repository: \url{https://github.com/fengxie009/MMB-by-MMB}.} from \citet{xie2024local}.
Similarly to LDECC, the assertion in the final line fails because node 4 ($Y$ in the example) is returned as adjacent to node 0 ($X$ in the example).

\begin{minted}{python}
import networkx as nx
import numpy as np

import MMB_by_MMB
import Utils.MMB_TC_Z as tc

dag = nx.DiGraph([(0, 1), (0, 3), (1, 4), (2, 1), (2, 4), (4, 3)])
data = np.zeros((2, len(dag.nodes)))
ci_test = lambda x, y, S, data, alpha: (
    nx.is_d_separator(dag, x, y, set(S)),
    float(nx.is_d_separator(dag, x, y, set(S))),
)
tc.CI_Test = ci_test
MMB_by_MMB.CI_Test = ci_test
mmb_by_mmb = MMB_by_MMB.MMB_by_MMB(
    Data=data,
    target=0,
    alpha=0.05,
    p=len(dag.nodes),
    maxK=max([val for _, val in dag.degree()]),
)
result = mmb_by_mmb.mmb_by_mmb()
assert result[3].tolist() == [1]
\end{minted}

In the case of LocICR, even though $Y$ is considered adjacent by the algorithm, the returned causal relation is correct.

\subsection{Issues with LDECC beyond LocalPC}
\label{sec:ldecc}
In the previous section we showed that LDECC might erroneously consider spouses as adjacent due to the LocalPC algorithm.
In this section we show that even if the LocalPC algorithm is fixed or replaced, LDECC is still not complete, as it might not orient all adjacent edges to the target, even though they are oriented in the true CPDAG.

In particular, Theorem 5 in \citep{gupta2023local} aims to show that \say{every orientable neighbor of the treatment $X$ will get oriented correctly by LDECC}. However, \Cref{example:ldecc_counter_example} shows a CPDAG where LDECC fails to orient an orientable parent of the treatment.

\begin{figure}[ht]
\begin{subfigure}{.49\linewidth}
    \centering
    \includegraphics[width=.8\linewidth]{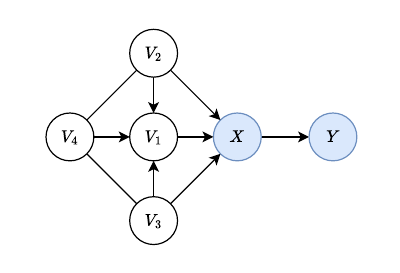}
    \caption{CPDAG for \Cref{example:ldecc_counter_example}}
    \label{fig:ldecc_counter_example}
\end{subfigure}
\begin{subfigure}{.49\linewidth}
    \centering
    \raisebox{10mm}{\includegraphics[width=.8\linewidth]{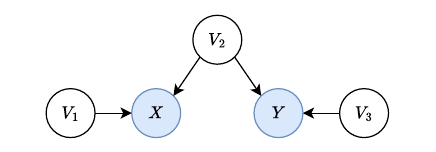}}
    \caption{DAG/CPDAG for \Cref{example:ldp_counter_example}}
    \label{fig:ldp_counter_example}
\end{subfigure}
\caption{(\subref{fig:ldecc_counter_example}): Example CPDAG for \Cref{example:ldecc_counter_example}, where LDECC fails to orient $V_1 \to X$ and identify $V_1$ as a parent of $X$. (\subref{fig:ldp_counter_example}): Example DAG/CPDAG for \Cref{example:ldp_counter_example} where LDP fails to identify a valid adjustment set for treatment $X$ and outcome $Y$, even though the causal effect of 0 and a valid adjustment set of $V_2$ is identifiable.}
\label{fig:ldp_ldecc_counterexamples}
\end{figure}

\begin{example}
\label{example:ldecc_counter_example}
Consider the CPDAG on \Cref{fig:ldecc_counter_example} with treatment $X$ and outcome $Y$. In this CPDAG, all parents of $X$ are orientable. However, LDECC fails to orient $V_1$ as a parent, as shown below.

\begin{compactenum}
    \item LDECC correctly identifies that Markov blanket of $X$ as $MB(X) = \{Y, V_1, V_2, V_3\}$.
    \item LDECC correctly identifies the adjacencies of $X$ as $Adj(X) = \{Y, V_1, V_2, V_3\}$.
    \item LDECC correctly finds $Y \indep V_{1,2,3,4} | X$.
    \item LDECC correctly finds $V_2 \indep V_3 | V_4$.
    \begin{compactenum}
        \item The condition at line 7 is satisfied, correctly identifying $V_2$ and $V_3$ as parents.
        \item The condition at line 19 is satisfied, correctly identifying $Y$ as a children.
    \end{compactenum}
    \item No other independence exists and LDECC \textbf{incorrectly} leaves $V_1$ as unoriented.
\end{compactenum}
\end{example}

Running PC on \Cref{example:ldecc_counter_example} would orient $V_1 \to X$ by first orienting the v-structure $V_2 \to V_1 \gets V_3$, then applying Meek rule 3 to orient $V_4 \to V_1$ and finally applying Meek rule 1 to orient $V_1 \to X$.
Crucially, in the v-structure $V_2 \to V_1 \gets V_3$ that originally starts the propagation of these orientation rules, both $V_2$ and $V_3$ are adjacent to $X$ and hence cannot be separated from it.
LDECC fails on \Cref{example:ldecc_counter_example} because it does not contain a condition to handle this scenario.
In particular, this structure contradicts the proof of Theorem 5 for Meek rule 1 in \citep{gupta2023local}, which states that both $V_2$ and $V_3$ can be separated from $X$ with separating sets that should include $V_1$.

\subsubsection{Code to Reproduce Ex.~\ref{example:ldecc_counter_example}}

\Cref{example:ldecc_counter_example} can be reproduced using the public implementation\footnote{GitHub repository: \url{https://github.com/acmi-lab/local-causal-discovery}.} of LDECC from \citet{gupta2023local} as follows, where the assertion in the final line fails because node 2 ($V_1$ in the example) is not recognized as a parent of node 0 ($X$ in the example) by LDECC.

\begin{minted}{python}
import networkx as nx
import pandas as pd
import numpy as np

from causal_discovery.ldecc import LDECCAlgorithm

dag = nx.DiGraph([(0,1), (2,0), (3,0), (4,0), (3,2), (4,2), (5,2), (5,3), (5,4)])
data = np.zeros((2,len(dag.nodes)))
ldecc = LDECCAlgorithm(
    treatment_node=0,
    outcome_node=1,
    alpha=0.05,
    use_ci_oracle=True,
    graph_true=dag
)
result = ldecc.run(pd.DataFrame(data))
assert result['tmt_parents'] == {2,3,4}
\end{minted}

\subsection{Issues with LDP}
\label{sec:ldp}
Local Discovery by Partitioning (LDP) is a causal discovery algorithm that identifies a valid adjustment set for a treatment-outcome pair by partitioning other variables according to their causal relationships to this pair \citep{maasch2024local}.
LDP requires the following assumptions to hold to be able to discover a valid backdoor adjustment set for the causal effect of treatment $X$ on outcome $Y$:
\begin{compactenum}
    \item $X$ and $Y$ are marginally dependent.
    \item $Y$ is not an ancestor of $X$.
    \item There exists a non-descendant of $Y$ that is marginally dependent on $Y$, but marginally independent of $X$.
    \item There exists a non-descendant of $X$ whose causal effect on $Y$ is \textit{fully mediated} by $X$, shares no confounder with $Y$, and is marginally independent of every non-descendant of $X$ that lie on active backdoor paths between $X$ and $Y$.
\end{compactenum}

\citet{maasch2024local} state that \say{the causal effect of $X$ on $Y$ can be of arbitrary strength or \textbf{null}}. Thus, as long as the conditions above hold $X$ does not need to be an ancestor of $Y$ (although it is then not clear what \textit{fully mediated} means in the fourth condition). However, in this case, \Cref{example:ldp_counter_example} shows a causal DAG with corresponding CPDAG from which a valid backdoor adjustment set and the true causal effect of 0 is identifiable, but LDP incorrectly raises a warning that a valid adjustment set is not identifiable.

\begin{example}
\label{example:ldp_counter_example}
Consider the causal DAG, with identical CPDAG, in \Cref{fig:ldp_counter_example} with treatment $X$ and outcome $Y$.
The DAG satisfies all conditions listed above, as
\begin{compactenum}
    \item $X$ and $Y$ are marginally dependent.
    \item $Y$ is not an ancestor of $X$.
    \item $V_3$ is a non-descendant of $Y$ that is marginally dependent on $Y$, but marginally independent of $X$.
    \item $V_1$ is a non-descendant of $X$ whose causal effect on $Y$ is fully mediated by $X$, shares no confounder with $Y$, and is marginally independent of $V_2$, the only non-descendant of $X$ that lie on active backdoor paths between $X$ and $Y$.
\end{compactenum}

Furthermore, we can identify a valid adjustment set $\{V_2\}$ and a causal effect of $0$ from the CPDAG. However, LDP fails to identify any valid backdoor adjustment sets as shown below.

\begin{compactenum}
    \item[Step 1] correctly identifies $\mathbf{Z}_{8} = \emptyset$
    \item[Step 2] correctly identifies $\mathbf{Z}_{4} = \{V_3\}$.
    \item[Step 3] \textbf{incorrectly} identifies $\mathbf{Z}_{5,7} = \emptyset$ instead of $\mathbf{Z}_{5,7} = \{V_1\}$, because $Y \indep V_1$.
    \item[Step 4] correctly identifies $\mathbf{Z}_{\text{POST}} = \emptyset$.
    \item[Step 5] correctly identifies $\mathbf{Z}_{\text{MIX}} = \emptyset$.
    \item[Step 6] is skipped because $\mathbf{Z}_{\text{MIX}} = \emptyset$, \textbf{incorrectly} resulting in $\mathbf{Z}_1 = \mathbf{Z}_{1,5} = \emptyset$.
    \item[Step 7] is skipped because $\mathbf{Z}_1 = \mathbf{Z}_{1,5} = \emptyset$ \textbf{incorrectly} resulting in $\mathbf{Z}_{5} = \emptyset$.
    \item[Step 8] \textbf{incorrectly} concludes that a valid adjustment set is not identifiable because $\mathbf{Z}_{5} = \emptyset$.
\end{compactenum}

\end{example}

LDP fails on \Cref{example:ldp_counter_example} because all non-descendants of the treatment that satisfy condition 4 are marginally independent of the outcome, violating the conditions of Step 3. Thus, LDP can be easily fixed by extending condition 4 to also require that at least one non-descendant satisfying all existing conditions is also marginally dependent on the outcome.

\subsubsection{Code to Reproduce Ex.~\ref{example:ldp_counter_example}}

\Cref{example:ldp_counter_example} can be reproduced using the public implementation\footnote{GitHub repository: \url{https://github.com/jmaasch/ldp}.} of LDP from \citet{maasch2024local} as follows, where the assertion in the final line fails due to node 2 ($V_2$ in \Cref{example:ldecc_counter_example}) not getting recognized as a valid adjustment set for treatment node 1 ($X$ in \Cref{example:ldecc_counter_example}) and outcome node 3 ($Y$ in \Cref{example:ldecc_counter_example}) by LDP.

\begin{minted}{python}
import numpy as np
import networkx as nx
import pandas as pd

from ldp import LDP

dag = nx.DiGraph([(0, 1), (2, 1), (2, 3), (4, 3)])
data = np.zeros((2, len(dag.nodes)))
ldp = LDP(data=pd.DataFrame(data), independence_test="oracle")
ldp.test = lambda x, y, S: float(nx.is_d_separator(dag, x, y, set(S) if S else set()))
res = ldp.partition_z(exposure=1, outcome=3)
assert res["Z1"] != []
\end{minted}

\clearpage
\section{EXTENDED LOCAL CAUSAL DISCOVERY ALGORITHMS}
\label{app:baselines_plus}

In this section we implement the extensions to local causal discovery algorithms outlined in \Cref{sec:extend}.
We show
MB-by-MB$^+$ in \Cref{alg:mb_by_mb_plus}, LDECC$^+$ in \Cref{alg:ldecc_plus} and LDP$^+$ in \Cref{alg:ldp_plus}.
In general, all three algorithms use \Cref{alg:relate} to determine ancestral relations and then return adjustment sets according to possible treatment-outcome relationships.
MB-by-MB$^+$ and LDECC$^+$ use MB-by-MB and LDECC respectively to find local information and return the locally valid parent sets for each target determined to be an explicit or a possible ancestor of the other.

\begin{algorithm}
    \centering
    \caption{MB-by-MB$^+$}
    \label{alg:mb_by_mb_plus}
    \begin{algorithmic}[1]
    \Require{Targets $X,Y \in \mathbf{V}$ and variables $\mathbf{V}$}
    \Ensure{Causal relation of $X$ and $Y$ in $Relation$ and the locally valid parent adjustment sets of (possible) treatments among $X$ and $Y$.}
    \State $AdjSets_{X \to Y} \gets \emptyset, AdjSets_{Y \to X} \gets \emptyset$
    \Step{\textbf{Step 1:} Determine ancestral relationships}
    \State $LocalRelate \gets$ LocalRelate($X,Y, \mathbf{V}$) \Comment{\Cref{alg:relate}}
    \If{$Relation(X,Y) = ExplAn$}
        \State $AdjSets_{X \to Y} \gets \text{LocalValidSets}(X, Y, G_X)$
    \ElsIf{$Relation(Y,X) = ExplAn$}
       \State $AdjSets_{Y \to X} \gets \text{LocalValidSets}(Y, X, G_Y)$
    \Else\Comment{Cannot determine treatment and outcome}
        \If{$Relation(X,Y) = PossAn$}
            \State $AdjSets_{X \to Y} \gets \text{LocalValidSets}(X, Y, G_X)$
        \EndIf
        \If{$Relation(Y,X) = PossAn$}
            \State $AdjSets_{Y \to X} \gets \text{LocalValidSets}(Y, X, G_Y)$
        \EndIf
    \EndIf
    \State\Return{$Relation, AdjSets$}
    \end{algorithmic}
\end{algorithm}

\begin{algorithm}
    \centering
    \caption{LDECC$^+$}
    \label{alg:ldecc_plus}
    \begin{algorithmic}[1]
    \Require{Targets $X,Y \in \mathbf{V}$ and variables $\mathbf{V}$}
    \Ensure{Causal relation of $X$ and $Y$ in $Relation$ and the locally valid parent adjustment sets of (possible) treatments among $X$ and $Y$.}
    \State $AdjSets_{X \to Y} \gets \emptyset, AdjSets_{Y \to X} \gets \emptyset$
    \Step{\textbf{Step 1:} Determine ancestral relationships}
    \State $LocalRelate \gets$ LocalRelate($X,Y, \mathbf{V}$) \Comment{\Cref{alg:relate}, but MB-by-MB replaced by LDECC at lines \ref{line:localrelate:G_x} and \ref{line:localrelate:G_y}.}
    \If{$Relation(X,Y) = ExplAn$}
        \State $AdjSets_{X \to Y} \gets \text{LocalValidSets}(X, Y, G_X)$
    \ElsIf{$Relation(Y,X) = ExplAn$}
       \State $AdjSets_{Y \to X} \gets \text{LocalValidSets}(Y, X, G_Y)$
    \Else\Comment{Cannot determine treatment and outcome}
        \If{$Relation(X,Y) = PossAn$}
            \State $AdjSets_{X \to Y} \gets \text{LocalValidSets}(X, Y, G_X)$
        \EndIf
        \If{$Relation(Y,X) = PossAn$}
            \State $AdjSets_{Y \to X} \gets \text{LocalValidSets}(Y, X, G_Y)$
        \EndIf
    \EndIf
    \State\Return{$Relation, AdjSets$}
    \end{algorithmic}
\end{algorithm}

LDP$^+$ uses MB-by-MB to find local information and uses LDP to determine a valid adjustment set for each target determined to be an explicit or a possible ancestor of the other.

\begin{algorithm}
    \centering
    \caption{LDP$^+$}
    \label{alg:ldp_plus}
    \begin{algorithmic}[1]
    \Require{Targets $X,Y \in \mathbf{V}$ and variables $\mathbf{V}$}
    \Ensure{Causal relation of $X$ and $Y$ in $Relation$ and the causal partitions determined by LDP for the approriate (possible) treatment-outcome pairs.}
    \State $AdjSets_{X \to Y} \gets \emptyset, AdjSets_{Y \to X} \gets \emptyset$
    \Step{\textbf{Step 1:} Determine ancestral relationships}
    \State $LocalRelate \gets$ LocalRelate($X,Y, \mathbf{V}$) \Comment{\Cref{alg:relate}}
    \If{$Relation(X,Y) = ExplAn$}
        \State $AdjSets_{X \to Y} \gets \text{LDP}(X, Y)$
    \ElsIf{$Relation(Y,X) = ExplAn$}
       \State $AdjSets_{Y \to X} \gets \text{LDP}(Y, X)$
    \Else\Comment{Cannot determine treatment and outcome}
        \If{$Relation(X,Y) = PossAn$}
            \State $AdjSets_{X \to Y} \gets \text{LDP}(X, Y)$
        \EndIf
        \If{$Relation(Y,X) = PossAn$}
            \State $AdjSets_{Y \to X} \gets \text{LDP}(Y, X)$
        \EndIf
    \EndIf
    \State\Return{$Relation, AdjSets$}
    \end{algorithmic}
\end{algorithm}

\clearpage
\section{ADDITIONAL ALGORITHMS}
\label{app:algos}

This section provides pseudocode for all subroutines implemented by \ac{LOAD}.
\Cref{alg:is_possible_ancestor} implements the negation of \Cref{thm:possible_ancestor}, i.e., that if $X \dep Y | Pa_G(X)$, then $X$ is a possible ancestor of $Y$.
Before testing this dependence, the algorithm checks whether possible ancestry can already be determined from the local information around $X$, when $Y$ is part of it.
\Cref{alg:is_explicit_ancestor} implements \Cref{thm:explicit_ancestor}, and also checks the local information around $X$ before testing the corresponding dependence.

\begin{minipage}{0.45\textwidth}
\begin{algorithm}[H]
    \centering
    \caption{IsExplAn}
    \label{alg:is_explicit_ancestor}
    \begin{algorithmic}[1]
        \Require Nodes $X, Y$ and graph $G$
        \If{$Y \in Ch_G(X)$}\label{line:isexplan:test_ch}
            \State\Return{True}\label{line:isexplan:is_ch}
        \ElsIf{$Y \in Pa_G(X) \cup Sib_G(X)$}\label{line:isexplan:test_pa_sib}
            \State\Return{False}\label{line:isexplan:is_pa_sib}
        \ElsIf{$X \dep Y | Pa_G(X) \cup Sib_G(X)$}\label{line:isexplan:test_explanc}
            \State\Return{True}
        \Else
            \State\Return{False}
        \EndIf\label{line:isexplan:end}
    \end{algorithmic}
\end{algorithm}
\end{minipage}
\hfill
\begin{minipage}{0.45\textwidth}
\begin{algorithm}[H]
    \centering
    \caption{IsPossAn}
    \label{alg:is_possible_ancestor}
    \begin{algorithmic}[1]
        \Require Nodes $X, Y$ and graph $G$
        \If{$Y \in Ch_G(X) \cup Sib_G(X)$}\label{line:ispossan:test_ch_sib}
            \State\Return{True}\label{line:ispossan:is_ch_sib}
        \ElsIf{$Y \in Pa_G(X)$}\label{line:ispossan:test_pa}
            \State\Return{False}\label{line:ispossan:is_pa}
        \ElsIf{$X \dep Y | Pa_G(X)$}\label{line:ispossan:test_possan}
            \State\Return{True}
        \Else
            \State\Return{False}
        \EndIf\label{line:ispossan:end}
    \end{algorithmic}
\end{algorithm}
\end{minipage}

\Cref{alg:locally_valid_sets} implements a slight variation of Algorithm 3 in \citep{maathuis2009estimating}.
In particular, it returns all locally valid parent adjustment sets (instead of the estimated causal effects using them). 

Intuitively this algorithm checks if considering a subset of the siblings $\mathbf{S}$ as parents of $T$ in addition to $Pa_G(T)$ would create new v-structures and hence orientations, which would contradict the fact that they are actually siblings and not parents in the original CPDAG. A simple way to avoid this is to only add new candidate parents that are adjacent to the currently considered parents.

\begin{algorithm}[H]
    \centering
    \caption{LocalValidSets (slight variation of Algorithm 3 in \citep{maathuis2009estimating} that returns all valid sets instead of causal effects)}
    \label{alg:locally_valid_sets}
    \begin{algorithmic}[1]
        \Require Treatment $T$, Outcome $O$ and graph $G$
        \State $ValidSets \gets \emptyset$
        \For{$\mathbf{S} \subseteq Sib_G(T)\setminus \{O\}$}
            \Step{Ensure that any new candidate parent $S \in \mathbf{S}$ is adjacent to all candidate parents $Pa_G(T) \cup \mathbf{S} \setminus \{S\}$.}
            \State{$Valid \gets$  True}
            \For{$S \in \mathbf{S}$}
                \If{$(Pa_G(T) \cup \mathbf{S} \setminus \{S \})\centernot\subseteq Adj_G(S)$}
                    \State{$Valid \gets$  False}
                    \State{\textbf{break}}
                \EndIf
            \EndFor
            \If{$Valid$}
                \State $ValidSets \gets ValidSets \cup \{Pa_G(T) \cup \mathbf{S}\}$
            \EndIf
        \EndFor
        \State\Return{$ValidSets$}
    \end{algorithmic}
\end{algorithm}

\Cref{alg:mb-by-mb} implements the MB-by-MB algorithm by \citet{wang2014discovering}.
MB-by-MB is a local causal discovery algorithm that identifies the parents, children and siblings of a target variable by sequentially discovering the Markov blankets of variables until some stopping criteria is met.
Similarly to the original paper, in this algorithm we use the notation of $MB^+(X) = MB(X) \cup \{X\}$.

We allow MB-by-MB to cache and reuse already identified separating sets, Markov blankets and local structures by previous runs of the algorithm, as these might be identified during earlier runs of MB-by-MB on different targets.
This allows us to avoid unnecessarily identifying the same relationships multiple times.
When the cached objects are empty, the algorithm behaves exactly the same as the original algorithm.

The local structure $L_X$ of $X$ is learned at line~\ref{line:mb_by_mb:learn_local_structure} by running the skeleton step of the PC algorithm \citep{spirtes2000causation} over $MB^+(X)$ and then orient v-structures.
While this structure might contain erroneous edges due to latent confounding, \citet{wang2014discovering} prove that the edges connected to $X$ and the v-structures containing $X$ in $L_X$ are correct and thus safe to add to the learned graph over all considered variables $G$ at line~\ref{line:mb_by_mb:update_G}.

\begin{algorithm}
\caption{Slight variation of MB-by-MB \citep{wang2014discovering} with caching}
\label{alg:mb-by-mb}
\footnotesize{
\begin{algorithmic}[1]
    \Require Target $T \in \mathbf{V}$, variables $\mathbf{V}$, cached Markov blankets $MB$ and cached local structures $L$
    
    \Procedure{Orient-Undirected-Edges}{$G$}
        \For{$(a \to b - c) \in G$}
            \If{exists a separating set for $(a,c)$ and $b \in sepset(a,c)$}
                \State Orient $b \to c$ in $G$
            \EndIf
        \EndFor
        \For{$(a \to b \to c - a) \in G$}
            \State Orient $a \to c$ in $G$
        \EndFor
        \For{$a - b, a - c \to b$ and $a - d \to b \in G$}
            \If{exists a separating set for $(c,d)$ and $a \in sepset(c,d)$}
                \State Orient $a \to b$ in $G$
            \EndIf
        \EndFor
        \State\Return{$G$}
    \EndProcedure

    \State Take cached separating sets $sepset$, Markov blankets $MB$ and local structures $L$ from previous runs of MB-by-MB.
    \State $\text{DoneList} = \emptyset$
    \State $\text{WaitList} = [T]$
    \State $G = \{\mathbf{V}, \emptyset\}$
    \Repeat
        \State $X \gets$ take a node from the head of WaitList
        \If{$MB(X) \in MB$}
            \State Use cached $MB(X)$ from an earlier run of MB-by-MB
        \Else
            \Step{Design choice: We use Grow-Shrink for Markov blanket discovery instead of IAMB in \citep{wang2014discovering}}
            \State Find $MB(X)$ using the Grow-Shrink Markov blanket algorithm (Figure 2 in \citep{margaritis1999bayesian})
            \State Add $MB(X)$ to cached $MB$
        \EndIf
        \State Add $[MB(X) \setminus \text{DoneList} \setminus \text{WaitList}]$ to the tail of WaitList
        \State Add $X$ to DoneList
        
        \If{$L_X \in L$}
            \State Use cached $L_X$ from an earlier run of MB-by-MB
        \ElsIf{$MB^+(X) \subseteq MB^+(X')$ for some $X' \in \text{DoneList}$}
            \State Set $L_X$ equal to the substructure of $L_{X'}$ over $MB^+(X)$
        \ElsIf{$MB(X) \subseteq \text{DoneList}$}
            \State Set $L_X$ equal to the substructure of $G$ over $MB^+(X)$
        \Else
            \Step{Design choice: we run the PC skeleton search and orient v-structures instead of IC in \citep{wang2014discovering}}
            \State Learn $L_X$ from observed data of $MB^+(X)$ using PC \citep{spirtes2000causation}  and updating cached $sepset$\label{line:mb_by_mb:learn_local_structure}
            \State Add $L_X$ to cached $L$
        \EndIf
        \State Put the edges connected to $X$ and the v-structures containing $X$ in $L_X$ to $G$\label{line:mb_by_mb:update_G}
        \State Call \textsc{Orient-Undirected-Edges($G$)} to orient undirected edges in $G$
        \State Remove all nodes from WaitList whose paths to $T$ in $G$ are blocked by directed edges
    \Until{WaitList is empty}
    \State\Return{$G, MB, L$}
\end{algorithmic}
}
\end{algorithm}

\clearpage
\section{PROOFS}
\label{app:proofs}

\subsection{Useful Results from Literature}

We use the following result for the correctness of the MB-by-MB algorithms from \citep{wang2014discovering} in our proof for \Cref{lem:relate} in \Cref{proof:relate}.

\begin{theorem}[Theorem 3 in \citep{wang2014discovering}]\label{thm:mb_by_mb}
Suppose that a causal network is causal sufficient and faithful to a probability distribution and that all conditional independencies are correctly checked.
Then the MB-by-MB algorithm can correctly discover the edges connected to the target node $T$.
Further these edges can be correctly oriented as the partially directed network representing the Markov equivalence class of the underlying global causal network.
\end{theorem}

This result implies that after MB-by-MB, all the nodes adjacent to a target will have their edges oriented as they would be in the underlying CPDAG. In other words, this result correctly identifies the parents, children and siblings (i.e., nodes connected by undirected edges) of each target.

Following \citet{perkovic2018complete} we also define an unshielded path as a path, where each consecutive triple $X,Y,Z$ is unshielded, i.e., $X$ is not adjacent to $Z$, and we use the following result from  \citep{ZHANG20081873} in our proofs for \Cref{lem:identifiability_implies_expl_an} in \Cref{proof:identifiability_implies_expl_an} and \Cref{lem:local_amenability} in \Cref{proof:local_amenability}.

\begin{lemma}[Lemma B.1 in \citep{ZHANG20081873}]\label{lem:b1zhang}
    Let $X$ and $Y$ be distinct nodes in a CPDAG $G$. If $p$ is a possibly directed path from $X$ to $Y$ in G, then some subsequence of $p$ forms an unshielded possibly directed path from $X$ to $Y$ in $G$.
\end{lemma}

We use the following results from \citep{andersson1997characterization} in our proofs for \Cref{lem:alg_is_expl_an}, \Cref{lem:sib_notadj_outcome} and \Cref{lem:local_amenability} in \Cref{proof:local_amenability}.

\begin{theorem}[Theorem 4.1 in \citep{andersson1997characterization}]\label{thm:4.1andersson}
    A graph $G=(\mathbf{V}, \mathbf{E})$ is equal to the CPDAG for some DAG $D$ if and only if $G$ satisfies the following four conditions.
    \begin{compactenum}
        \item $G$ is a chain graph.
        \item For every chain component $\tau$ of $G$, $G_\tau$ is chordal.
        \item The structure $X \to Y - Z$ does not occur as an induced subgraph of $G$.
        \item Every directed edge $X \to Y \in G$ is strongly protected in $G$.
    \end{compactenum}
\end{theorem}

In our proofs, we only use items 1 and 3 in \Cref{thm:4.1andersson}, while for the terms used in the other items, we refer to the original paper \citep{andersson1997characterization}.
As discussed by \citet{andersson1997characterization}, item 1 implies that any CPDAG is a chain graph, i.e. it may have both directed and undirected edges but it may contain \emph{no partially directed cycles}, which are formed by a possibly causal path from $X$ to $Y$ with a directed path from $Y$ to $X$. Item 3 implies that no CPDAGs can contain the the structure $X \to Y - Z$.

We use the following results by \citet{maathuis2015generalized} in our proof for \Cref{lem:local_amenability}.

\begin{corollary}[Corollary 4.2 in \citep{maathuis2015generalized}]\label{cor:4.2maathuis}
    Let $X$ and $Y$ be two distinct vertices in a CPDAG $C$. Let $C_{\underline{X}}$ be the graph obtained from $C$ by removing all directed edges out of $X$. Then there exists a generalized back-door set relative to $(X, Y)$ and $C$ if and only if $Y \in Pa_C(X)$ and $Y \notin PossDe_{C_{\underline{X}}}(X)$. Moreover, if such a generalized back-door set exists, then $Pa_C(X)$ is such a set.
\end{corollary}

A generalized back-door set relative to $(X, Y)$ is a valid adjustment set that allows for the estimation of the causal effect of $X$ on $Y$.
The formal definition of a generalized back-door set is given by Definition 3.7 in \citep{maathuis2015generalized}.
\Cref{cor:4.2maathuis} states that if a generalized back-door set exists relative to $(X, Y)$ in a CPDAG $C$, (which is the case when the causal effect of $X$ on $Y$ is identifiable in $C$), then the definite parents $Pa_C(X)$ are such a valid adjustment set.

\subsection{Proof of Lem.~\ref{lem:relate}}
\label{proof:relate}
We introduce the following Lemmas to show the soundness and completeness of \Cref{alg:is_explicit_ancestor} and \Cref{alg:is_possible_ancestor}, which we use in the proofs of \Cref{lem:relate}.

\begin{lemma}\label{lem:alg_is_expl_an}
    \Cref{alg:is_explicit_ancestor} is sound and complete in finding explicit ancestral relations between a pair of targets.
\end{lemma}
\begin{proof}
    Consider target pair $X$ and $Y$ in a CPDAG $G$.
    We need to prove that IsExplAn($X,Y$) described in \Cref{alg:is_explicit_ancestor} returns True iff $X$ is an explicit ancestor of $Y$ in $G$, i.e., $X$ has a directed path to $Y$ in the CPDAG.

    We consider first the case in which $X$ and $Y$ are adjacent. In this case, $X$ is an explicit ancestor of $Y$ iff $Y$ is a definite child of $X$. The ``if'' direction is trivial: if $Y$ is a child of $X$, then there exists a directed path from $X$ to $Y$, thus $X$ is an explicit ancestor of $Y$ (lines \ref{line:isexplan:test_ch} and \ref{line:isexplan:is_ch}).

    For the ``only if'' direction, if $X$ is an explicit ancestor of $Y$, then by definition there exists a directed path from $X$ to $Y$.
    Then, an edge between $X$ and $Y$ has to be oriented as $X \to Y$, otherwise there would be a partially directed cycle in $G$, which is forbidden in CPDAGs as shown by \Cref{thm:4.1andersson}.
    This means that if $X$ and $Y$ is adjacent, but $Y$ is not a child of $X$, i.e., $Y$ is a parent or a sibling of $X$, then $X$ cannot be an explicit ancestor of $Y$ (lines \ref{line:isexplan:test_pa_sib} and \ref{line:isexplan:is_pa_sib}).

    In case $X$ and $Y$ are not adjacent, then \Cref{alg:is_explicit_ancestor} implements \Cref{thm:explicit_ancestor} by \citet{fang2022local} at lines \ref{line:isexplan:test_explanc} to \ref{line:isexplan:end}, which is a sound and complete criterion for determining explicit ancestry.
\end{proof}

\begin{lemma}\label{lem:alg_is_poss_an}
    \Cref{alg:is_possible_ancestor} is sound and complete in finding possible ancestral relations between a pair of targets.
\end{lemma}
\begin{proof}
    Consider target pair $X$ and $Y$ in a CPDAG $G$.
    We need to prove that IsPossAn($X,Y$) described in  \Cref{alg:is_possible_ancestor} returns True iff $X$ is a possible ancestor of $Y$ in $G$, i.e., there is a possibly directed path from $X$ to $Y$.

    In case $X$ and $Y$ are adjacent, then $X$ is a possible ancestor of $Y$ iff $Y$ is a child or sibling of $X$. The ``if'' direction is trivial: if $Y$ is a child or sibling of $X$, then there is a possibly directed path from $X$ to $Y$, thus by definition $X$ is a possible ancestor of $Y$ (lines \ref{line:ispossan:test_ch_sib} and \ref{line:ispossan:is_ch_sib}).

    For the ``only if'' direction, if $X$ is a possible ancestor of $Y$, then $Y$ cannot be a definite parent of $X$, since otherwise $Y$ would also be a definite ancestor, i.e., an ancestor of $X$ in every DAG represented by $G$. By acyclicity this means that in none of these DAGs $X$ can be an ancestor of $Y$, thus $X$ would not be a possible ancestor of $Y$ by definition (lines \ref{line:ispossan:test_pa} and \ref{line:ispossan:is_pa}).

    In case $X$ and $Y$ are not adjacent, then \Cref{alg:is_possible_ancestor} implements \Cref{thm:possible_ancestor} by \citet{fang2022local} at lines \ref{line:ispossan:test_possan} to \ref{line:ispossan:end}, which is a sound and complete criterion for determining explicit ancestry.
\end{proof}

We can now use these results to prove that the algorithm for identifying the relations between targets (\Cref{alg:relate}) is sound and complete in its outputs.

\relate*
\begin{proof}
    At the start of the algorithm, the estimated causal relations for the target pair $X$ and $Y$ are set as definite non-ancestor as default in both directions, denoted as $DefNonAn$ (line \ref{line:localrelate:initialize}).
    We then use MB-by-MB to discover the local information around both $X$ and $Y$ (lines 2-3).
    
    By \Cref{thm:mb_by_mb}, this local information correctly identifies the parents, children and siblings of each of the targets in the CPDAG. This information is enough to use the previous lemmas and algorithms to correctly identify the explicit ancestry in each direction (lines \ref{line:localrelate:explan_start} to \ref{line:localrelate:explan_end}) with \Cref{alg:is_explicit_ancestor}, which we show is sound and complete in \Cref{lem:alg_is_expl_an}, and possible ancestry in each direction (lines \ref{line:localrelate:possan_start} to \ref{line:localrelate:possan_end}) with \Cref{alg:is_possible_ancestor}, which we show is sound and complete in \Cref{lem:alg_is_poss_an}. 
    While for explicit ancestry between $X$ and $Y$ only one direction can hold at the same time, and the other direction defaults to definite non-ancestor, for possible ancestry both directions are possible at the same time (e.g., when $X$ and $Y$ are connected by an undirected path), so we test them both in any case (lines \ref{line:localrelate:possan_start} to \ref{line:localrelate:possan_end}).

    Finally, assume that a target, say $X$, is a definite non-ancestor of the other, say $Y$.
    By definition, $X$ is a definite non-ancestor of $Y$ when it is not a possible ancestor of $Y$.
    By definition, if $X$ is not a possible ancestor of $Y$, then it is also not an explicit ancestor of $Y$.
    Furthermore, if $Y$ is an explicit ancestor of $X$, then  $X$ is not a possible ancestor of $Y$.
    Since \Cref{alg:is_explicit_ancestor} and \Cref{alg:is_possible_ancestor} are sound and complete for testing explicit and possible ancestry, $X$ is a definite non-ancestor of $Y$ iff \Cref{alg:relate}: (i) finds that $Y$ is an explicit ancestor of $X$, or (ii) if it finds that $X$ is neither an explicit nor a possible ancestor of $Y$, in which case it never updates the ancestral relation of $X$ on $Y$, and correctly returns definite non-ancestry.
\end{proof}

\subsection{Proof of Lem.~\ref{lem:identifiability_implies_expl_an}}
\label{proof:identifiability_implies_expl_an}

\identifiabilityimpliesexplan*
\begin{proof}

    We prove this by contradiction.
    Assume that $X$ is a possible but not an explicit ancestor of $Y$, so there are only possibly directed, but no directed paths from $X$ to $Y$.
    Let $p$ be the shortest possibly directed path from $X$ to $Y$ in $G$.
    According to \Cref{lem:b1zhang}, there is a subsequence of $p$ that is an unshielded possibly directed path from $X$ to $Y$.
    However, since $p$ is the shortest possibly directed path from $X$ to $Y$, the only subsequence of $p$ that could be a possibly directed path from $X$ to $Y$ is $p$ itself.
    Hence, by \Cref{lem:b1zhang}, $p$ is an unshielded possibly directed path from $X$ to $Y$.
    If $p$ is unshielded, then there cannot be any colliders in $p$, otherwise there would be an unshielded collider $V \to V' \gets V''$ in $p$, which would also be directed as such in the CPDAG $G$, and thus $p$ would not be a possibly directed path in $G$ since the edge $V' \gets V''$ is directed backwards.
    Thus, every node on $p$ has to be a non-collider in every DAG in the MEC represented by $G$.
    
    However, since $p$ is not a ``fully'' directed path in $G$, there is both a DAG $D_1$ in the MEC in which it is directed from $X$ to $Y$, but also another DAG $D_2$ in the MEC in which it is not directed $X$ to $Y$, but still open when the conditioning set is empty.
    Consequently, in $D_1$ no valid adjustment sets for $X$ and $Y$ can contain a node on $p$ as it is a causal path from $X$ to $Y$.
    However, in $D_2$ all valid adjustments set have to contain at least one node on $p$, as it is an otherwise open, but not causal path from $X$ to $Y$.
    Hence, there is no shared set of nodes that is a valid adjustment set in all DAGs in the \ac{MEC} of $G$ and therefore the causal effect of $X$ on $Y$ is not identifiable in $G$.
\end{proof}

\subsection{Proof of Lem.~\ref{lem:local_amenability}}

\label{proof:local_amenability}

We introduce the following lemma to aid us in the proof of \Cref{lem:local_amenability}. 

\begin{lemma}
\label{lem:sib_notadj_outcome}
    For any two distinct nodes $X$ and $Y$ such that $X \in PossAn_G(Y)$
    in a CPDAG $G$, if $G$ is adjustment amenable relative to $X$ and $Y$, i.e., every possibly directed path from $X$ to $Y$ starts with a directed edge out of $X$, then for all $V \in Sib_G(X)$, $Y \notin Adj_G(V)$. 
\end{lemma}
\begin{proof}
    Consider any $V \in Sib_G(X)$, i.e., $X - V$.
    The edges $V - Y$ or $V \to Y$ cannot be in $G$, since then there would be a possibly directed path from $X$ to $Y$ starting with $V$ as $X - V - Y$ or $X - V \to Y$, which would make $G$ not adjustment amenable relative to $X$ and $Y$ by definition, since the first edge is not directed out of $X$.
    
    Since $X \in PossAn_G(Y)$, there is a a possibly directed path from $X$ to $Y$.
    Then, if $V \gets Y$ would be in $G$, there would be a partially directed cycle in $G$ as $X \to \cdots Y \to V - X$, which is forbidden in CPDAGs as shown by \Cref{thm:4.1andersson}.
    Thus, $V$ and $Y$ cannot be adjacent.
\end{proof}

\begin{figure}[ht]
\begin{subfigure}{.49\linewidth}
    \centering
    \includegraphics[width=.8\linewidth]{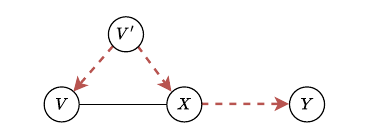}
    \caption{Example path for case 1a.}
    \label{subfig:1a}
\end{subfigure}
\begin{subfigure}{.49\linewidth}
    \centering
    \includegraphics[width=.8\linewidth]{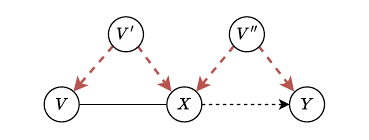}
    \caption{Example path for case 1b.}
    \label{subfig:1b}
\end{subfigure}
\begin{subfigure}{.49\linewidth}
    \centering
    \includegraphics[width=.8\linewidth]{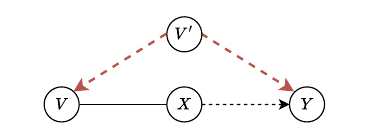}
    \caption{Example path for case 2a.}
    \label{subfig:2a}
\end{subfigure}
\begin{subfigure}{.49\linewidth}
    \centering
    \includegraphics[width=.8\linewidth]{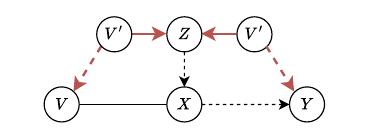}
    \caption{Example path for case 2b.}
    \label{subfig:2b}
\end{subfigure}
\caption{
Example paths, shown in red, for the different cases considered in the ``only if'' direction of \Cref{lem:local_amenability}.
The dashed arrows indicate that there can be a longer, possibly directed path between the two nodes instead of a direct edge.
(\subref{subfig:1a}): A path $p$ between $V$ and $Y$ that goes through $X$ such that $X$ is a non-collider on $p$. (\subref{subfig:1b}): A path $p$ between $V$ and $Y$ that goes through $X$ such that $X$ is a collider on $p$.
(\subref{subfig:2a}): A path $p$ between $V$ and $Y$ that does not go through $X$ and does not contain a node $Z$ that is both an ancestor of $X$ and a collider on $p$.
(\subref{subfig:2b}): A path $p$ between $V$ and $Y$ that does not go through $X$ but contains a node $Z$ that is both an ancestor of $X$ and a collider on $p$.}
\label{fig:localamenability_diagrams}
\end{figure}

\localamenability*
\begin{proof}
    We first consider the ``if'' direction and assume that $V \indep Y | Pa_G(V) \cup \{X\}, \forall V \in Sib_G(X)$. We will now show that all possibly directed paths from $X$ to $Y$ start with a definite child of $X$, which implies that they start with a directed edge out of $X$, i.e., they are amenable relative to $(X,Y)$ according to \citet{perkovic2015complete}. All nodes that are adjacent to $X$ are either siblings, parents or children of $X$. Since a possibly directed path from $X$ to $Y$ cannot start with an edge into $X$, which happens if the first node on the path is a parent of $X$, we just need to show that the first node on all of this path cannot be a sibling of $X$.
    
    Consider any $V \in Sib_G(X)$. Assuming the faithfulness condition, $V$ and $Y$ are not adjacent, since they can be d-separated. 
    If there are open paths between $V$ and $Y$, conditioning on the definite parents $Pa_G(V)$ blocks paths that start with an edge into $V$, since by definition the definite parents cannot be colliders on those paths. If any other open path exists, i.e., starting either with an undirected or a directed edge out of $V$, since we assume it has to be blocked to have d-separation, and it cannot be blocked by $Pa_G(V)$, then it has to be blocked by $X$. This means $X$ is a node on all of those paths. By definition, the possibly directed paths from $V$ to $Y$ are a subset of these paths, so they all contain $X$.
    This implies that no possibly directed path from $X$ to $Y$ goes through any sibling $V$ (since then $X$ would have to appear twice on the path). Since they also cannot go through a parent of $X$, this means that all possibly directed paths from $X$ to $Y$ start with a child of $X$.
    Thus, all possibly directed paths from $X$ to $Y$ start with an edge out $X$.
    By definition, this means that $G$ is amenable relative to $(X,Y)$.

    For the other ``only if'' direction, we now assume that $G$ is adjustment amenable relative to $(X,Y)$, i.e., 
    all possibly directed paths from $X$ to $Y$ start with a directed edge out of $X$. Consider any $V \in Sib_G(X)$.
    By \Cref{lem:sib_notadj_outcome}, $V$ and $Y$ are not adjacent and thus they can be d-separated, which implies a conditional independence given the appropriate separating set, since we assume the causal Markov assumption.
    
    We now show that $Pa_G(V) \cup \{X\}$ blocks all paths between $V$ and $Y$. To do this, we consider the following (exhaustive) cases of possible paths between $V$ and $Y$ and possible DAGs in the MEC represented by $G$ that they can appear in, and show that in all of these cases the paths are blocked by $Pa_G(V) \cup \{X\}$ in the DAGs:
    \begin{compactenum}
        \item Any path $p$ between $V$ and $Y$ that goes through $X$ and
        \begin{compactenum}
            \item any DAG $D$ in the MEC represented by $G$ in which $X$ is a non-collider on $p$ (as shown in \Cref{subfig:1a});
            \item any DAG $D$ in the MEC represented by $G$ in which $X$ is a collider on $p$ (as shown in \Cref{subfig:1b});
        \end{compactenum}
        \item Any path $p$ between $V$ and $Y$ that does not go through $X$ and
        \begin{compactenum}
            \item any DAG $D$ in the MEC represented by $G$ in which $p$ does not contain a node $Z$ that is both an ancestor of $X$ and a collider on $p$ (as shown in \Cref{subfig:2a});
            \item any DAG $D$ in the MEC represented by $G$ in which $p$ contains a node $Z$ that is both an ancestor of $X$ and a collider on $p$ (as shown in \Cref{subfig:2b});
        \end{compactenum}
    \end{compactenum}

    \paragraph{Case 1a.} Consider any path $p$ between $V$ and $Y$ that goes through $X$ in any DAG $D$ in the MEC represented by $G$ where $X$ is a non-collider on $p$, such as the one shown in \Cref{subfig:1a}.
    Then, $p$ is blocked by conditioning on $X$ regardless of whether we also condition on $Pa_G(V)$ or not.
    Thus, for any path $p$ between $V$ and $Y$ that goes through $X$ in any DAG $D$ in the MEC represented by $G$ in which $X$ is a non-collider on $p$, $p$ is blocked by $Pa_G(V) \cup \{X\}$.

    \paragraph{Case 1b.} Consider any path $p$ between $V$ and $Y$ that goes through $X$ in any DAG $D$ in the MEC represented by $G$ where $X$ is a collider on $p$, such as the one shown in \Cref{subfig:1b}.
    Since $X$ is a collider on $p$ in $D$ then conditioning on $X$ can potentially unblock $p$.
    In this case, we have to show that additionally conditioning on the definite parents $Pa_G(V)$ blocks $p$ in $D$.
    If $X$ is a collider on $p$ in $D$ and there are no other colliders on $p$, then the sub-path between $X$ and $Y$ on $p$ in $D$ must be a path between $X$ and $Y$ that starts with an edge into $X$.
    Since $G$ is adjustment amenable relative to $(X,Y)$ then, by Theorem 3.6 in \citep{perkovic2020identifying}, the causal effect of $X$ on $Y$ is identifiable in $G$.
    In this case, by \Cref{cor:4.2maathuis}, the set of definite parents $Pa_G(X)$ of $X$ blocks all paths between $X$ and $Y$ starting with an edge into $X$ in $D$, including $p$.
    
    We can show that this also implies that $Pa_G(V)$ also blocks all paths between $X$ and $Y$ starting with an edge into $X$ in $D$, including $p$, because $Pa_G(X) \subseteq Pa_G(V)$.
    To show this, consider any $P \in Pa_G(X)$.
    $P$ and $V$ must be adjacent in $G$, since the unshielded structure $P \to X - V$ cannot appear in a CPDAG, as shown by \Cref{thm:4.1andersson}.
    Furthermore, $P$ cannot be a child of $V$, because then $V \to P \to X - V$ would create a partially directed cycle which is forbidden in CPDAGs, as shown by \Cref{thm:4.1andersson}.
    For the same reason, $P$ also cannot be an sibling of $V$, as $P \to X - V - P$ also creates a partially directed cycle.
    This means that every $P \in Pa_G(X)$ has to also be a parent of $V$.
    Thus, for any path $p$ between $V$ and $Y$ that go through $X$ and any DAG $D$ in the MEC represented by $G$ in which $X$ is a collider on $p$, $p$ is blocked by $Pa_G(V) \cup \{X\}$.
    
    \paragraph{Case 2.} Before considering cases 2a and 2b, we first show that any path $p$ between $V$ and $Y$ that does not go through $X$ is blocked by $Pa_{G}(V)$ in $G$ (and thus in all DAGs $D$ in the MEC represented by $G$). We will then also have to show that this still holds even if we add $X$ to the conditioning set.
    
    The amenability of the causal effect of $X$ on $Y$ means that all possibly directed paths from $X$ to $Y$ start with an edge out of $X$. This implies that all possibly directed paths from $V$ to $Y$ have to go through $X$, otherwise there would exist a possibly directed path from $X$ to $Y$ through $V$ starting with an undirected edge as $X - V \cdots Y$ in $G$.
    Thus, all paths between $V$ and $Y$ that do not go through $X$, including $p$, are not possibly directed paths from $V$ to $Y$.
    
    Now, consider $G$ with all possibly directed paths from $V$ to $Y$ removed, denoted as $G'$, which allows us to focus on the other, not possibly directed, paths, including $p$.
    Since by construction there are no possibly directed paths from $V$ to $Y$ in $G'$, then $V$ is a definite non-ancestor of $Y$ in $G'$.
    Then, \Cref{thm:possible_ancestor} states that $V \indep Y | Pa_{G'}(V)$, i.e., all remaining paths between $V$ and $Y$, including $p$, are blocked by $Pa_{G'}(V)$ in $G'$.
    Since the incoming edges into $V$ in $G$ and $G'$ are the same, because we only removed the possibly directed paths from $V$ to $Y$, which all start with an outgoing or undirected edge out of $V$, then by definition the definite parents of $V$ are the same, i.e., $Pa_{G'}(V) = Pa_{G}(V)$.
    Hence, all paths that are not possibly directed from $V$ to $Y$ in $G$, including $p$, are blocked by $Pa_{G}(V)$. We now show that adding $X$ to this conditioning set still keeps $p$ blocked by considering the two cases in which $p$ contains or does not contain a node $Z$ that is both an ancestor of $X$ and a collider on $p$. 

    \paragraph{Case 2a.} Consider any DAG $D$ in the MEC represented by $G$ in which $p$ does not contain a node $Z$ that is both an ancestor of $X$ and a collider on $p$, such as the one shown in \Cref{subfig:2a}.
    Then, additionally conditioning on $X$ does not open $p$.
    Thus, for any path $p$ between $V$ and $Y$ that does not go through $X$ and any DAG $D$ in the MEC represented by $G$ in which $p$ does not contain a node $Z$ that is both an ancestor of $X$ and a collider on $p$, $p$ is blocked by $Pa_G(V) \cup \{X\}$.

    \paragraph{Case 2b.} Consider any DAG $D$ in the MEC represented by $G$ in which $p$ contains a node $Z$ that is both an ancestor of $X$ and a collider on $p$, such as the one shown in \Cref{subfig:2b}.
    Then, additionally conditioning on $X$ might open up $p$.
    In the following, we show that this does not happen.
    Since $Z$ is an ancestor of $X$ in the DAG $D$, it is a possible ancestor of $X$ in the CPDAG $G$ and there is a possibly directed path from $Z$ to $X$ in $G$.
    Then, $Z$ cannot be a descendant of $V$ in $D$, because then there would be a possibly directed path from $V$ to $Z$ in $G$, and thus a partially directed cycle $V \cdots Z \cdots X - V$ in $G$, which is forbidden in CPDAGs as shown by \Cref{thm:4.1andersson}.
    If $Z$ cannot be a descendant of $V$ in any $D$, then  $V$ is a definite non-ancestor of $Z$ in $G$.
    Then, by \Cref{thm:possible_ancestor}, $V \indep Z | Pa_{G}(V)$.
    This implies that the sub-path of $p$ between $V$ and $Z$ is already blocked by $Pa_{G}(V)$ in $G$, and also by $Pa_G(V) \cup \{X\}$. Since we just need a subpath of $p$ to be blocked by $Pa_G(V) \cup \{X\}$ for the whole path $p$ to be blocked, this means that  $p$ is blocked by $Pa_G(V) \cup \{X\}$.
    Thus, for any path $p$ between $V$ and $Y$ that does not go through $X$ and any DAG $D$ in the MEC represented by $G$ in which $p$ contains a node $Z$ that is both an ancestor of $X$ and a collider on $p$, $p$ is blocked by $Pa_G(V) \cup \{X\}$.

    In summary, $Pa_G(V) \cup \{X\}$ blocks all paths between $V$ and $Y$ in $G$. By the Causal Markov assumption, this means the conditional independence $V \indep Y|Pa_G(V) \cup \{X\}$ holds for any sibling $V$ of $X$.
\end{proof}

\subsection{Proof of Cor.~\ref{lem:load_identifiability}}
\label{proof:load_identifiability}

\loadidentifiability*
\begin{proof}
    LOAD begins by identifying the causal relations between the target pair using \Cref{alg:relate} (line~\ref{line:load:relate}).
    By \Cref{lem:relate}, the identified causal relations are sound and complete for definite non-ancestry, possible ancestry and explicit ancestry.
    If no explicit ancestral relation is found in either direction, then by \Cref{lem:identifiability_implies_expl_an}, the causal effect is not identifiable and LOAD correctly returns this, as well as the locally valid parent adjustment sets (lines \ref{line:load:no_explan_start} to \ref{line:load:no_explan_end}). 

    If an explicit ancestral relation is found in either direction, then LOAD tests amenability (lines \ref{line:load:amenability_start} to \ref{line:load:amenability_end}) by employing \Cref{alg:is_amenable} for each sibling of the treatment (i.e., the target that is identified as the explicit ancestor of the other).     \Cref{thm:mb_by_mb} ensures that the local information used for testing amenability is correct.
    The test of amenability fails if a sibling is found to be adjacent to the outcome, shown to be correct by \Cref{lem:sib_notadj_outcome}, or if \Cref{lem:local_amenability} does not hold, in which case LOAD again returns the locally valid parent adjustment sets.

    Thus, LOAD returns that the causal effect is identifiable if and only if the CPDAG is amenable relative to the target pair, i.e., if the causal effect between the targets is identifiable.
\end{proof}

\subsection{Proof of Lem.~\ref{lemma:possde}}
\label{proof:possde}

We first report a definition of a general projection for CPDAGs with single treatment and outcome based on a set $\mathbf{F}$, that is inspired by Def.~17 by \citet{witte2020efficient}, but it focuses on CPDAGs and does not specify the set $\mathbf{F}$.

\begin{definition}[Projection for CPDAGs with single treatment and outcome given $\mathbf{F}$]
    Let $G$ be a CPDAG with nodes $\mathbf{V}$, and let $T, O \in \mathbf{V}$, and $\mathbf{F} \subset \mathbf{V} \setminus \{T,O\}$.    
    A projection $\hat{G}^{T,O}$ of $G$ is a graph with nodes $\mathbf{V} \setminus \mathbf{F}$ and edges as follows. For distinct nodes $W_i, W_j \in \mathbf{V} \setminus \mathbf{F}$,
    \begin{compactenum}
        \item $\hat{G}^{T,O}$ contains a directed edge $W_i \to W_j$ if and only if $G$ contains a directed path $W_i \to \cdots \to W_j$ on which all non-endpoint nodes are in $\mathbf{F}$,
        \item $\hat{G}^{T,O}$ contains a bi-directed edge $Wi \leftrightarrow Wj$ if and only if $G$ contains a path, with at least one non-endpoint node, of the form $W_i \gets \cdots \to W_j$ on which all non-endpoints are non-colliders and in $\mathbf{F}$,
        \item $\hat{G}^{T,O}$ contains an undirected edge $W_i - W_j$ if and only if $G$ contains $W_i - W_j$.
    \end{compactenum}
\end{definition}

In our setting with a CPDAG with a single treatment and outcome, the original forbidden projection by \citet{witte2020efficient} is then the projection with $\mathbf{F}= PossDe_G(PossCn_G(T,O)) \setminus \{T, O\}$, where $PossCn_G(T,O)$ are the nodes on possibly directed paths from $T$ to $O$, excluding $T$. If the causal effect between $T$ and $O$ is identifiable, i.e., the CPDAG $G$ is amenable to single treatment and outcome $(T,O)$, Prop.~19 and Lemma~20 in \citet{witte2020efficient} show that there are no bidirected edges in the projection, so the resulting graph is also a CPDAG.  Moreover, the parents of the outcome except the treatment are the optimal adjustment set.
We show that if we consider a larger set $\mathbf{D}=PossDe_G(T)\setminus \{T, O\}$ for the projection, the resulting graphs is also still a CPDAG and that the parents of the outcome except the treatment in this graph are still the optimal adjustment set.

\possde*

\begin{proof}
We start by applying a forbidden projection for treatment $T$ and outcome $O$ on the original amenable CPDAG $G$ and getting a graph $\tilde{G}^{T,O}$ with a node set $\mathcal{V}$. As shown in Prop.~22 by \citet{witte2020efficient} this is still a CPDAG, but without any possible causal nodes and any possible descendants of $O$ (which includes its siblings), except $O$ itself. The optimal adjustment set are the parents of $T$ except $T$ in $\tilde{G}^{T,O}$.
We can further project out the remaining variables in $\tilde{G}^{T,O}$, $\Delta = PossDe_{G}(T) \setminus PossDe_{G}(PossCn_{G}(T,O)) \cup \{T, O\}$. These will be still the remaining possible descendants of $T$ in this new graph, i.e., $PossDe_{\tilde{G}^{T,O}}(T)$, which are definite non-ancestors of $O$ and its possible ancestors, e.g., its parents, in both graphs, since otherwise they would be possible causal nodes in the original graph $G$ and removed in the first projection. As such, following \citet{schubert2025snap}, they can be safely marginalized out in $G^{T,O}$ without changing any orientation in the graph regarding $O$, its possible ancestors, e.g., parents, i.e., without changing the parent set of $O$, and hence the optimal adjustment set.

To prove that $G^{T,O}$ is still a CPDAG, we first show that it does not contain bi-directed edges, and then show that we do not add any other edge to the remaining variables with respect to the forbidden projection $\tilde{G}^{T,O}$.

In order for a bi-directed edge to appear in $G^{T,O}$, there have to exist two nodes $W_i, W_j \notin \mathbf{D}$ such that $W_i \gets \cdots \to W_j$ in $G$ and all non-endpoint nodes on this path are non-colliders and in $\mathbf{D}$. This means that both $W_i$ and $W_j$ have to be possible descendants of a node in $\mathbf{D}$. By transitivity, possible descendants of nodes in $\mathbf{D}$ are also in $\mathbf{D}$, except for $T$ and $O$, which are then the only possible candidates for $W_i$ and $W_j$.
On the other hand, there cannot exists such a path $W_i=T \gets \cdots \to O=W_j$ in $G$ such that all non-endpoint nodes on this path are by definition of $\mathbf{D}$ also possible descendants of $T$, since this would be a contradiction based on the orientation of the path.

We now show that the modified projection does not add any other directed or undirected edge compared to the original forbidden projection $\tilde{G}^{T,O}$.
A fully undirected path or (possibly) directed path from $T$ to other variables $V \in \mathcal{V}\setminus \{\Delta \cup O, T\}$ that passes through some variable $D \in \Delta$ would imply that $V$ is also in $\Delta$ by definition of possible descendants, which is a contradiction. Moreover, it is not possible to have a  directed path from $V$ to $T$ through some $D \in \Delta$, since this would contradict the fact that $\Delta$ are possible descendants of $T$.

We will therefore consider the case when there is a possibly directed path $p=(V, \dots, D_1, \dots, D_k, \dots, T)$ in $\mathcal{G}$ for $V \in \mathcal{V}\setminus \{\Delta \cup \{O, T\}\}$ where all non-endpoint nodes $D_1, \dots, D_k$ are in $\Delta$. The subpath $p'$ from $D_1 - D_2 \dots - D_k - T$ has to undirected, since it cannot be possibly directed from any $D_i$ to $T$ by definition of $\Delta$ and it cannot be possibly directed in the other direction by assumption. Additionally, the subpath $p''$ from $V \dots D_1$ cannot be completely undirected, otherwise $V$ would be part of $\Delta$. In order for $p'$ to stay undirected, we need that the variable $V_j$ on $p''$ closest to $D_1$ that has an orientation is part of a shielded triple $V_i \to V_j - D_1$ with additionally $V_i - D_1$ or $V_i \to D_1$ , otherwise the orientation will propagate further due to Meek's rules \citep{meek1995causal}, but we cannot use $V_i -D_1$ , since this would mean that $V_i \in \Delta$, because it is connected by an undirected path to $T$, which is a contradiction. So the only possible shield is $V_i \to D_1$, but then since we have $V_i  \to D_1 - D_2$ this would imply that there needs to be another shield $V_i \to D_2$, on which we can apply the same argument on each node $D_i$, until we would get a shield $V_i \to T$. This means that in the projection we would not need to add any edge from $V_i$ to $T$. A simplified case is when we only have three variables $V \to D - T$, which would require $V - T$, which is again a contradiction, or $V \to T$ which would mean we do not need to add a new edge after the projection.
The only possible path from $T$ to other variables $V \in \mathcal{V}\setminus \{\Delta \cup O, T\}$ that passes through some variable $D \in \Delta$ would have $D$ as a collider, which means we can safely remove this variable and still have a valid CPDAG, which means we do not need to add any edge to the original $\tilde{G}^{T,O}$.
\end{proof}

\subsection{Proof of Thm.~\ref{thm:load_oset}} 
\label{proof:load_oset}

\loadoset*
\begin{proof}
    By Thm. 3.13 of \citep{henckel2022graphical}, the optimal adjustment set is identifiable iff the causal effect between $T$ and $O$ is identifiable in the full CPDAG $G$. \Cref{lem:load_identifiability} shows that LOAD is sound and complete in determining the identifiability of causal effects.

    When the causal effect of $T$ on $O$ is identifiable, we show that LOAD returns the correct optimal adjustment set.
    To construct the estimated optimal adjustment, LOAD identifies the possible descendants of $T$ using \Cref{alg:is_possible_ancestor} (lines \ref{line:load:possde_start} to \ref{line:load:possde_end}).
    Since \Cref{alg:is_possible_ancestor} is sound and complete in finding possible ancestral relations as shown in \Cref{lem:alg_is_poss_an}, these are exactly $PossDe_G(X)$.
    
    When the causal effect of $T$ on $O$ is identifiable, then, by Lem.~\ref{lemma:possde}, the modified forbidden projection $G^{T,O} = G((\mathbf{V}\setminus PossDe_G(T)) \cup \{T,O\})$ is also a CPDAG and thus we can perform local causal discovery using MB-by-MB.
    We perform local causal discovery on the outcome $O$ in the modified forbidden projection $G^{T,O}$ at line \ref{line:load:forb_proj} to identify the parents of $O$ in the modified forbidden projection $G^{T,O}$ i.e., $Pa_{G^{T,O}}(O)$.
    Then, as shown by \citet{witte2020efficient}, $Pa_{G^{T,O}}(O) \setminus \{T\}$ is exactly the optimal adjustment set.
    The correctness of the local information throughout all of these steps is ensured by \Cref{thm:mb_by_mb}.
    Thus, LOAD is sound and complete in determining the optimal adjustment set.
\end{proof}

\section{POTENTIAL FURTHER OPTIMIZATIONS}
\label{sec:optimizations}

Here we describe a set of potential further optimizations to our method that we did not use in the main paper, due to the limited improvement in our experimental setting.
The LocalRelate (\Cref{alg:relate}) and LOAD (\Cref{alg:load}) algorithms run the local causal discovery sub-routine sequentially on several target variables with information cached and shared between different executions as shown by our implementation of MB-by-MB in \Cref{alg:mb-by-mb}.
Execution time can be further improved by running local causal discovery on multiple target nodes in parallel with the same shared memory already implemented.
In particular, local causal discovery can be run in parallel for the two target variables by \Cref{alg:relate} at lines~\ref{line:localrelate:G_x} and \ref{line:localrelate:G_y}.
Furthermore, local causal discovery on all siblings of the treatment in Step 3 of LOAD at line~\ref{line:load:amenability_start} can also be fully parallelized.
Note, however, that parallelizing Step 3 can result in performing more CI tests, as the current implementation terminates early when amenability fails.

Another potential optimization involves the last step of LOAD at line~\ref{line:load:forb_proj}, which runs local causal discovery around the outcome over the modified forbidden projection.
We conjecture that instead of local causal discovery, recovering only the Markov blanket of the outcome already identifies the parents, and thus the optimal adjustment set, without having to orient any edges.

\begin{lemma}\label{con:load_optimization}
    For a treatment $T$ and outcome $O$ in a graph $G$, the Markov blanket of $O$ in the modified forbidden projection $G^{T,O}$ directly identifies the optimal adjustment as
    $$
    Oset_G(T,O) = MB_{G^{T,O}}(O) \setminus \{T\}.
    $$
\end{lemma}

\begin{proof}
    In Lem.~\ref{lemma:possde} we showed that $Oset_G(T,O) = Pa_{G^{T,O}}(O) \setminus \{T\}$.
    Thus, we need to show that $Pa_{G^{T,O}}(O) = MB_{G^{T,O}}(O)$.
    It holds that $Pa_{G^{T,O}}(O) \subseteq  MB_{G^{T,O}}(O)$ by definition.
    
    For the other direction, we have to show that $MB_{G^{T,O}}(O) \subseteq Pa_{G^{T,O}}(O)$.
    Since $G^{T,O}$ is the modified forbidden projection, it does not contain any possible descendants of $T$.
    Hence, $O$ has no children or undirected siblings in $G^{T,O}$ and thus $MB_{G^{T,O}}(O)$ also does not contain any children or undirected siblings of $O$ in $G^{T,O}$.
    Furthermore, if $O$ has no children or undirected siblings in $G^{T,O}$, then it also does not have co-parents in $G^{T,O}$.
    This means that $MB_{G^{T,O}}(O)$ can only contain the parents of $O$ in $G^{T,O}$, i.e., $MB_{G^{T,O}}(O) \subseteq Pa_{G^{T,O}}(O)$.
\end{proof}

\Cref{con:load_optimization} allows us to run Markov blanket discovery once on the outcome variable at line~\ref{line:load:forb_proj} instead of possibly on multiple variables through MB-by-MB. On the other hand, in our experiments, the computational gain is very limited, so we did not implement this strategy in the main paper, also because this might make any potential extension to the causally insufficient case more complicated.

\section{EXPERIMENTAL DETAILS}
\label{sec:experimental_details}

For our experiments, we used the following libraries: igraph \citep{csardi2006igraph} (GNU GPL version 2 or later), networkx \citep{hagberg2008exploring} (3-Clause BSD License), bnlearn \citep{scutari2010bnlearn} (MIT License), pcalg \citep{kalisch2012pcalg} (GNU GPL version 2 or later), dagitty \citep{textor2016dagitty} (GNU GPL), causal-learn \citep{zheng2024causal} (MIT License) and RCD \citep{mokhtarian2024recursive} (BSD 2-Clause License).
In particular, we used the \ac{CI} test implementations from the causal-learn package \citep{zheng2024causal}.
We publish our code at \url{https://github.com/matyasch/load}.

All experiments were performed on AMD Rome CPUs, using 48 CPU cores and 84 GiB of memory.
We let each experiment run for at most 24 hours.
If an experiment over 100 seeds did not finish in the given time or did not fit in the given memory, then we do not report any results for it.

\subsection{Real-world networks}
\label{app:real_networks}

The MAGIC-NIAB network has 44 nodes, an average degree 3, and is parameterized by linear Gaussian data.
We show the network structure of MAGIC-NIAB in Fig.~\ref{fig:magic-niab}.
The Andes network has 223 nodes, an average degree 3.03, and is parameterized by binary data.
We show the network structure of Andes in Fig.~\ref{fig:andes}.

\begin{figure}
    \centering
    \includegraphics[width=.8\linewidth]{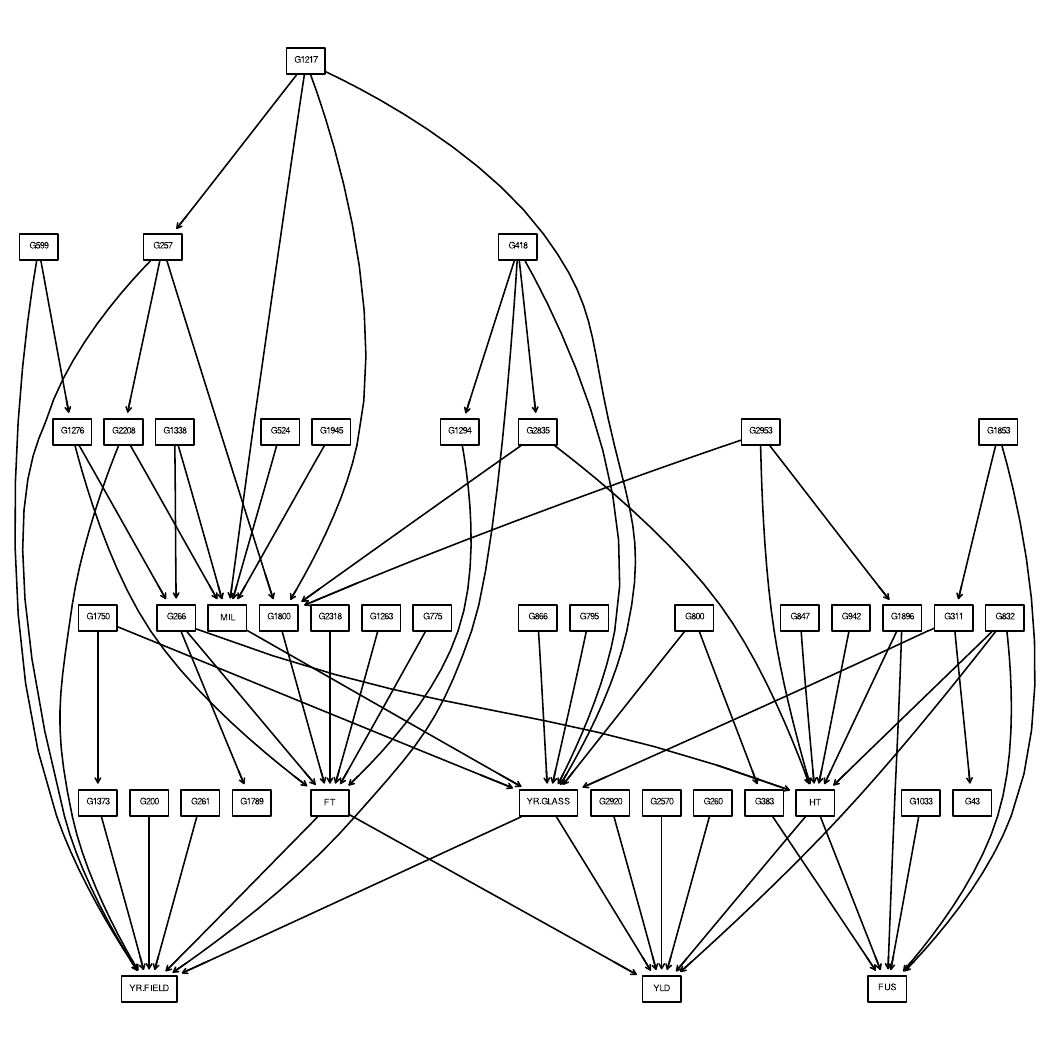}
    \caption{The MAGIC-NIAB network.}
    \label{fig:magic-niab}
\end{figure}

\begin{figure}
    \centering
    \includegraphics[width=\linewidth]{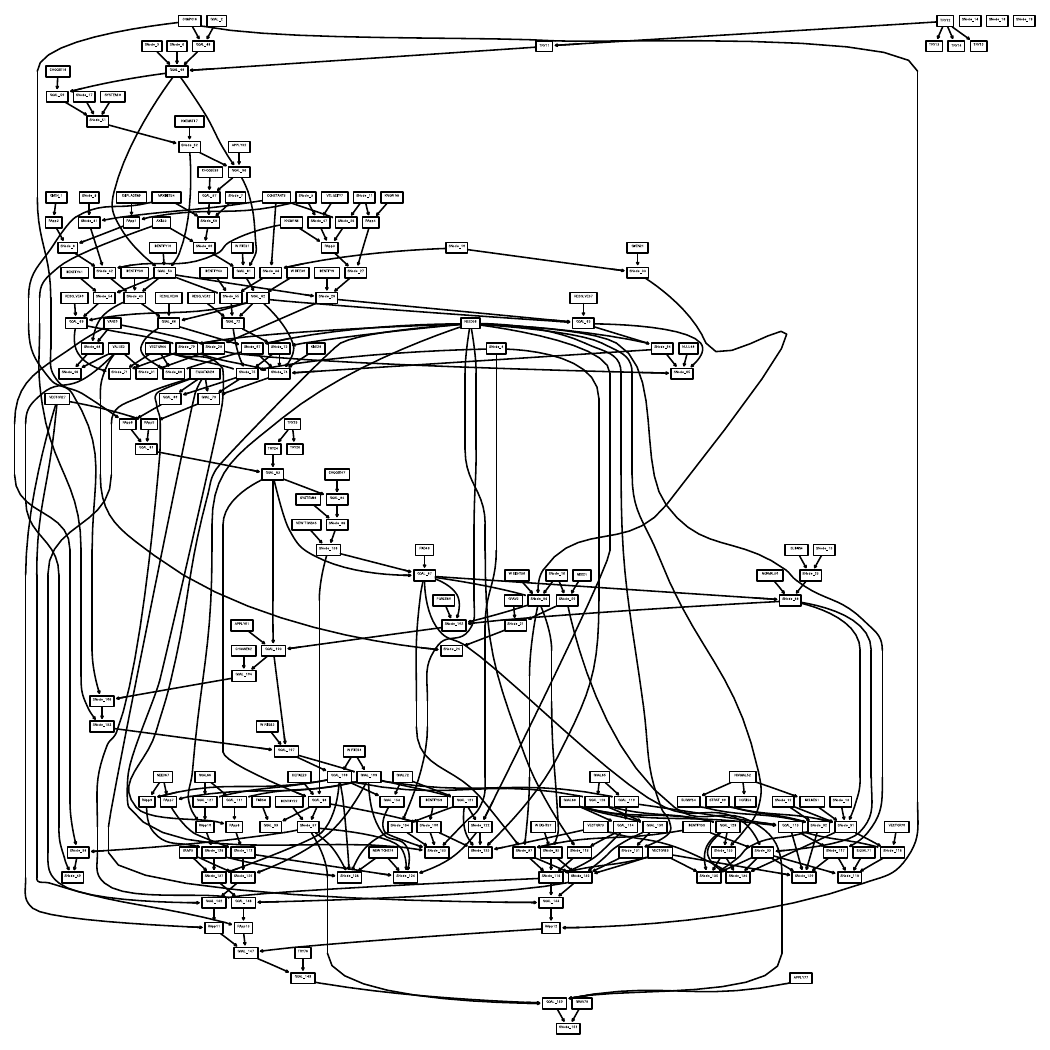}
    \caption{The Andes network.}
    \label{fig:andes}
\end{figure}

\section{EXTENDED RESULTS}
\label{sec:extended_results}

We repeat the main results presented in \Cref{fig:main_results} in tables with numbers shown explicitly for additional clarity, due to several overlapping results.
\Cref{tab:main_results_dsep} reports the main results for d-separation, \Cref{tab:main_results_fshz} for Fisher-Z tests and \Cref{tab:main_results_kci} KCI tests on linear Gaussian data, and \Cref{tab:main_results_gsq} for $G^2$ tests on binary data.

On top of the metrics already shown in \Cref{fig:main_results}, the tables also report the computation time of all algorithms over all experiments.
In the d-separation setting, the computation times show the same trends as the number of CI tests.
However, when using Fisher-Z or $G^2$ tests, the computation time of MARVEL increases substantially and the gap between the local methods (MB-by-MB, LDECC and LDP) and methods like PC and
SNAP($\infty$) gets smaller with the increasing number of nodes. LOAD performs as expected in-between local and global methods for lower numbers of nodes, but is slower than global methods on extremely large graphs.
This is due to MARVEL, MB-by-MB$^+$, LDECC$^+$ and LOAD all employing some form of Markov blanket search, which generally perform CI tests with large conditioning set sizes and require more computations. 

\begin{table*}[th]
\caption{Results over number of nodes with $n_{\mathbf{D}} = 10000$, $\overline{d} = 2$ and $d_{\max} = 10$ and target pairs such that one is an explicit ancestor of the other, for d-separation CI tests.}
\label{tab:main_results_dsep}
\centering
\footnotesize{
\begin{tabular}{@{}lrrrrrr@{}}
\toprule
                & \multicolumn{6}{c}{d-separation CI tests}          \\
\midrule
\# Nodes & \multicolumn{1}{c}{100} & \multicolumn{1}{c}{200} & \multicolumn{1}{c}{400} & \multicolumn{1}{c}{600} & \multicolumn{1}{c}{800} & \multicolumn{1}{c}{1000} \\
\midrule
                & \multicolumn{6}{c}{Number of CI tests ($\times 10^3$) $\downarrow$}   \\
\midrule
PC              & 9.00 $\pm$ 1.46 & 28.61 $\pm$ 3.04 & 96.81 $\pm$ 4.78 & 205.14 $\pm$ 5.19 & 353.48 $\pm$ 5.51 & 542.52 $\pm$ 6.66 \\
MB-by-MB$^+$    & 0.66 $\pm$ 0.30 & 1.14 $\pm$ 0.43 & 2.20 $\pm$ 0.71 & 3.16 $\pm$ 1.11 & 4.84 $\pm$ 1.86 & 5.64 $\pm$ 1.98  \\
MARVEL          & 5.48 $\pm$ 0.14 & 21.11 $\pm$ 0.37 & 82.07 $\pm$ 0.50 & - & - & - \\
LDECC$^+$       & 7.95 $\pm$ 6.90 & 23.00 $\pm$ 23.59 & 100.55 $\pm$ 86.36 & 179.14 $\pm$ 188.47 & 400.75 $\pm$ 324.45 & 566.69 $\pm$ 508.87 \\
LDP$^+$         & 0.88 $\pm$ 0.34 & 1.51 $\pm$ 0.47 & 2.96 $\pm$ 0.75 & 4.25 $\pm$ 1.15 & 6.29 $\pm$ 1.91 & 7.45 $\pm$ 2.06 \\
SNAP($\infty$)  & 5.11 $\pm$ 0.18 & 20.02 $\pm$ 0.14 & 79.98 $\pm$ 0.17 & 179.89 $\pm$ 0.21 & 319.74 $\pm$ 0.13 & 499.67 $\pm$ 0.15 \\
LOAD            & 0.86 $\pm$ 0.36 & 1.49 $\pm$ 0.53 & 2.94 $\pm$ 0.92 & 4.36 $\pm$ 1.39 & 6.41 $\pm$ 2.23 & 7.63 $\pm$ 2.47 \\
\midrule
                & \multicolumn{6}{c}{Computation Time (s) $\downarrow$}   \\
\midrule
PC              & 3.67 $\pm$ 0.71 & 19.67 $\pm$ 2.03 & 117.25 $\pm$ 11.66 & 334.99 $\pm$ 44.21 & 771.44 $\pm$ 43.63 & 1334.79 $\pm$ 200.84 \\
MB-by-MB$^+$    & 0.15 $\pm$ 0.07 & 0.42 $\pm$ 0.15 & 1.41 $\pm$ 0.46 & 3.11 $\pm$ 1.02 & 5.53 $\pm$ 2.13 & 8.13 $\pm$ 2.68  \\
MARVEL          & 2.06 $\pm$ 0.30 & 13.23 $\pm$ 2.70 & 113.14 $\pm$ 11.77 & - & - & -  \\
LDECC$^+$       & 2.71 $\pm$ 2.26 & 12.87 $\pm$ 11.98 & 99.77 $\pm$ 82.67 & 210.67 $\pm$ 212.49 & 725.48 $\pm$ 575.05 & 1120.81 $\pm$ 993.60 \\
LDP$^+$         & 0.58 $\pm$ 0.19 & 2.44 $\pm$ 0.42 & 7.69 $\pm$ 1.95 & 20.83 $\pm$ 3.59 & 50.77 $\pm$ 15.00 & 91.10 $\pm$ 19.53 \\
SNAP($\infty$)  & 3.07 $\pm$ 0.27 & 14.42 $\pm$ 2.33 & 93.44 $\pm$ 22.16 & 313.71 $\pm$ 40.02 & 786.31 $\pm$ 35.68 & 1514.65 $\pm$ 65.48 \\
LOAD            & 0.21 $\pm$ 0.08 & 0.55 $\pm$ 0.19 & 1.99 $\pm$ 0.61 & 4.15 $\pm$ 1.24 & 7.48 $\pm$ 2.57 & 12.31 $\pm$ 3.46 \\
\midrule
                & \multicolumn{6}{c}{F1 score of Oset $\uparrow$} \\
\midrule
PC              & 1.00 $\pm$ 0.00 & 1.00 $\pm$ 0.00 & 1.00 $\pm$ 0.00 & 1.00 $\pm$ 0.00 & 1.00 $\pm$ 0.00 & 1.00 $\pm$ 0.00 \\
MB-by-MB$^+$    & 0.00 $\pm$ 0.00 & 0.00 $\pm$ 0.00 & 0.00 $\pm$ 0.00 & 0.00 $\pm$ 0.00 & 0.00 $\pm$ 0.00 & 0.00 $\pm$ 0.00  \\
MARVEL          & 1.00 $\pm$ 0.00 & 1.00 $\pm$ 0.00 & 1.00 $\pm$ 0.00  & - & - & -  \\
LDECC$^+$       & 0.00 $\pm$ 0.00 & 0.00 $\pm$ 0.00 & 0.00 $\pm$ 0.00 & 0.00 $\pm$ 0.00 & 0.00 $\pm$ 0.00 & 0.00 $\pm$ 0.00 \\
LDP$^+$         & 0.01 $\pm$ 0.04 & 0.00 $\pm$ 0.00 & 0.00 $\pm$ 0.00 & 0.00 $\pm$ 0.00 & 0.00 $\pm$ 0.00 & 0.00 $\pm$ 0.00 \\
SNAP($\infty$)  & 1.00 $\pm$ 0.00 & 1.00 $\pm$ 0.00 & 1.00 $\pm$ 0.00 & 1.00 $\pm$ 0.00 & 1.00 $\pm$ 0.00 & 1.00 $\pm$ 0.00 \\
LOAD            & 1.00 $\pm$ 0.00 & 1.00 $\pm$ 0.00 & 1.00 $\pm$ 0.00 & 1.00 $\pm$ 0.00 & 1.00 $\pm$ 0.00 & 1.00 $\pm$ 0.00 \\
\midrule
                & \multicolumn{6}{c}{Intervention distance $\downarrow$} \\
\midrule
PC              & 0.004 $\pm$ 0.004 & 0.004 $\pm$ 0.003 & 0.003 $\pm$ 0.003 & 0.003 $\pm$ 0.003 & 0.003 $\pm$ 0.002 & 0.003 $\pm$ 0.003 \\
MB-by-MB$^+$    & 0.031 $\pm$ 0.035 & 0.019 $\pm$ 0.018 & 0.032 $\pm$ 0.034 & 0.026 $\pm$ 0.032 & 0.030 $\pm$ 0.031 & 0.031 $\pm$ 0.030 \\
MARVEL          & 0.004 $\pm$ 0.004 & 0.004 $\pm$ 0.003 & 0.003 $\pm$ 0.003  & - & - & - \\
LDECC$^+$       & 0.031 $\pm$ 0.035 & 0.019 $\pm$ 0.018 & 0.032 $\pm$ 0.034 & 0.026 $\pm$ 0.032 & 0.030 $\pm$ 0.031 & 0.031 $\pm$ 0.030 \\
LDP$^+$         & 0.011 $\pm$ 0.012 & 0.007 $\pm$ 0.005 & 0.008 $\pm$ 0.008 & 0.009 $\pm$ 0.010 & 0.009 $\pm$ 0.009 & 0.010 $\pm$ 0.010 \\
SNAP($\infty$)  & 0.004 $\pm$ 0.004 & 0.004 $\pm$ 0.003 & 0.003 $\pm$ 0.003 & 0.003 $\pm$ 0.003 & 0.003 $\pm$ 0.002 & 0.003 $\pm$ 0.003 \\
LOAD            & 0.004 $\pm$ 0.004 & 0.004 $\pm$ 0.003 & 0.003 $\pm$ 0.003 & 0.003 $\pm$ 0.003 & 0.003 $\pm$ 0.002 & 0.003 $\pm$ 0.003 \\
\bottomrule
\end{tabular}
}
\end{table*}

\begin{table*}[th]
\caption{Results over number of nodes with $n_{\mathbf{D}} = 10000$, $\overline{d} = 2$ and $d_{\max} = 10$ and target pairs such that one is an explicit ancestor of the other, for Fisher-Z CI tests on linear Gaussian data.}
\label{tab:main_results_fshz}
\centering
\footnotesize{
\begin{tabular}{@{}lrrrrrr@{}}
\toprule
                & \multicolumn{6}{c}{Fisher-Z CI tests}          \\
\midrule
\# Nodes & \multicolumn{1}{c}{100} & \multicolumn{1}{c}{200} & \multicolumn{1}{c}{400} & \multicolumn{1}{c}{600} & \multicolumn{1}{c}{800} & \multicolumn{1}{c}{1000} \\
\midrule
                & \multicolumn{6}{c}{Number of CI tests ($\times 10^3$) $\downarrow$}   \\
\midrule
PC              & 8.28 $\pm$ 0.95 & 27.16 $\pm$ 1.75 & 96.02 $\pm$ 2.49 & 206.76 $\pm$ 3.04 & 359.64 $\pm$ 3.42 & 555.47 $\pm$ 4.69 \\
MB-by-MB$^+$    & 0.75 $\pm$ 0.35 & 1.37 $\pm$ 0.53 & 2.75 $\pm$ 0.77 & 4.39 $\pm$ 1.51 & 6.53 $\pm$ 2.40 & 8.20 $\pm$ 3.00 \\
MARVEL          & 5.65 $\pm$ 0.16 & 21.99 $\pm$ 0.42 & 87.64 $\pm$ 1.08 & - & - & - \\
LDECC$^+$       & 4.93 $\pm$ 4.86 & 16.29 $\pm$ 18.11 & 61.89 $\pm$ 69.00 & 109.95 $\pm$ 143.96 & 233.69 $\pm$ 273.58 \\
LDP$^+$         & 0.96 $\pm$ 0.37 & 1.75 $\pm$ 0.55 & 3.56 $\pm$ 0.77 & 5.64 $\pm$ 1.54 & 8.11 $\pm$ 2.44 & 10.23 $\pm$ 3.02 \\
SNAP($\infty$)  & 5.01 $\pm$ 0.06 & 19.93 $\pm$ 0.04 & 79.81 $\pm$ 0.01 & 179.70 $\pm$ 0.01 & 319.60 $\pm$ 0.00 & 499.50 $\pm$ 0.00 \\
LOAD            & 0.95 $\pm$ 0.40 & 1.79 $\pm$ 0.62 & 3.77 $\pm$ 0.95 & 6.10 $\pm$ 1.83 & 8.77 $\pm$ 2.78 & 11.15 $\pm$ 3.37 \\
\midrule
                & \multicolumn{6}{c}{Computation Time (s) $\downarrow$}   \\
\midrule
PC              & 1.54 $\pm$ 0.16 & 5.08 $\pm$ 0.28 & 18.66 $\pm$ 0.43 & 40.85 $\pm$ 0.51 & 72.80 $\pm$ 0.62 & 117.75 $\pm$ 0.93 \\
MB-by-MB$^+$    & 0.14 $\pm$ 0.07 & 0.26 $\pm$ 0.10 & 0.65 $\pm$ 0.32 & 0.93 $\pm$ 0.35 & 1.46 $\pm$ 0.60 & 1.91 $\pm$ 0.72 \\
MARVEL          & 1.25 $\pm$ 0.04 & 7.97 $\pm$ 0.10 & 77.41 $\pm$ 0.40 & - & - & -  \\
LDECC$^+$       & 1.21 $\pm$ 1.16 & 4.76 $\pm$ 5.14 & 24.23 $\pm$ 26.15 & 54.63 $\pm$ 69.10 & 134.30 $\pm$ 152.48 & -  \\
LDP$^+$         & 0.24 $\pm$ 0.08 & 0.68 $\pm$ 0.19 & 3.22 $\pm$ 0.83 & 9.29 $\pm$ 3.06 & 22.78 $\pm$ 8.02 & 43.69 $\pm$ 15.19 \\
SNAP($\infty$)  & 1.02 $\pm$ 0.04 & 4.26 $\pm$ 0.08 & 18.16 $\pm$ 0.51 & 43.38 $\pm$ 1.32 & 82.60 $\pm$ 2.82 & 137.85 $\pm$ 5.39 \\
LOAD            & 0.18 $\pm$ 0.08 & 0.35 $\pm$ 0.13 & 0.76 $\pm$ 0.19 & 1.38 $\pm$ 0.46 & 1.90 $\pm$ 0.66 & 2.52 $\pm$ 0.77 \\
\midrule
                & \multicolumn{6}{c}{F1 score of Oset $\uparrow$} \\
\midrule
PC              & 0.74 $\pm$ 0.40 & 0.80 $\pm$ 0.37 & 0.85 $\pm$ 0.32 & 0.69 $\pm$ 0.41 & 0.60 $\pm$ 0.46 & 0.66 $\pm$ 0.43 \\
MB-by-MB$^+$    & 0.00 $\pm$ 0.00 & 0.00 $\pm$ 0.00 & 0.00 $\pm$ 0.00 & 0.00 $\pm$ 0.00 & 0.00 $\pm$ 0.00 & 0.00 $\pm$ 0.00  \\
MARVEL          & 0.91 $\pm$ 0.27 & 0.92 $\pm$ 0.25 & 0.81 $\pm$ 0.37  & - & - & -  \\
LDECC$^+$       & 0.00 $\pm$ 0.00 & 0.00 $\pm$ 0.00 & 0.00 $\pm$ 0.00 & 0.00 $\pm$ 0.00 & 0.00 $\pm$ 0.00 & -  \\
LDP$^+$         & 0.00 $\pm$ 0.01 & 0.00 $\pm$ 0.00 & 0.00 $\pm$ 0.00 & 0.00 $\pm$ 0.00 & 0.00 $\pm$ 0.00 & 0.00 $\pm$ 0.00 \\
SNAP($\infty$)  & 0.56 $\pm$ 0.41 & 0.31 $\pm$ 0.38 & 0.16 $\pm$ 0.30 & 0.03 $\pm$ 0.11 & 0.00 $\pm$ 0.02 & 0.00 $\pm$ 0.00 \\
LOAD            & 0.97 $\pm$ 0.12 & 0.99 $\pm$ 0.06 & 0.98 $\pm$ 0.09 & 0.94 $\pm$ 0.18 & 0.95 $\pm$ 0.15 & 0.96 $\pm$ 0.13 \\
\midrule
                & \multicolumn{6}{c}{Intervention distance $\downarrow$} \\
\midrule
PC              & 0.152 $\pm$ 0.313 & 0.038 $\pm$ 0.085 & 0.053 $\pm$ 0.147 & 0.078 $\pm$ 0.188 & 0.099 $\pm$ 0.225 & 0.131 $\pm$ 0.290 \\
MB-by-MB$^+$    & 0.092 $\pm$ 0.081 & 0.067 $\pm$ 0.083 & 0.120 $\pm$ 0.130 & 0.094 $\pm$ 0.112 & 0.086 $\pm$ 0.096 & 0.113 $\pm$ 0.114 \\
MARVEL          & 0.032 $\pm$ 0.057 & 0.014 $\pm$ 0.017 & 0.036 $\pm$ 0.076  & - & - & - \\
LDECC$^+$       & 0.181 $\pm$ 0.312 & 0.080 $\pm$ 0.100 & 0.202 $\pm$ 0.293 & 0.189 $\pm$ 0.290 & 0.132 $\pm$ 0.201 & -\\
LDP$^+$         & 0.037 $\pm$ 0.044 & 0.023 $\pm$ 0.025 & 0.037 $\pm$ 0.045 & 0.029 $\pm$ 0.037 & 0.029 $\pm$ 0.030 & 0.041 $\pm$ 0.042 \\
SNAP($\infty$)  & 0.111 $\pm$ 0.194 & 0.145 $\pm$ 0.199 & 0.365 $\pm$ 0.327 & 0.479 $\pm$ 0.263 & 0.594 $\pm$ 0.321 & 0.586 $\pm$ 0.321 \\
LOAD            & 0.010 $\pm$ 0.030 & 0.003 $\pm$ 0.003 & 0.004 $\pm$ 0.006 & 0.004 $\pm$ 0.004 & 0.005 $\pm$ 0.007 & 0.005 $\pm$ 0.005 \\
\bottomrule
\end{tabular}
}
\end{table*}

\begin{table*}[th]
\caption{Results over number of nodes with $n_{\mathbf{D}} = 1000$, $\overline{d} = 2$ and $d_{\max} = 10$ and target pairs such that one is an explicit ancestor of the other, for KCI CI tests on binary data.}
\label{tab:main_results_kci}
\centering
\small{
\begin{tabular}{@{}lrrr@{}}
\toprule
                & \multicolumn{3}{c}{KCI tests}          \\
\midrule
\# Nodes & \multicolumn{1}{c}{10} & \multicolumn{1}{c}{15} & \multicolumn{1}{c}{20} \\
\midrule
                & \multicolumn{3}{c}{Number of CI tests $\downarrow$}   \\
\midrule
PC              & 134.33 $\pm$ 31.11 & 282.21 $\pm$ 70.97 & 454.97 $\pm$ 86.31 \\
MB-by-MB$^+$    & 118.46 $\pm$ 46.16 & 194.31 $\pm$ 95.69 & 211.46 $\pm$ 92.08 \\
MARVEL          & 85.41 $\pm$ 22.67 & 196.51 $\pm$ 49.77 & 322.07 $\pm$ 62.74\\
LDECC$^+$       & 160.88 $\pm$ 80.08 & 296.03 $\pm$ 164.97 & 384.71 $\pm$ 220.95\\
LDP$^+$         & 136.20 $\pm$ 47.53 & 232.32 $\pm$ 104.91 & 258.59 $\pm$ 94.33 \\
SNAP($\infty$)  & 92.52 $\pm$ 35.49 & 156.89 $\pm$ 43.81 & 242.79 $\pm$ 53.92 \\
LOAD            & 127.82 $\pm$ 47.10 & 211.42 $\pm$ 99.53 & 239.56 $\pm$ 93.61 \\
\midrule
                & \multicolumn{3}{c}{Computation Time (s) $\downarrow$}   \\
\midrule
PC              & 50.37 $\pm$ 16.94 & 98.77 $\pm$ 38.19 & 223.26 $\pm$ 70.62 \\
MB-by-MB$^+$    & 56.27 $\pm$ 22.72 & 126.40 $\pm$ 64.11 & 105.36 $\pm$ 46.32 \\
MARVEL          & 43.77 $\pm$ 10.80 & 100.21 $\pm$ 23.77 & 180.03 $\pm$ 29.99 \\
LDECC$^+$       & 71.71 $\pm$ 40.32 & 126.50 $\pm$ 81.87 & 159.61 $\pm$ 97.09 \\
LDP$^+$         & 65.01 $\pm$ 25.20 & 105.60 $\pm$ 52.09 & 119.03 $\pm$ 48.46 \\
SNAP($\infty$)  & 26.98 $\pm$ 18.91 & 33.02 $\pm$ 24.64 & 36.24 $\pm$ 30.05 \\
LOAD            & 62.11 $\pm$ 24.26 & 101.09 $\pm$ 49.99 & 162.64 $\pm$ 67.13 \\
\midrule
                & \multicolumn{3}{c}{F1 score of Oset $\uparrow$} \\
\midrule
PC              & 0.52 $\pm$ 0.46 & 0.32 $\pm$ 0.45 & 0.52 $\pm$ 0.47 \\
MB-by-MB$^+$    & 0.02 $\pm$ 0.13 & 0.06 $\pm$ 0.20 & 0.00 $\pm$ 0.00 \\
MARVEL          & 0.74 $\pm$ 0.41 & 0.44 $\pm$ 0.48 & 0.55 $\pm$ 0.48 \\
LDECC$^+$       & 0.03 $\pm$ 0.14 & 0.03 $\pm$ 0.16 & 0.02 $\pm$ 0.11 \\
LDP$^+$         & 0.14 $\pm$ 0.34 & 0.15 $\pm$ 0.35 & 0.11 $\pm$ 0.31 \\
SNAP($\infty$)  & 0.52 $\pm$ 0.47 & 0.35 $\pm$ 0.45 & 0.40 $\pm$ 0.46 \\
LOAD            & 0.74 $\pm$ 0.41 & 0.61 $\pm$ 0.47 & 0.75 $\pm$ 0.41 \\
\midrule
                & \multicolumn{3}{c}{Intervention distance $\downarrow$} \\
\midrule
PC              & 0.35 $\pm$ 0.44 & 0.51 $\pm$ 0.43 & 0.37 $\pm$ 0.47 \\
MB-by-MB$^+$    & 0.20 $\pm$ 0.33 & 0.29 $\pm$ 0.38 & 0.18 $\pm$ 0.31 \\
MARVEL          & 0.18 $\pm$ 0.35 & 0.26 $\pm$ 0.32 & 0.27 $\pm$ 0.42 \\
LDECC$^+$       & 0.29 $\pm$ 0.38 & 0.41 $\pm$ 0.45 & 0.29 $\pm$ 0.45 \\
LDP$^+$         & 0.15 $\pm$ 0.29 & 0.24 $\pm$ 0.33 & 0.16 $\pm$ 0.30 \\
SNAP($\infty$)  & 0.34 $\pm$ 0.39 & 0.51 $\pm$ 0.41 & 0.38 $\pm$ 0.42 \\
LOAD            & 0.17 $\pm$ 0.32 & 0.28 $\pm$ 0.39 & 0.13 $\pm$ 0.29 \\
\bottomrule
\end{tabular}
}
\end{table*}

\begin{table*}[th]
\caption{Results over number of nodes with $n_{\mathbf{D}} = 10000$, $\overline{d} = 2$ and $d_{\max} = 10$ and target pairs such that one is an explicit ancestor of the other, for $G^2$ CI tests on binary data.}
\label{tab:main_results_gsq}
\centering
\footnotesize{
\begin{tabular}{@{}lrrrrrr@{}}
\toprule
                & \multicolumn{6}{c}{$G^2$ CI tests}          \\
\midrule
\# Nodes & \multicolumn{1}{c}{100} & \multicolumn{1}{c}{200} & \multicolumn{1}{c}{400} & \multicolumn{1}{c}{600} & \multicolumn{1}{c}{800} & \multicolumn{1}{c}{1000} \\
\midrule
                & \multicolumn{6}{c}{Number of CI tests ($\times 10^3$) $\downarrow$}   \\
\midrule
PC              & 6.15 $\pm$ 0.21 & 22.77 $\pm$ 0.37 & 87.79 $\pm$ 0.64 & 195.23 $\pm$ 0.77 & 345.68 $\pm$ 0.90 & 539.23 $\pm$ 1.48 \\
MB-by-MB$^+$    & 1.04 $\pm$ 0.25 & 1.71 $\pm$ 0.49 & 2.56 $\pm$ 0.65 & 3.20 $\pm$ 0.78 & 3.95 $\pm$ 0.95 & 4.71 $\pm$ 1.04 \\
MARVEL          & 4.95 $\pm$ 0.00 & 19.90 $\pm$ 0.00 & 79.80 $\pm$ 0.00 & - & - & - \\
LDECC$^+$        & - & - & - & - & - & -  \\
LDP$^+$         & 1.13 $\pm$ 0.25 & 1.93 $\pm$ 0.47 & 3.03 $\pm$ 0.73 & 3.96 $\pm$ 0.93 & 5.11 $\pm$ 1.19 & 6.23 $\pm$ 1.21 \\
SNAP($\infty$)  & 4.95 $\pm$ 0.00 & 19.90 $\pm$ 0.00 & 79.80 $\pm$ 0.00 & 179.70 $\pm$ 0.00 & 319.60 $\pm$ 0.00 & 499.50 $\pm$ 0.00 \\
LOAD            & 1.19 $\pm$ 0.33 & 2.00 $\pm$ 0.55 & 3.08 $\pm$ 0.91 & 4.25 $\pm$ 1.24 & 5.19 $\pm$ 1.50 & 6.10 $\pm$ 1.47
 \\
\midrule
                & \multicolumn{6}{c}{Computation Time (s) $\downarrow$}   \\
\midrule
PC              & 3.93 $\pm$ 0.57 & 13.84 $\pm$ 2.75 & 57.46 $\pm$ 17.63 & 129.28 $\pm$ 11.20 & 229.50 $\pm$ 37.73 & 378.14 $\pm$ 24.23 \\
MB-by-MB$^+$    & 1.87 $\pm$ 0.57 & 4.16 $\pm$ 1.26 & 8.99 $\pm$ 2.71 & 16.22 $\pm$ 5.43 & 16.29 $\pm$ 3.78 & 20.66 $\pm$ 4.55 \\
MARVEL          & 153.59 $\pm$ 26.30 & 1068.25 $\pm$ 206.65 & 9074.14 $\pm$ 800.29 & - & - & -  \\
LDECC$^+$        & - & - & - & - & - & -  \\
LDP$^+$         & 3.16 $\pm$ 0.73 & 7.85 $\pm$ 1.49 & 19.39 $\pm$ 2.91 & 32.80 $\pm$ 8.03 & 64.22 $\pm$ 9.30 & 104.63 $\pm$ 17.39 \\
SNAP($\infty$)  & 2.95 $\pm$ 0.41 & 13.02 $\pm$ 1.34 & 55.38 $\pm$ 4.29 & 130.66 $\pm$ 6.99 & 249.79 $\pm$ 6.83 & 382.60 $\pm$ 48.54 \\
LOAD            & 2.03 $\pm$ 0.74 & 4.65 $\pm$ 1.38 & 9.35 $\pm$ 2.87 & 15.24 $\pm$ 4.40 & 19.83 $\pm$ 5.83 & 24.64 $\pm$ 6.55 \\
\midrule
                & \multicolumn{6}{c}{F1 score of Oset $\uparrow$} \\
\midrule
PC              & 0.34 $\pm$ 0.42 & 0.31 $\pm$ 0.40 & 0.28 $\pm$ 0.40 & 0.21 $\pm$ 0.36 & 0.26 $\pm$ 0.40 & 0.19 $\pm$ 0.34 \\
MB-by-MB$^+$    & 0.00 $\pm$ 0.00 & 0.00 $\pm$ 0.00 & 0.00 $\pm$ 0.00 & 0.00 $\pm$ 0.00 & 0.00 $\pm$ 0.00 & 0.00 $\pm$ 0.00 \\
MARVEL          & 0.00 $\pm$ 0.00 & 0.00 $\pm$ 0.00 & 0.00 $\pm$ 0.00  & - & - & -  \\
LDECC$^+$        & - & - & - & - & - & -  \\
LDP$^+$         & 0.00 $\pm$ 0.00 & 0.00 $\pm$ 0.00 & 0.00 $\pm$ 0.00 & 0.00 $\pm$ 0.00 & 0.00 $\pm$ 0.00 & 0.00 $\pm$ 0.00 \\
SNAP($\infty$)  & 0.09 $\pm$ 0.23 & 0.00 $\pm$ 0.00 & 0.00 $\pm$ 0.00 & 0.00 $\pm$ 0.00 & 0.00 $\pm$ 0.00 & 0.00 $\pm$ 0.00 \\
LOAD            & 0.39 $\pm$ 0.44 & 0.23 $\pm$ 0.35 & 0.24 $\pm$ 0.39 & 0.13 $\pm$ 0.27 & 0.14 $\pm$ 0.32 & 0.06 $\pm$ 0.17 \\
\midrule
                & \multicolumn{6}{c}{Intervention distance $\downarrow$} \\
\midrule
PC              & 0.018 $\pm$ 0.033 & 0.017 $\pm$ 0.027 & 0.022 $\pm$ 0.042 & 0.019 $\pm$ 0.036 & 0.016 $\pm$ 0.029 & 0.021 $\pm$ 0.037 \\
MB-by-MB$^+$    & 0.022 $\pm$ 0.041 & 0.026 $\pm$ 0.038 & 0.033 $\pm$ 0.055 & 0.031 $\pm$ 0.056 & 0.038 $\pm$ 0.058 & 0.035 $\pm$ 0.052 \\
MARVEL          & 0.067 $\pm$ 0.059 & 0.055 $\pm$ 0.045 & 0.067 $\pm$ 0.064  & - & - & - \\
LDECC$^+$        & - & - & - & - & - & -  \\
LDP$^+$         & 0.022 $\pm$ 0.041 & 0.025 $\pm$ 0.037 & 0.032 $\pm$ 0.054 & 0.031 $\pm$ 0.054 & 0.038 $\pm$ 0.061 & 0.036 $\pm$ 0.053 \\
SNAP($\infty$)  & 0.045 $\pm$ 0.055 & 0.050 $\pm$ 0.046 & 0.063 $\pm$ 0.059 & 0.054 $\pm$ 0.057 & 0.058 $\pm$ 0.055 & 0.054 $\pm$ 0.046 \\
LOAD            & 0.021 $\pm$ 0.040 & 0.024 $\pm$ 0.037 & 0.032 $\pm$ 0.055 & 0.031 $\pm$ 0.057 & 0.039 $\pm$ 0.061 & 0.036 $\pm$ 0.052 \\
\bottomrule
\end{tabular}
}
\end{table*}

\section{ABLATIONS}
\label{sec:ablations}

\subsection{CI Testing Errors}
\label{sec:ci_errors}

Our results in \Cref{fig:main_results} and \Cref{sec:extended_results} indicate that identifying the optimal adjustment set is much more challenging in the binary data setting compared to linear Gaussian data for all algorithms.
This can be attributed to the fact that Fisher-Z tests can more easily compute partial correlation based on the inverse of the sample covariance matrix, which is not possible for the $G^2$ tests on binary data.
Instead, tests for discrete data, such as the $G^2$ test, have to use the observed counts of samples with a given combination of values for all of the variables in the conditioning set for each test, resulting in higher sample complexity.
Intuitively, the larger this set is, the smaller the number of samples for each combination, which will give us less accurate results with lower sample sizes

To show empirically that in our settings the $G^2$ tests are less accurate than Fisher-Z tests, we perform an ablation study under the same setting as \Cref{fig:main_results} with 100, 200 and 400 nodes, where we report the percentage of erroneous CI tests results.
Our results in \Cref{fig:ci_err} show that $G^2$ tests indeed have indeed many more errors than Fisher-Z tests in our data.

\subsection{Errors in Local Neighborhoods}
\label{sec:nb_error}

The performance of LOAD decreases slightly with more structural errors in the local neighborhoods estimated by the local causal discovery subroutine.
We show this empirically in \Cref{fig:ci_nb_err} where we report the average number of erroneous relations in the estimated local neighborhoods (i.e., the size of the symmetric difference between the estimated and the true parents, children and undirected neighbors).
We evaluate for both Fisher-Z tests on linear Gaussian data and $G^2$ tests on binary data, over various samples sizes and $n_{\mathbf{V}} = 200$, $\overline{d} = 2$ and $d_{\max} = 10$.
As expected, fewer errors in the local neighborhoods lead to better recovery of the optimal adjustment set (\Cref{fig:ci_f1}).
Notably, even with an average of almost two erroneous node in every estimation of a local neighborhood for 500 samples with Fisher-Z CI tests, the F1 score (and intervention distance as shown in \Cref{fig:samples}) of LOAD is still better than most baselines.

\begin{figure*}[ht!]
    \centering
    \begin{subfigure}[c]{\linewidth}
        \centering
        \includegraphics[width=.25\linewidth]{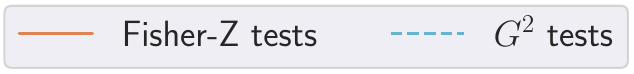}
    \end{subfigure}
    \begin{subfigure}[c]{.32\linewidth}
        \includegraphics[width=\linewidth]{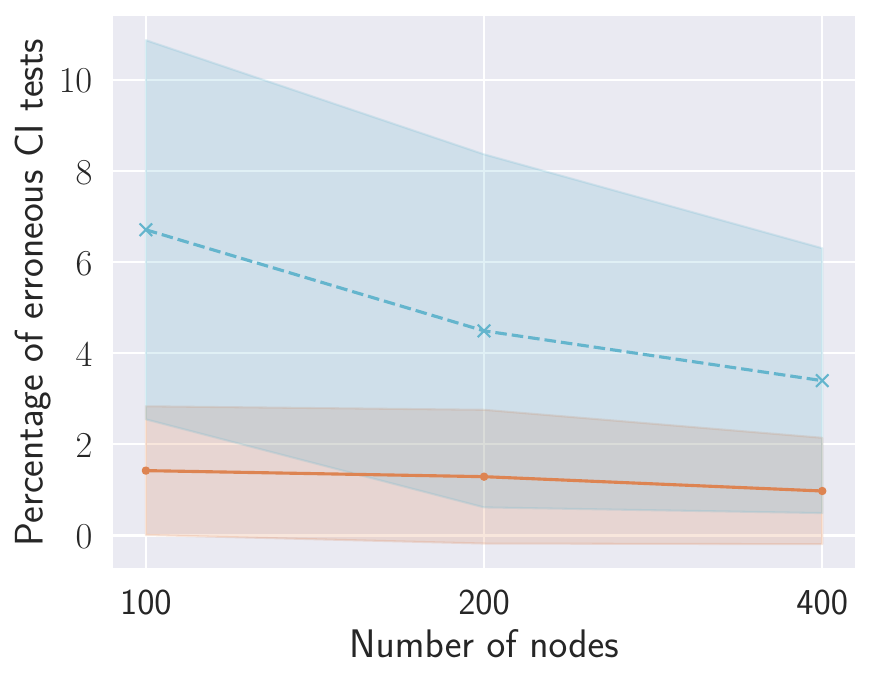}
        \caption{CI test errors.}
        \label{fig:ci_err}
    \end{subfigure}
    \begin{subfigure}[c]{.32\linewidth}
        \includegraphics[width=\linewidth]{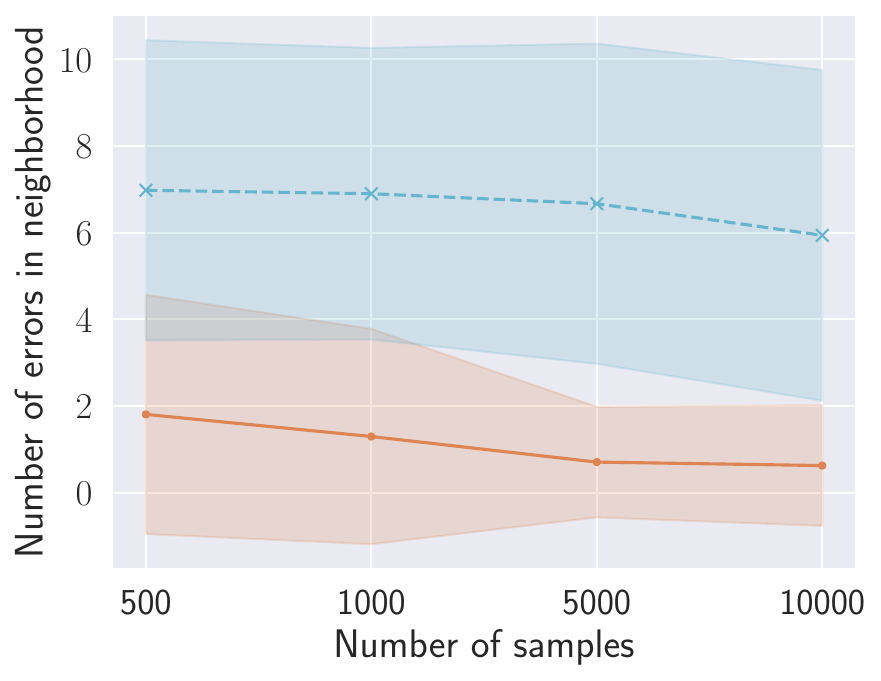}
        \caption{Neighborhood errors.}
        \label{fig:ci_nb_err}
    \end{subfigure}
    \begin{subfigure}[c]{.32\linewidth}
        \includegraphics[width=\linewidth]{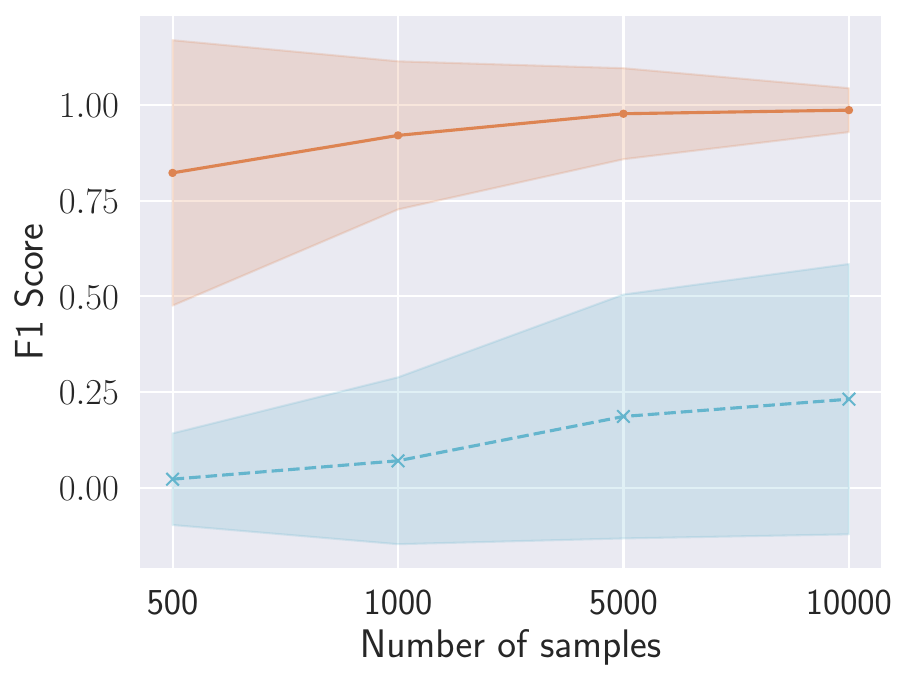}
        \caption{F1 Score.}
        \label{fig:ci_f1}
    \end{subfigure}
    \caption{
    Errors in LOAD for Fisher-Z tests on linear Gaussian data and for $G^2$ tests on binary data.
    \Cref{fig:ci_err}: Percentage of erroneous CI tests results over various number of nodes, with $n_{\mathbf{D}} = 10000$.
    \Cref{fig:ci_nb_err} and \Cref{fig:ci_f1}: Number of erroneous neighbors estimated by local causal discovery and F1 score of estimating the optimal adjustment set over various numbers of samples and $n_{\mathbf{V}} = 200$.
    For all figures $\overline{d} = 2$, $d_{\max} = 10$ and target pairs such that one is an explicit ancestor of the other.
    The shadow area denotes the range of the standard deviation.}
    \label{fig:ci_diff}
\end{figure*}

\subsection{Performance of Steps}
\label{sec:steps}
In this section we study the performance of the different steps of LOAD.
In general, the correctness of the first two steps should highly influence the performance of later steps, since if LOAD identifies the wrong causal relations between the targets in Step 1 or make a mistake in identifying if a causal effect is identifiable in Step 2, this will change the flow of the complete algorithm.
Similarly Steps 3 and 4 which together identify mediators are crucial to identifying the correct optimal adjustment set.
On the other hand, we assume that the last 2 steps are less important in terms of the accuracy of the causal effect estimation because in principle the output could still be potentially a valid set, which would allow the causal effect estimation to still be unbiased, although with a higher variance.

We empirically test this using linear Gaussian data with different sample sizes to study different error rates in each phase.
We report the precision and recall scores of each step.
In particular, we measure how well Step 1 can recover the causal relation between the targets, Step 2 can determine identifiability, how well Steps 3 can identify possible descendants, and finally how Step 4 can recover the optimal adjustment set.
In case LOAD determines the causal relation or the identifiability of the causal effect incorrectly, then later steps automatically have a precision and recall of 0.

Our results in \Cref{fig:steps} show that recovering all and only the correct possible descendants of the treatment in Step 3 is the biggest challenge for LOAD.
Note, that this is generally not because LOAD determines a wrong causal relation or effect identifiability, as the precision and recall of Steps 1 and 2 are high.
In fact, Step 2 achieves a perfect precision at all sample sizes, meaning that it never predicts that a causal effect is identifiable when it is not.
Together with the high recall of Step 2, they demonstrate that \Cref{alg:is_amenable} can robustly determine identifiability across different sample sizes.
Furthermore, the precision and recall of Step 4 is also high, indicating that the recovered parents of the outcome from local causal discovery even under an imperfect forbidden projection can successfully recover optimal adjustment sets.

\begin{figure*}[ht!]
    \centering
    \begin{subfigure}[b]{\linewidth}
        \centering
        \includegraphics[width=.25\linewidth]{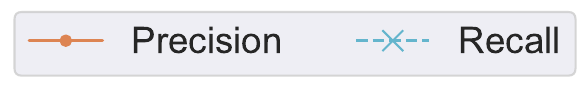}
    \end{subfigure}
    \begin{subfigure}[b]{\linewidth}
        \begin{subfigure}[b]{0.24\linewidth}
            \centering
            \caption*{500 samples}
            \includegraphics[width=.9\linewidth]{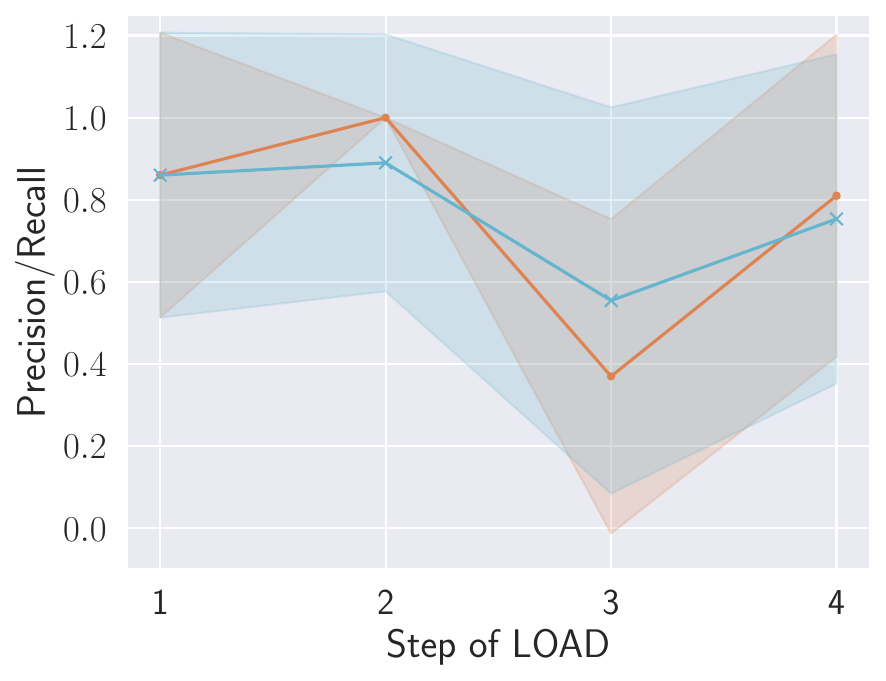}
        \end{subfigure}
        \begin{subfigure}[b]{0.24\linewidth}
            \centering
            \caption*{1000 samples}
            \includegraphics[width=.9\linewidth]{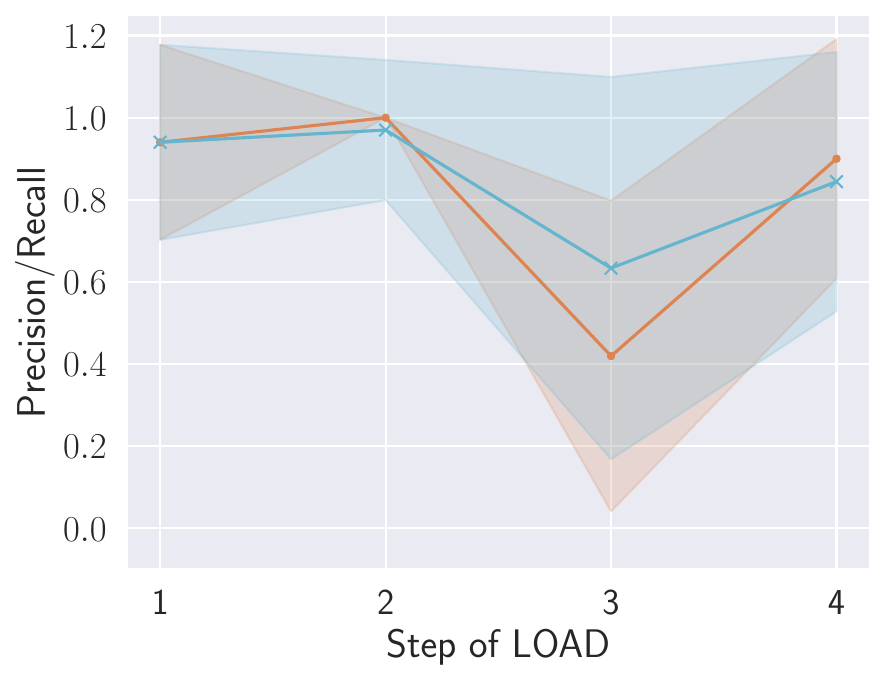}
        \end{subfigure}
        \begin{subfigure}[b]{0.24\linewidth}
            \centering
            \caption*{5000 samples}
            \includegraphics[width=.9\linewidth]{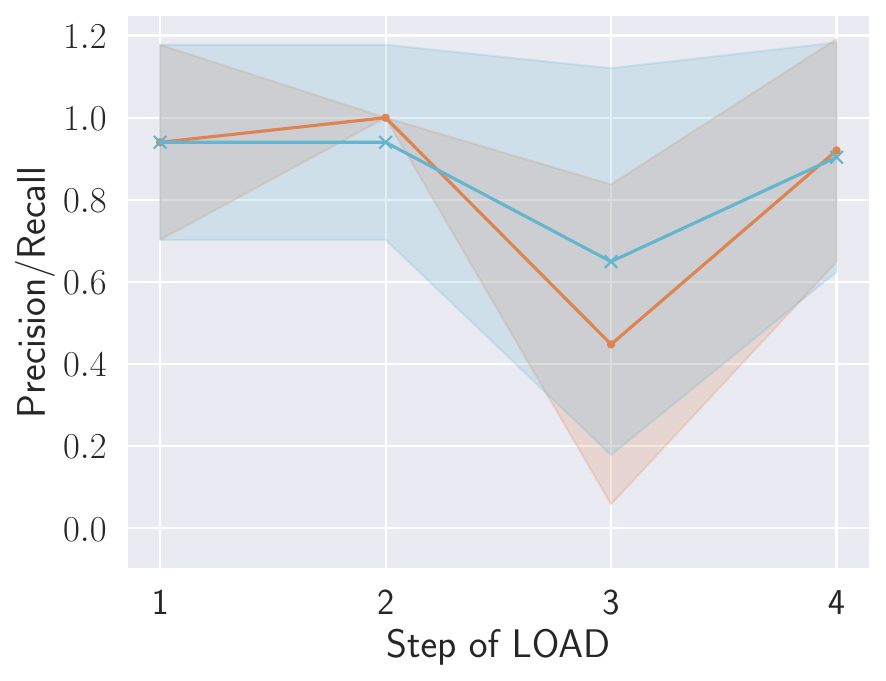}
        \end{subfigure}
        \begin{subfigure}[b]{0.24\linewidth}
            \centering
            \caption*{10000 samples}
            \includegraphics[width=.9\linewidth]{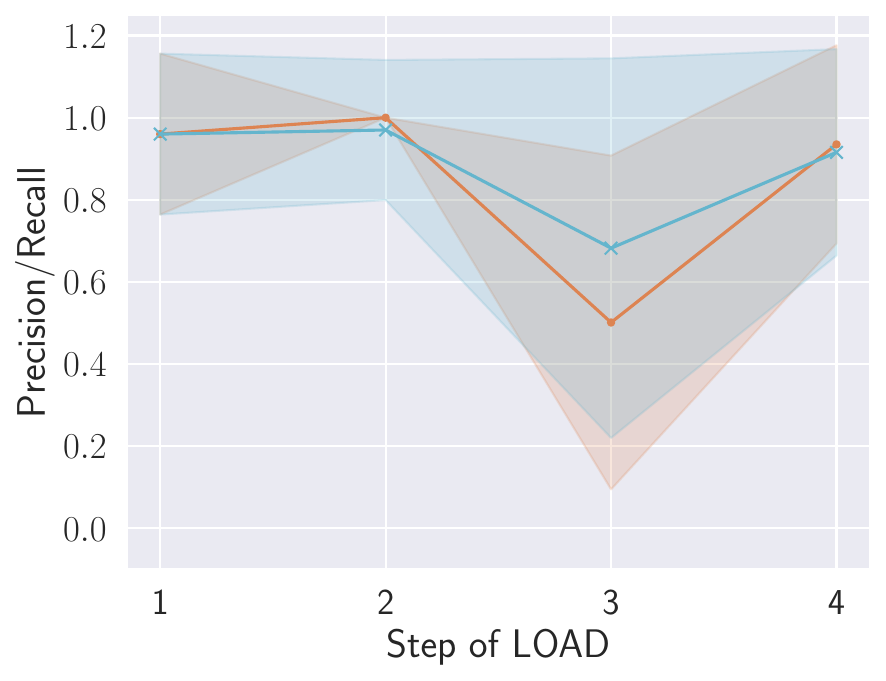}
        \end{subfigure}
    \end{subfigure}
    \caption{Precision and Recall of the different steps of LOAD in \Cref{alg:load} over different sample sizes, with Fisher-Z tests on linear Gaussian data, $n_{\mathbf{V}} = 200$, $\overline{d} = 2$ and $d_{\max} = 10$ and target pairs such that one is an explicit ancestor of the other.
    The shadow area denotes the range of the standard deviation.}
    \label{fig:steps}
\end{figure*}

\subsection{Alternative local causal discovery algorithms}
\label{sec:lcd_algorithm}
We implement LOAD in \Cref{alg:load} with the MB-by-MB local causal discovery subroutine.
In this section, we compare this implementation to an alternative where we replace MB-by-MB by the CMB local causal discovery algorithm \citep{gao2015local}, using the implementation of \citet{yu2020feature}.
Our comparison between the two implementations is shown in \Cref{fig:cmb}.
Our results show that while LOAD with MB-by-MB takes less time with d-separation and Fisher-Z CI tests and performs less CI tests in all data settings.
Similarly, CMB performs comparably on binary data, but achieves much worse F1 scores and intervention distances with Fisher-Z CI tests and even d-separation CI tests.
As the latter should not be possible, since CMB should be also sound and complete similarly to MB-by-MB, we assume there might be a bug in its implementation.

\begin{figure*}[ht!]
    \centering
    \begin{subfigure}[b]{\linewidth}
        \centering
        \includegraphics[width=.4\linewidth]{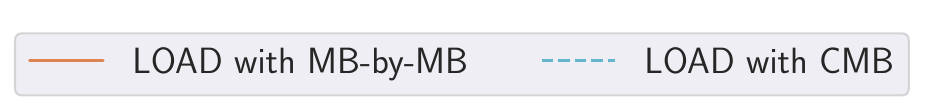}
    \end{subfigure}
    \begin{subfigure}[b]{.8\linewidth}
        \begin{subfigure}[b]{0.32\linewidth}
            \centering
            \caption*{ $\quad$ d-separation tests}
            \includegraphics[width=\linewidth]{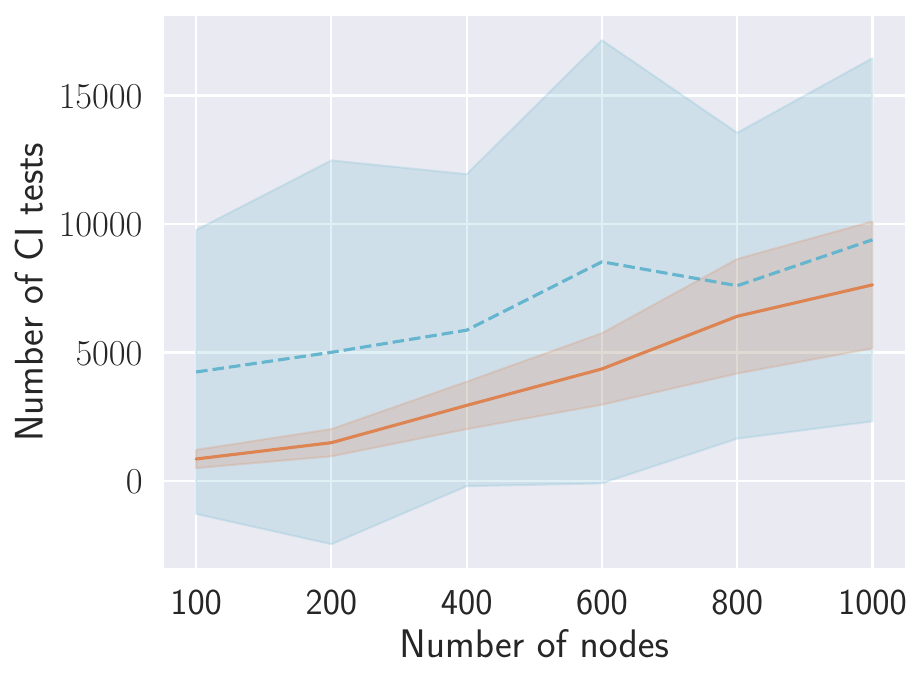}
            \includegraphics[width=\linewidth]{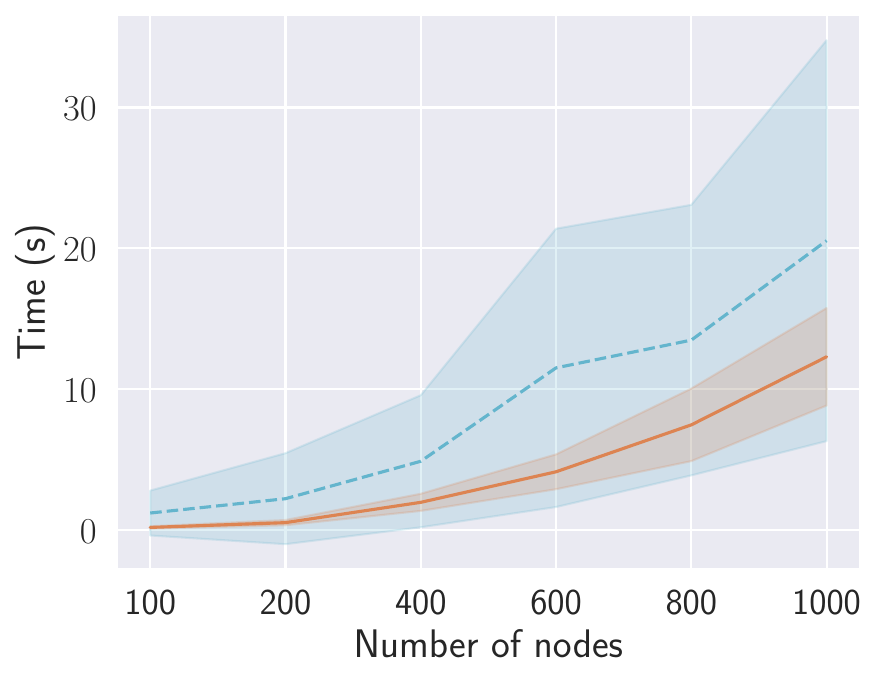}
            \includegraphics[width=\linewidth]{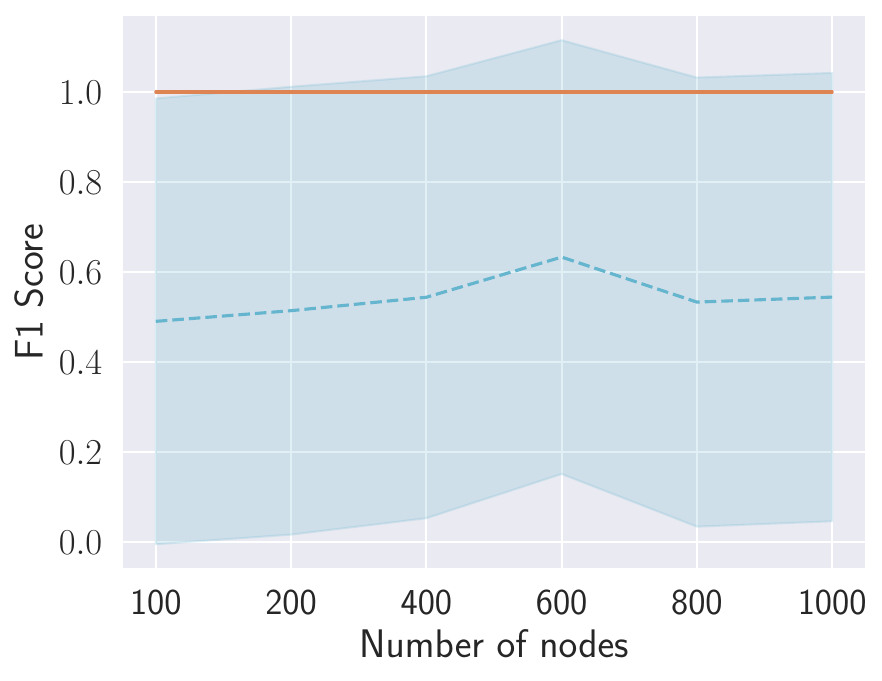}
            \includegraphics[width=\linewidth]{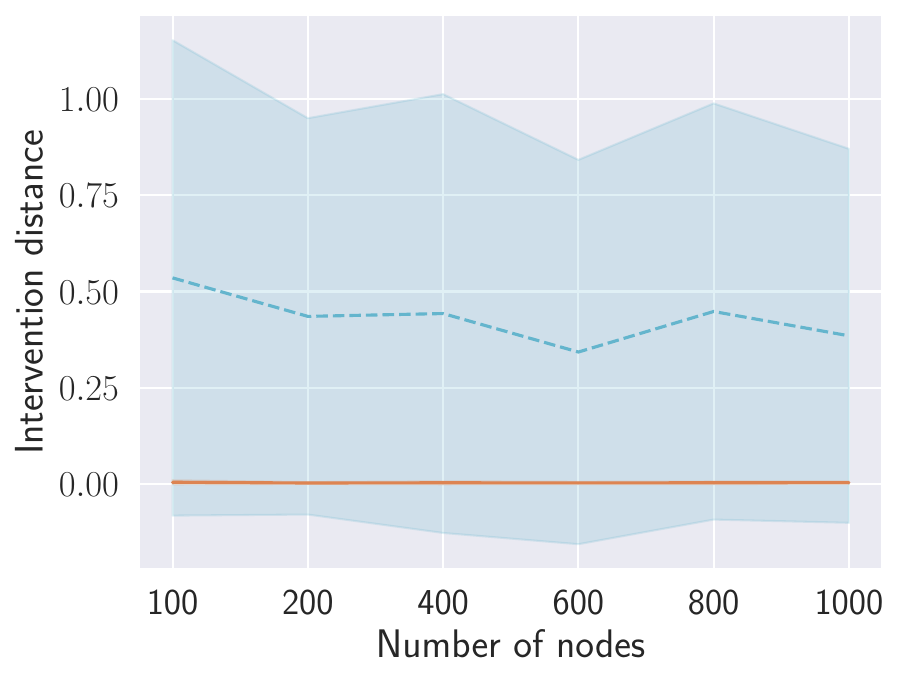}
        \end{subfigure}
        \begin{subfigure}[b]{0.32\linewidth}
            \centering
            \caption*{$\quad$ Fisher-Z tests}
            \includegraphics[width=\linewidth]{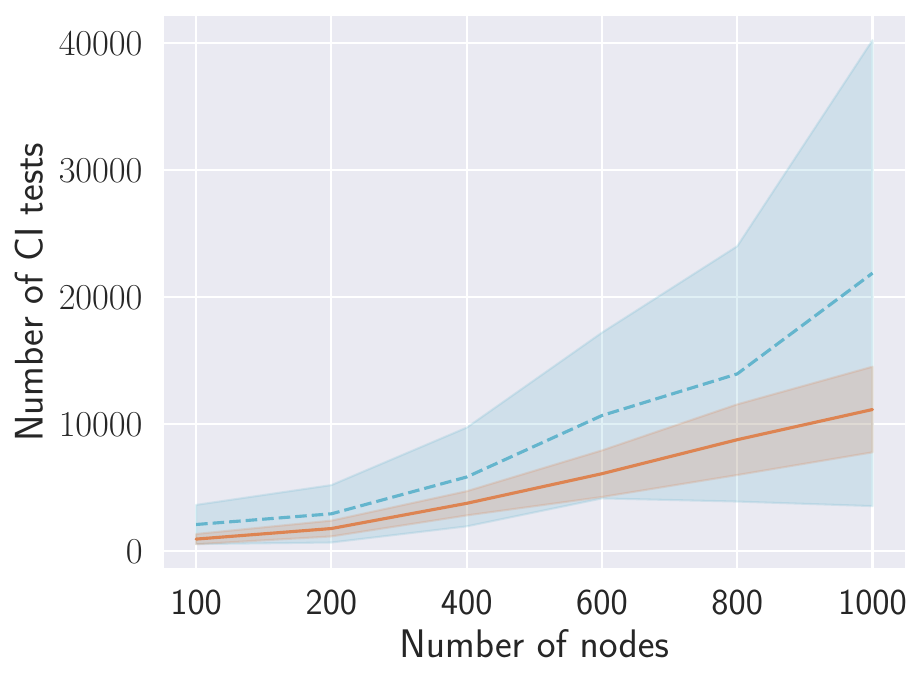}
            \includegraphics[width=\linewidth]{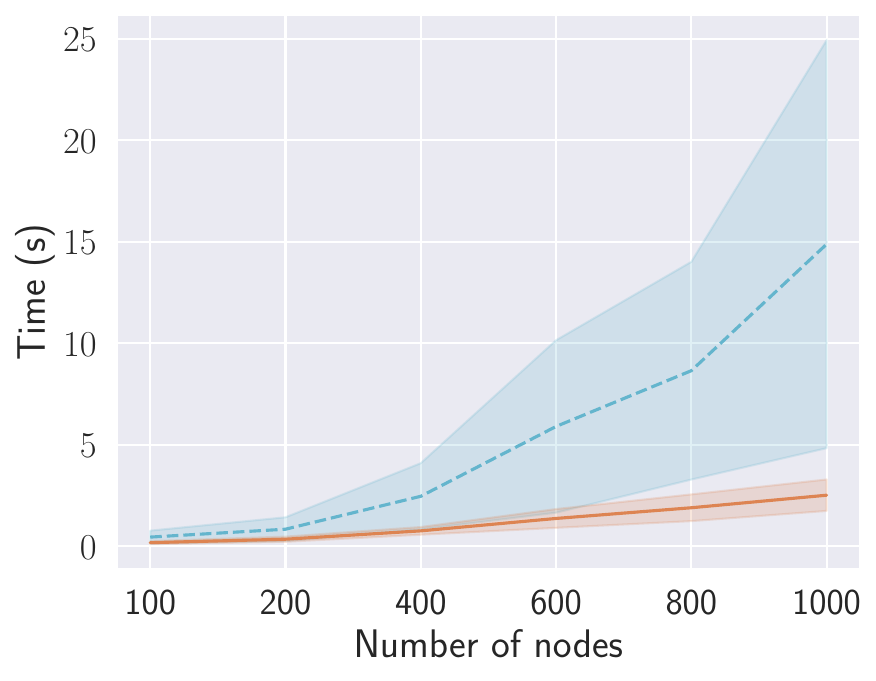}
            \includegraphics[width=\linewidth]{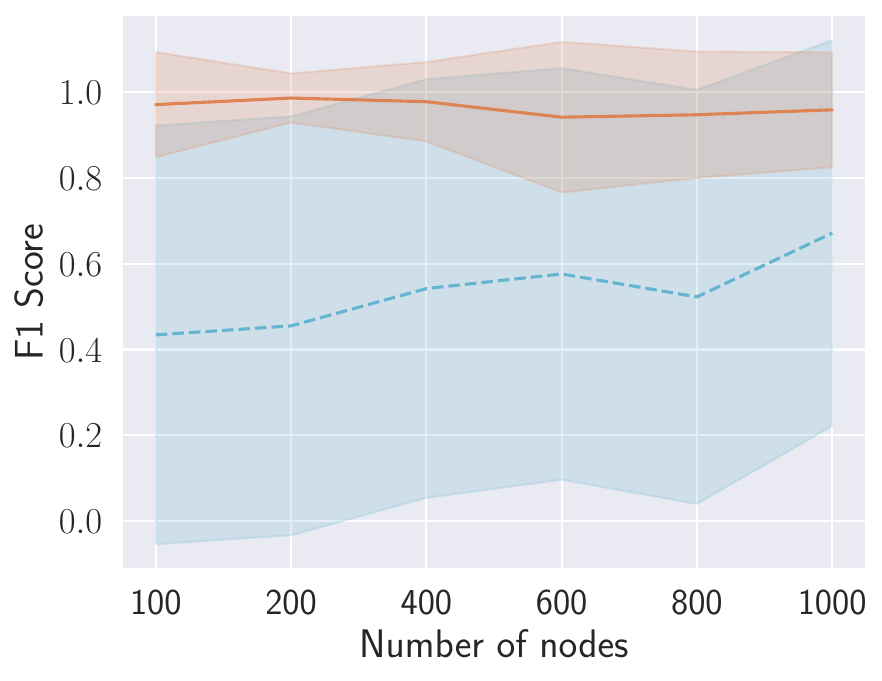}
            \includegraphics[width=\linewidth]{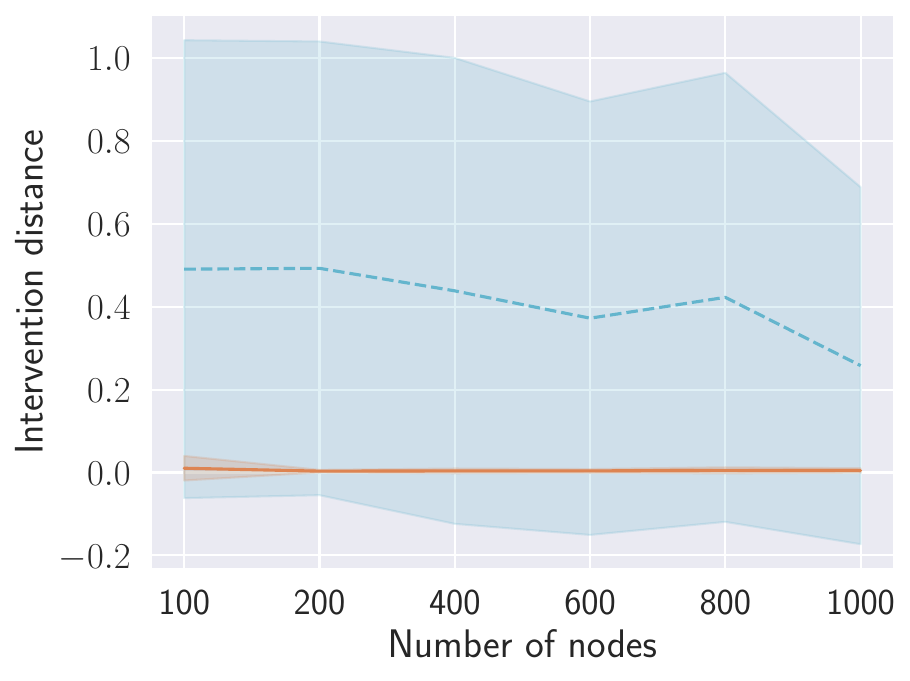}
        \end{subfigure}
        \begin{subfigure}[b]{0.32\linewidth}
            \centering
            \caption*{$G^2$ tests}
            \includegraphics[width=\linewidth]{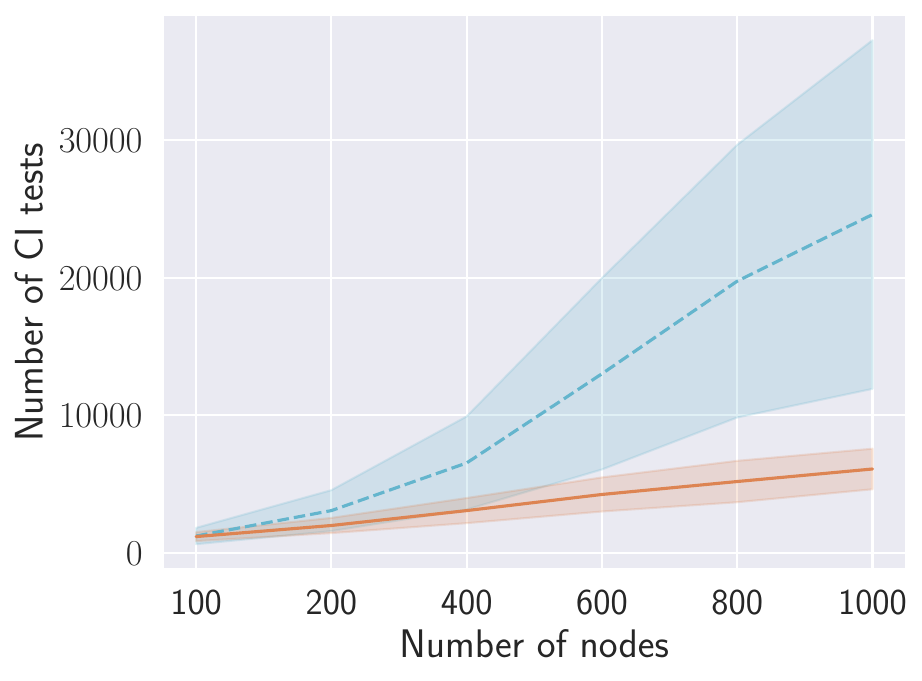}
            \includegraphics[width=\linewidth]{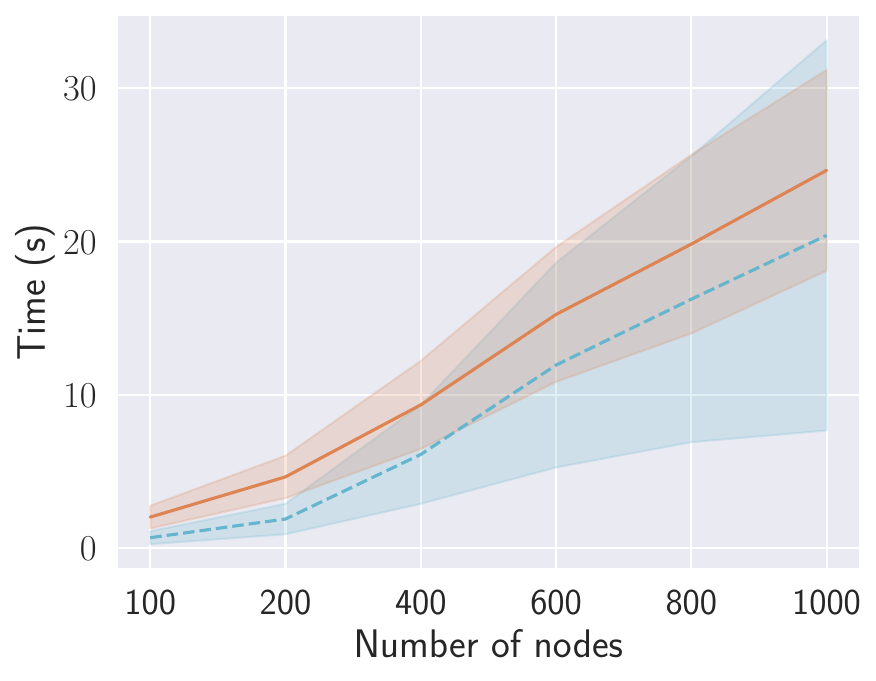}
            \includegraphics[width=\linewidth]{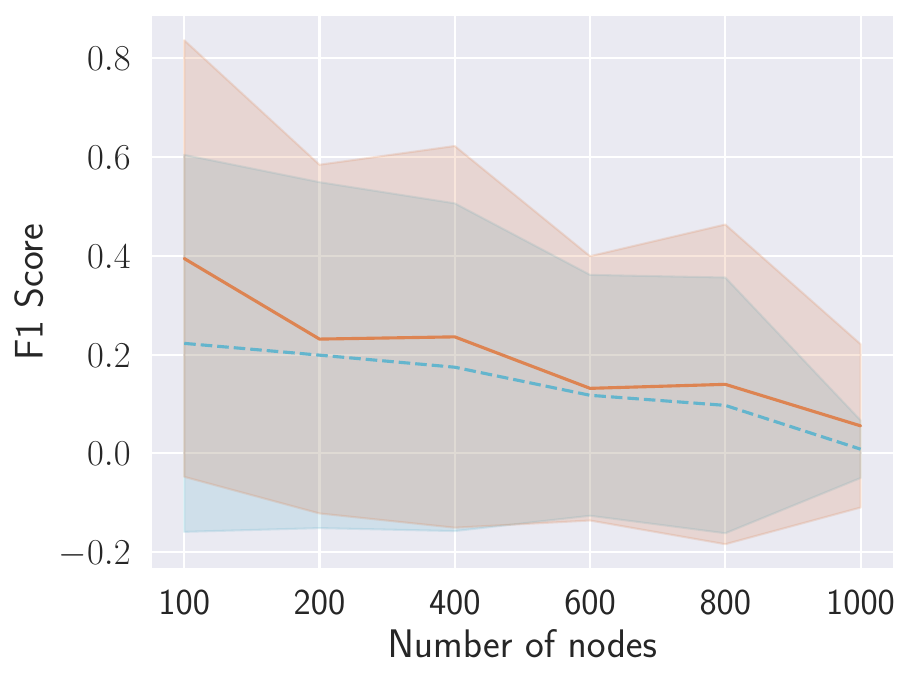}
            \includegraphics[width=\linewidth]{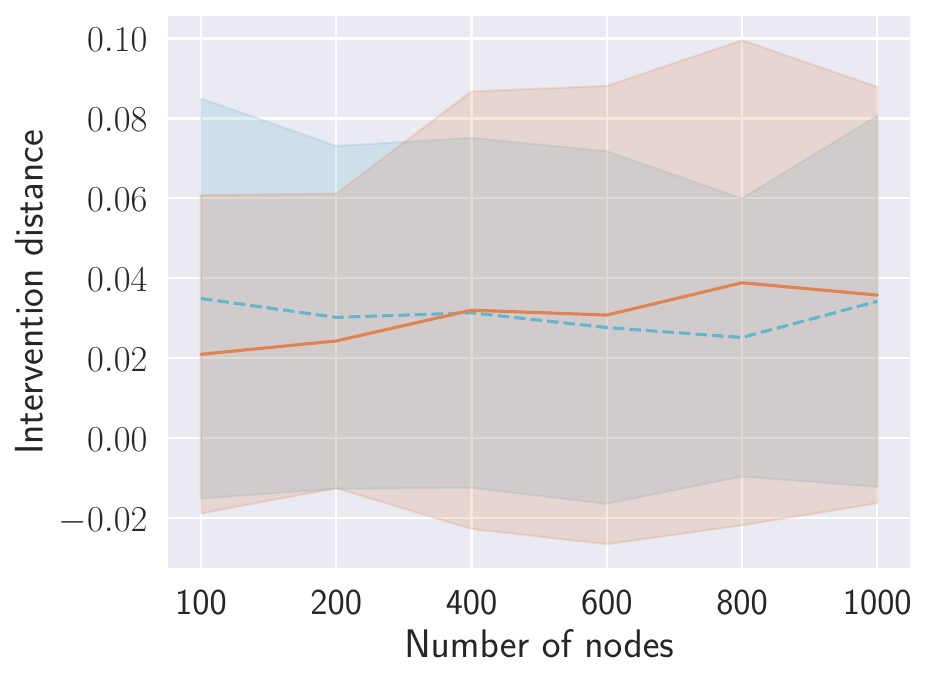}
        \end{subfigure}
    \end{subfigure}
    \caption{Results comparing LOAD using the MB-by-MB and the CMB algorithms for local causal discovery over various number of nodes, $n_{\mathbf{D}} = 10000$, $\overline{d} = 2$ and $d_{\max} = 10$ and target pairs such that one is an explicit ancestor of the other.
    The shadow area denotes the range of the standard deviation.}
    \label{fig:cmb}
\end{figure*}

\subsection{Alternative Markov blanket discovery algorithms}
\label{sec:mb_algorithm}

For the main results in \Cref{fig:main_results}, we implement LOAD using the MB-by-MB local causal discovery algorithm as a sub-routine, which employs the Grow-Shrink algorithm to find the Markov blankets of variables.
In this section, we compare utilizing Grow-Shrink to alternative Markov blanket discovery algorithm employed by MB-by-MB when running LOAD.

In \Cref{fig:total_conditioning}, we show results for replacing Grow-Shrink with Total Conditioning \citep{pellet2008using}.
While Total-Conditioning performs fewer CI tests and achieves comparable intervention distances, it requires much more computation time in finite data settings.
Furthermore, it completely fails to identify any nodes in the optimal adjustment sets correctly on binary data.

\begin{figure*}[ht!]
    \centering
    \begin{subfigure}[b]{\linewidth}
        \centering
        \includegraphics[width=.5\linewidth]{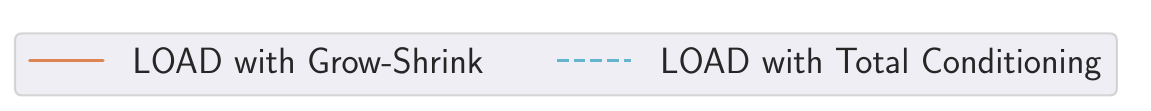}
    \end{subfigure}
    \begin{subfigure}[b]{.8\linewidth}
        \begin{subfigure}[b]{0.32\linewidth}
            \centering
            \caption*{ $\quad$ d-separation tests}
            \includegraphics[width=\linewidth]{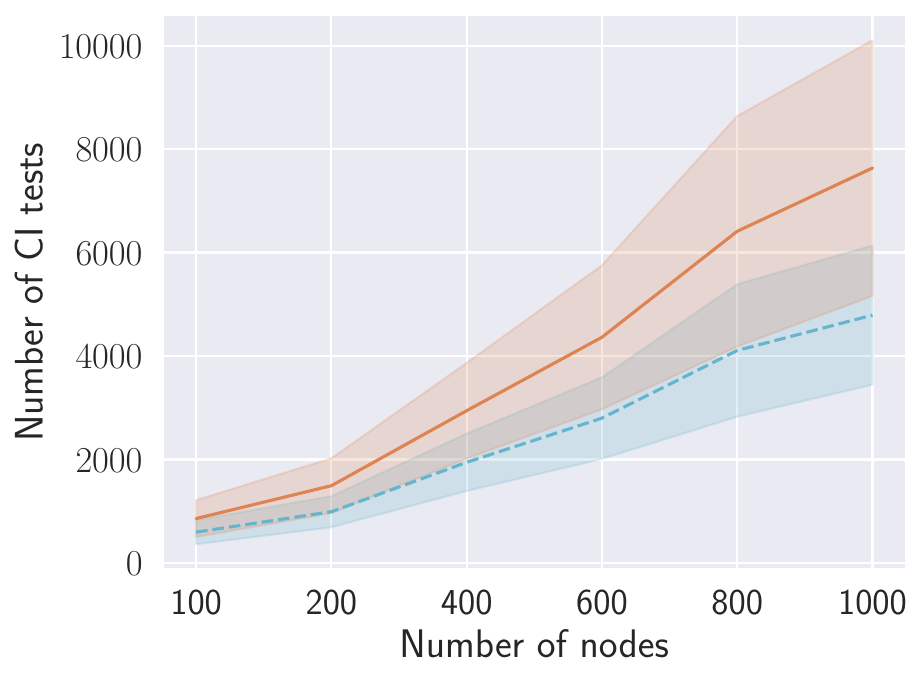}
            \includegraphics[width=\linewidth]{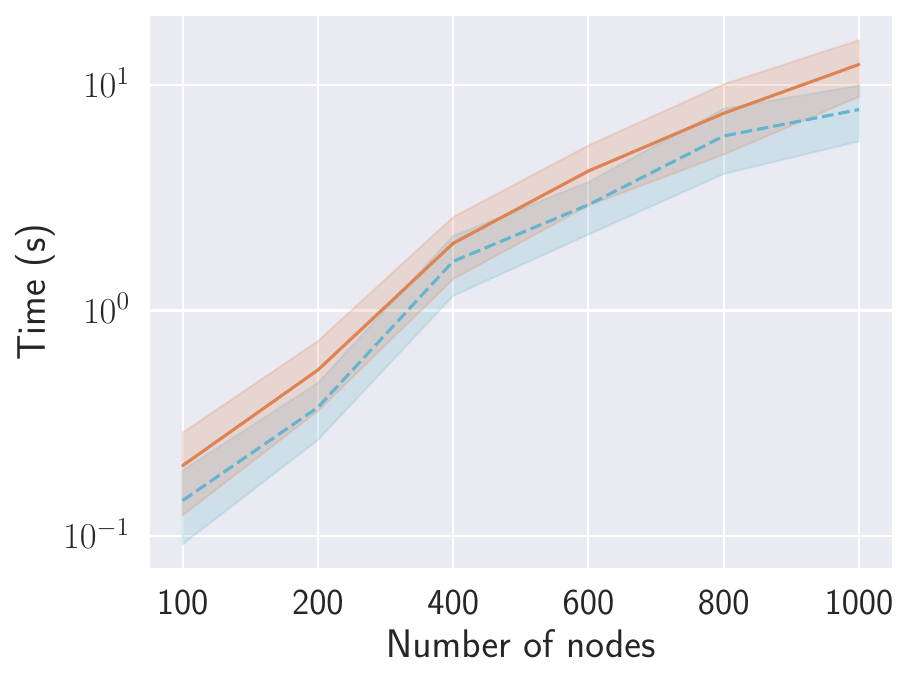}
            \includegraphics[width=\linewidth]{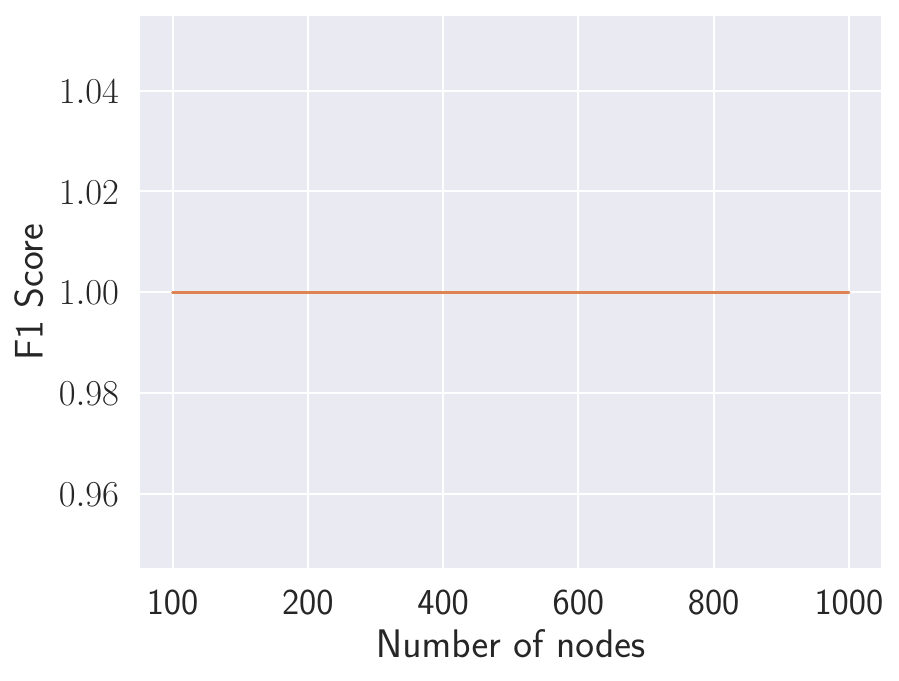}
            \includegraphics[width=\linewidth]{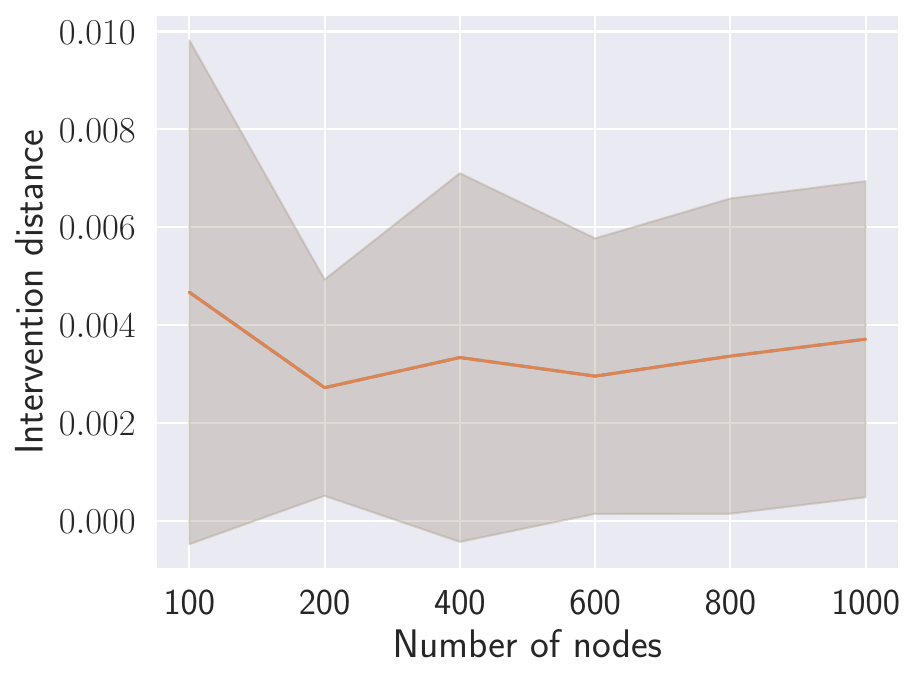}
        \end{subfigure}
        \begin{subfigure}[b]{0.32\linewidth}
            \centering
            \caption*{$\quad$ Fisher-Z tests}
            \includegraphics[width=\linewidth]{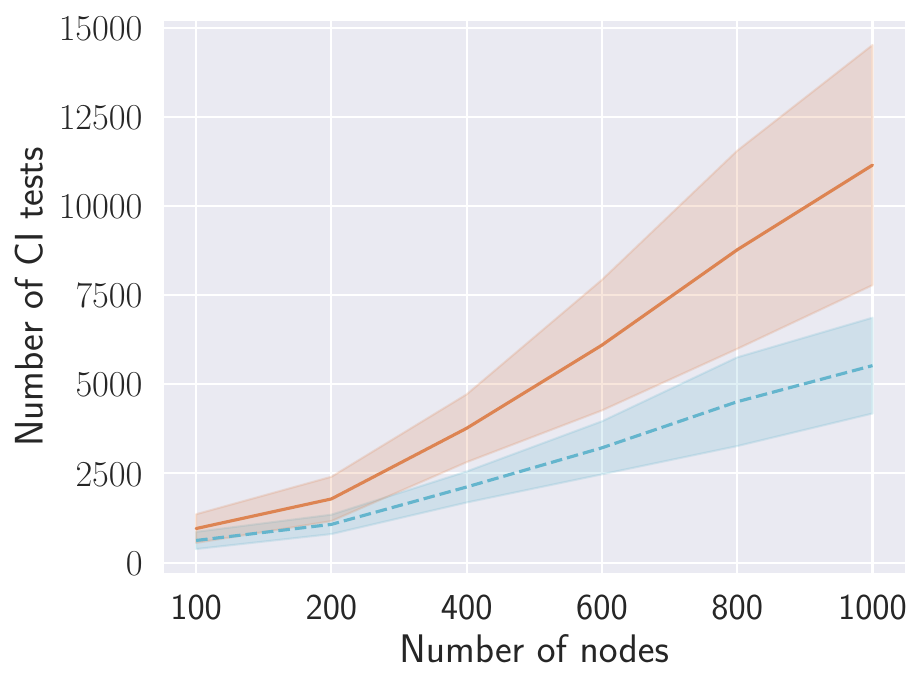}
            \includegraphics[width=\linewidth]{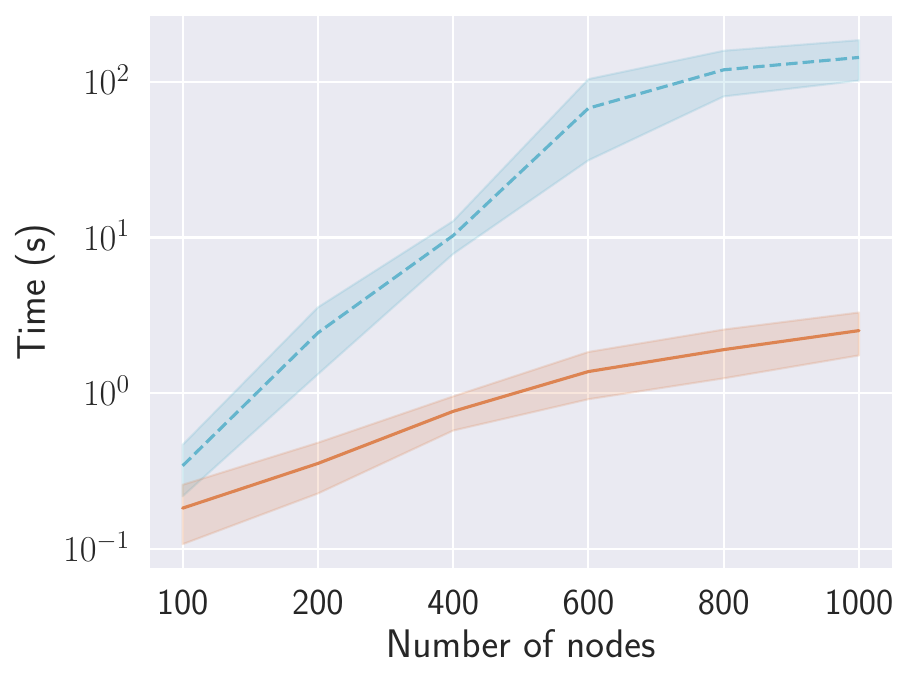}
            \includegraphics[width=\linewidth]{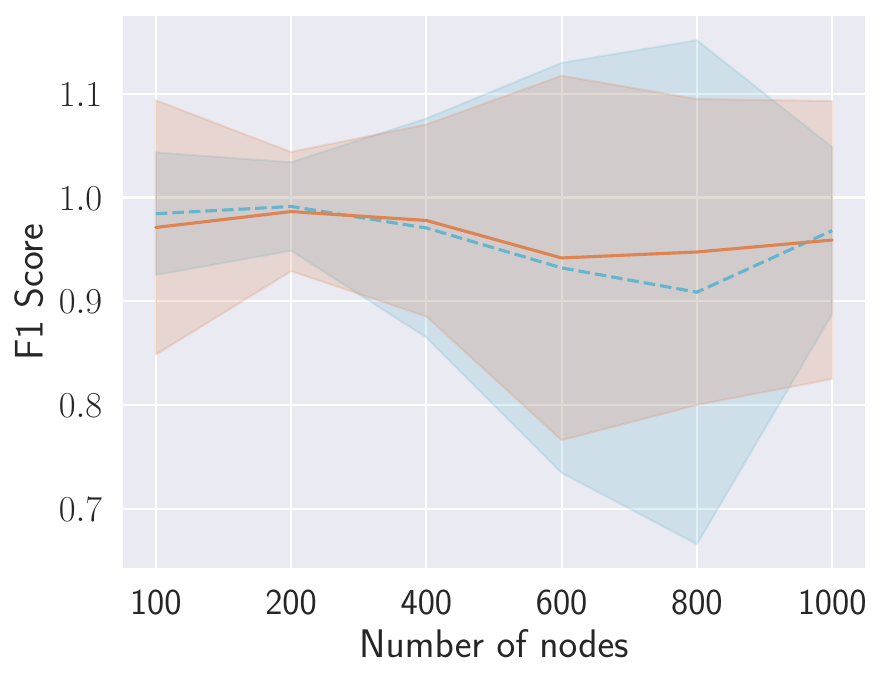}
            \includegraphics[width=\linewidth]{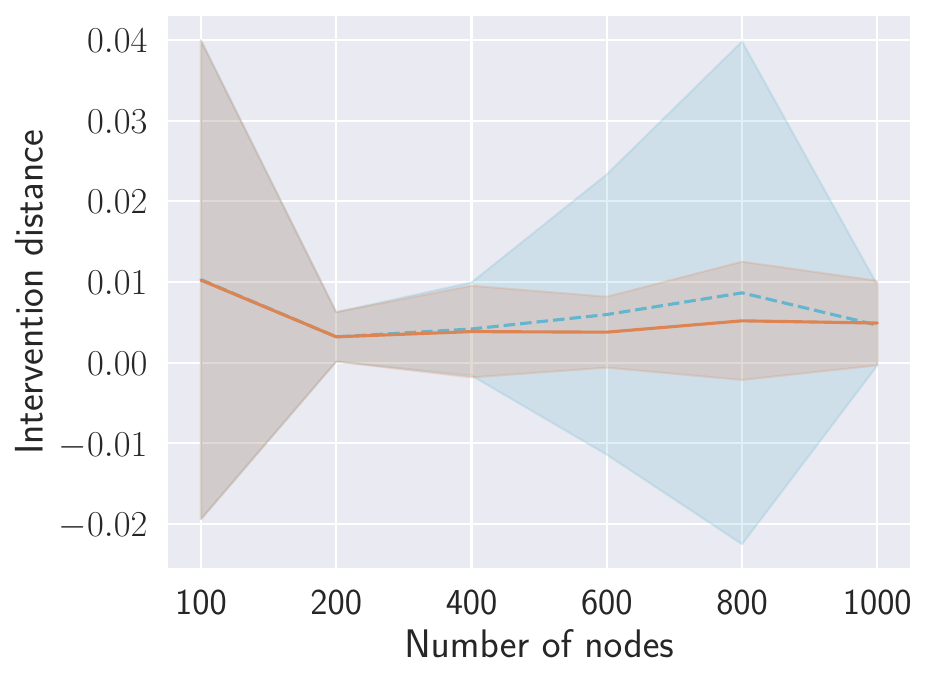}
        \end{subfigure}
        \begin{subfigure}[b]{0.32\linewidth}
            \centering
            \caption*{$G^2$ tests}
            \includegraphics[width=\linewidth]{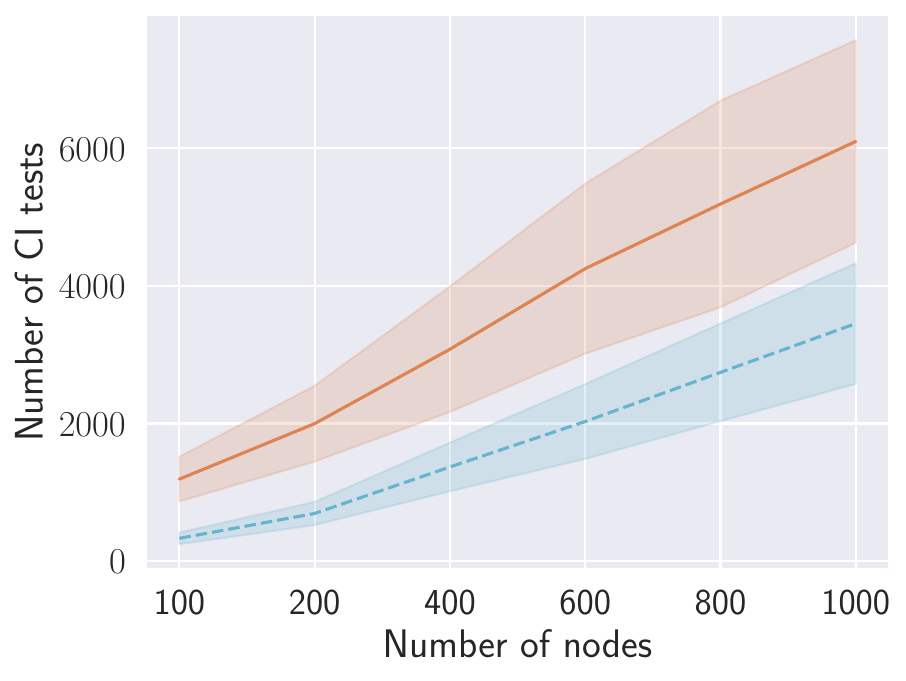}
            \includegraphics[width=\linewidth]{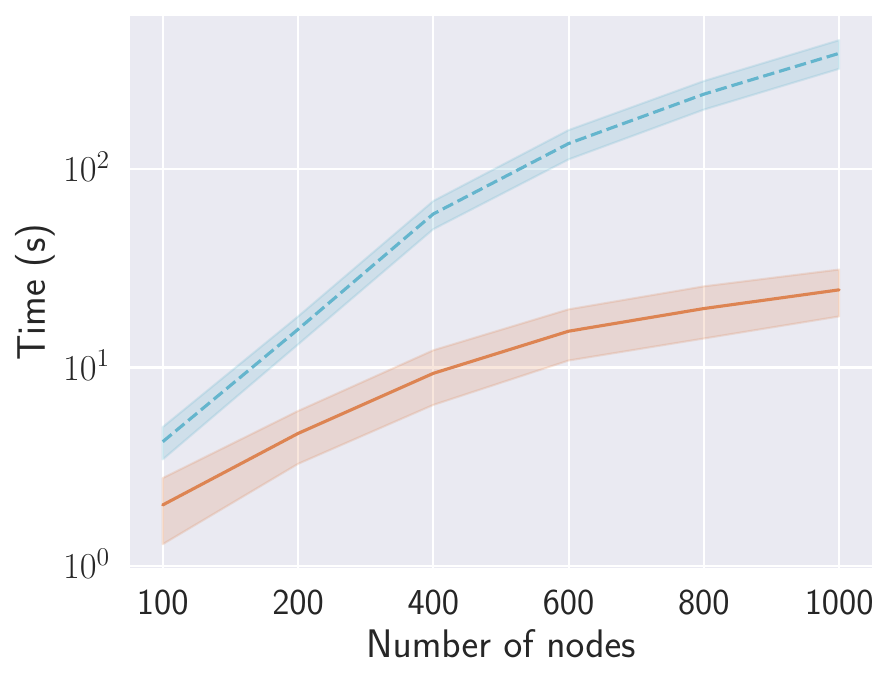}
            \includegraphics[width=\linewidth]{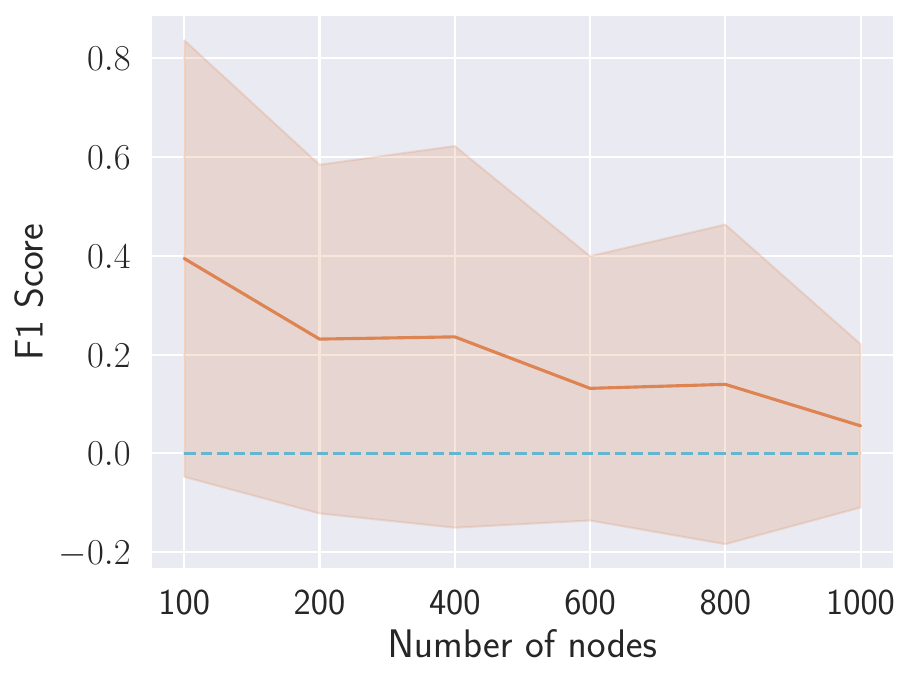}
            \includegraphics[width=\linewidth]{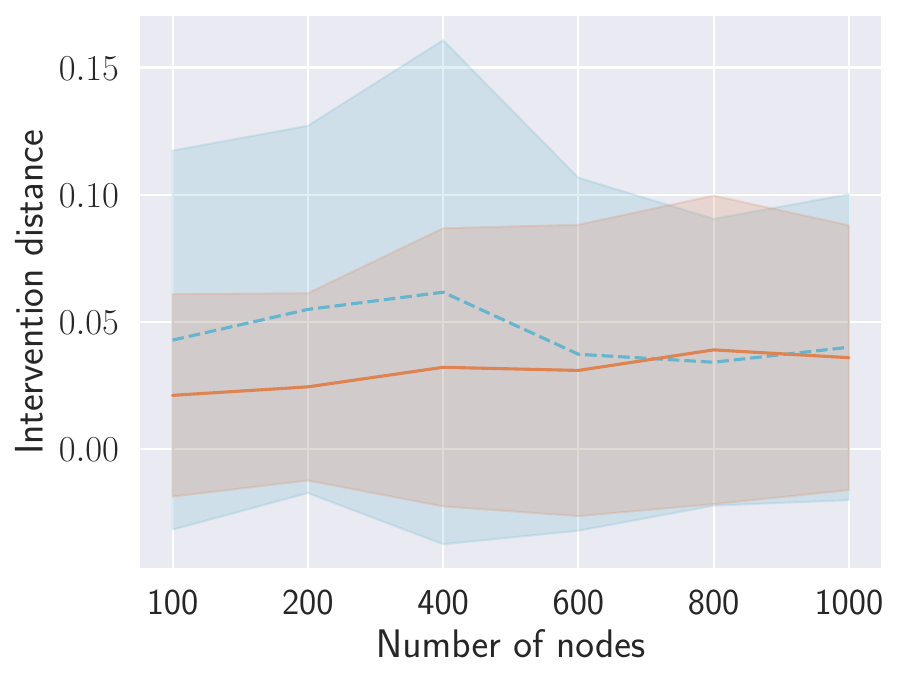}
        \end{subfigure}
    \end{subfigure}
    \caption{Results comparing LOAD using the Grow-Shrink and the Total-Conditioning algorithms for Markov blanket discovery over various number of nodes, $n_{\mathbf{D}} = 10000$, $\overline{d} = 2$ and $d_{\max} = 10$ and target pairs such that one is an explicit ancestor of the other.
    The shadow area denotes the range of the standard deviation.}
    \label{fig:total_conditioning}
\end{figure*}

In \Cref{fig:s2tmb}, we compare to the S$^2$TMB score based Markov blanket discovery algorithm \citep{GAO2017277}.
We use the implementation of \citet{yu2020feature}, which can only handle discrete data.
Thus, we use CI tests on binary data with 10, 15 and 20 nodes under the same other settings as \Cref{fig:main_results}.
Our results show that while S$^2$TMB performs fewer CI tests than when using Grow-Shrink, it takes much more time.
Furthermore, it achieves a worse F1 scores and intervention distances across all graph sizes.

\begin{figure*}[ht!]
    \centering
    \begin{subfigure}[b]{\linewidth}
        \centering
        \includegraphics[width=.4\linewidth]{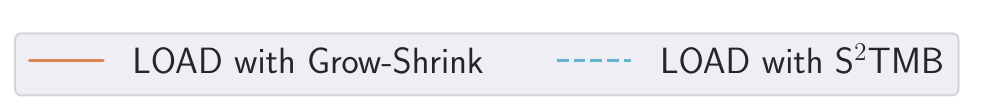}
    \end{subfigure}
        \includegraphics[width=.24\linewidth]{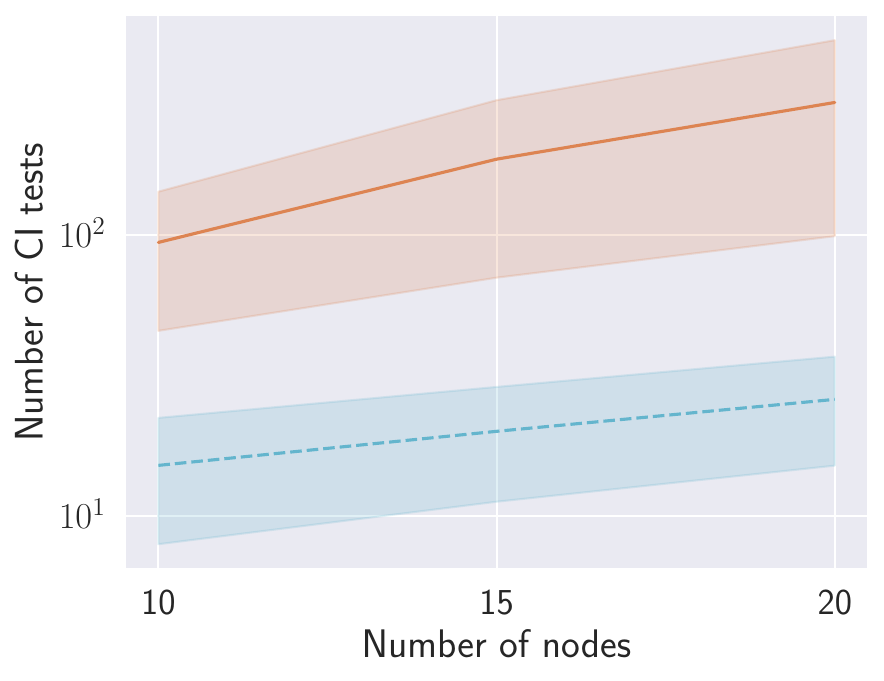}
        \includegraphics[width=.24\linewidth]{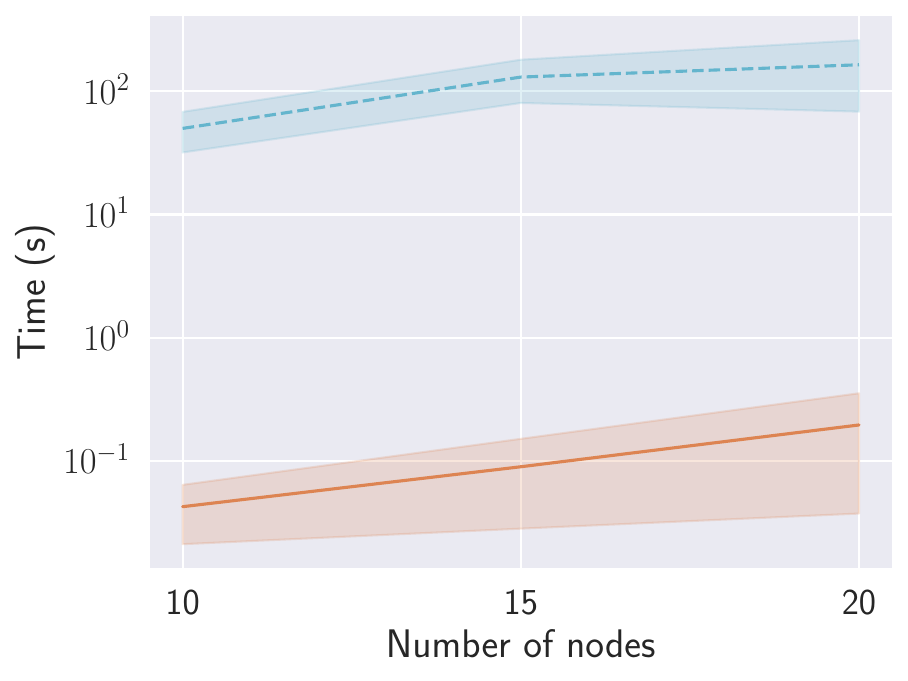}
        \includegraphics[width=.24\linewidth]{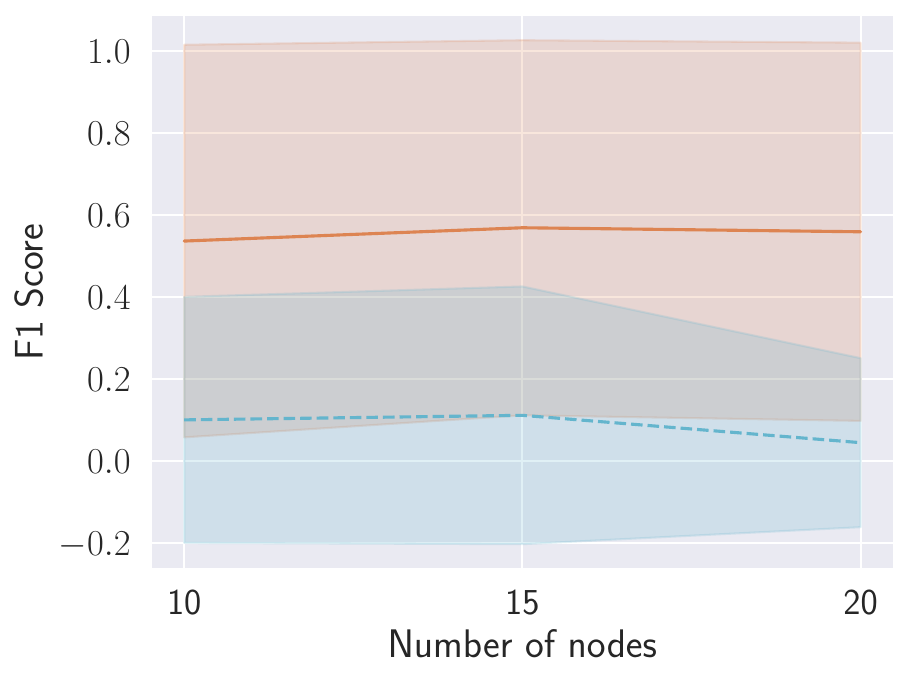}
        \includegraphics[width=.24\linewidth]{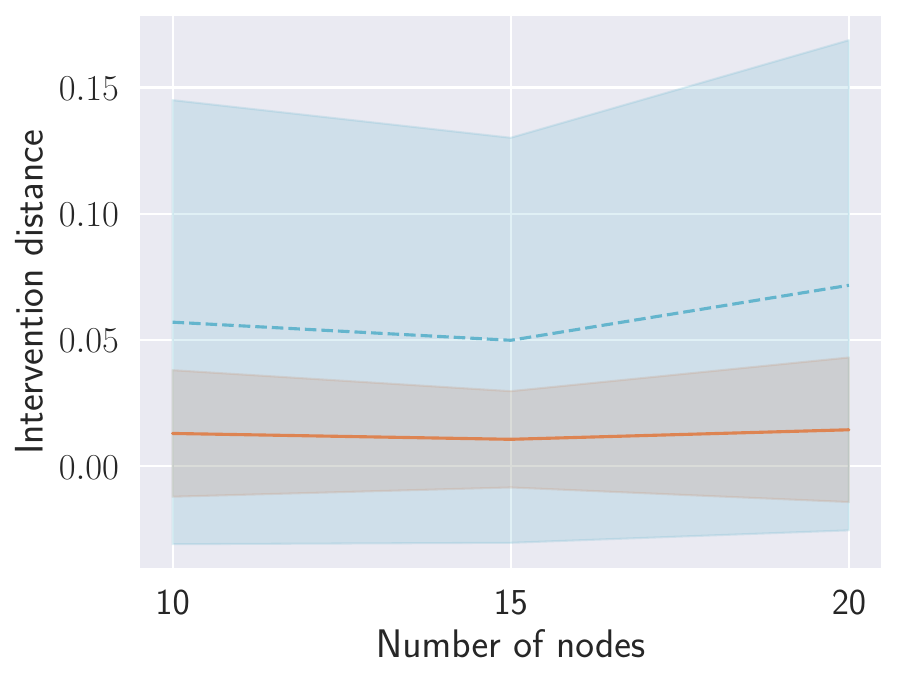}
    \caption{Results comparing LOAD using the Grow-Shrink and the S$^2$TMB algorithms for Markov blanket discovery on binary data sampled from graphs over various number of nodes, $n_{\mathbf{D}} = 10000$, $\overline{d} = 2$ and $d_{\max} = 10$ and target pairs such that one is an explicit ancestor of the other.
    The shadow area denotes the range of the standard deviation.}
    \label{fig:s2tmb}
\end{figure*}

\subsection{Providing Treatment-Outcome Relation as Background Knowledge}
\label{sec:bk_experiments}

In this section we evaluate the scenario where the treatment-outcome relationship is known by background knowledge.
This scenario fits the requirements of most local causal discovery algorithms in literature, such as MB-by-MB \citep{wang2014discovering}, LDECC \citep{gupta2023local} and LDP \citep{maasch2024local}.
Thus, for the experiments in this section, we run the original versions of these algorithms with the background knowledge provided.
Furthermore, we introduce a version of LOAD that can also utilize this background knowledge, denoted as LOAD$^*$, which simply skips the first step of LOAD for determining ancestral relationships.
Global algorithms, such as PC, MARVEL and SNAP remain unchanged.

Otherwise, we use the same data setup as in the main paper, i.e., $n_{\mathbf{D}} = 10000$ data samples, expected degree of $\overline{d} = 2$ and maximum degree of $d_{\max} = 10$, and use a significance level of $\alpha = 0.01$ for all algorithms and CI tests.

Our results are shown in \Cref{fig:bk_results}.
Since LDECC only has to run once on the provided treatment instead of both targets, it was possible to evaluate it on binary data using $G^2$ CI tests.

Results for the number of CI tests time and intervention distance are similar to the case where the background knowledge is not provided.
The number of CI tests and computation time increases steadily with the number of variables for all methods under all settings, except for LDECC where they decrease on binary data using $G^2$ tests.
Since this is not the case for LDECC with d-separation or Fisher-Z on linear Gaussian data, or for any other other method on binary data, we suspect that this arises due to errors in CI testing performed by LDECC.
In general, when the treatment-outcome relationship is known, LDP is the cheapest in terms of CI tests and computation time.

On binary data with 100-400 nodes, LOAD$^*$ is able to outperform PC in terms of F1 score of the estimated Oset, and remains comparable throughout all number of nodes.

\begin{figure*}[ht]
    \centering
    \begin{subfigure}[b]{\linewidth}
        \centering
        \includegraphics[width=.8\linewidth]{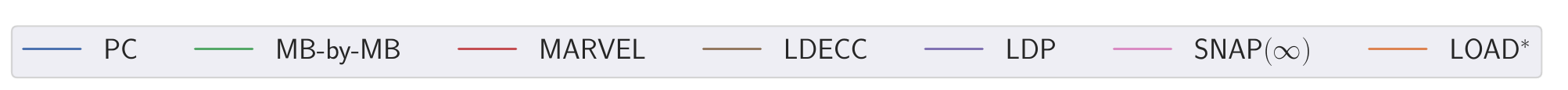}
    \end{subfigure}
    \begin{subfigure}[b]{.8\linewidth}
        \begin{subfigure}[b]{0.32\linewidth}
            \centering
            \caption*{ $\quad$ d-separation tests}
            \includegraphics[width=\linewidth]{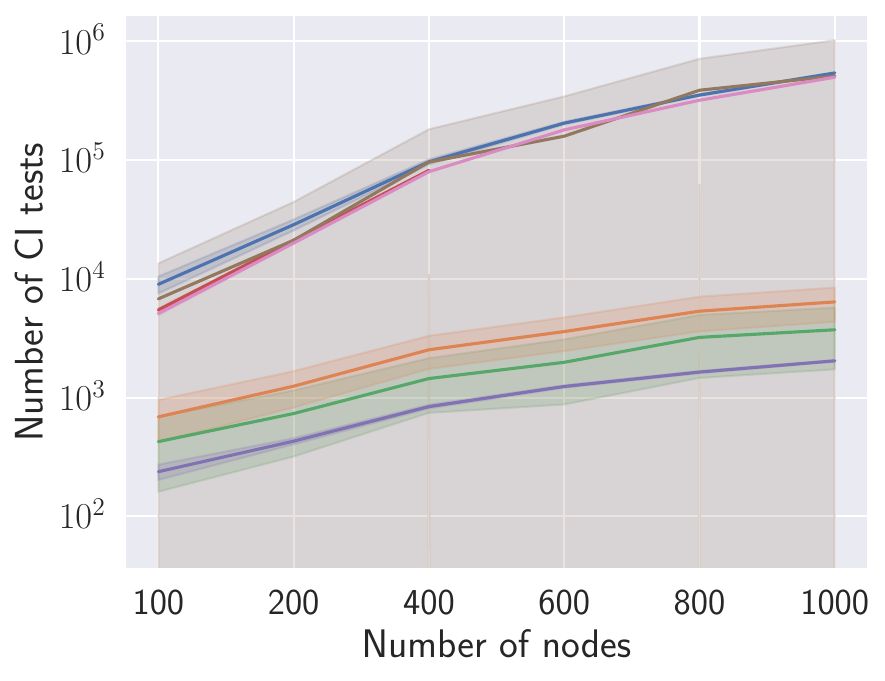}
            \includegraphics[width=\linewidth]{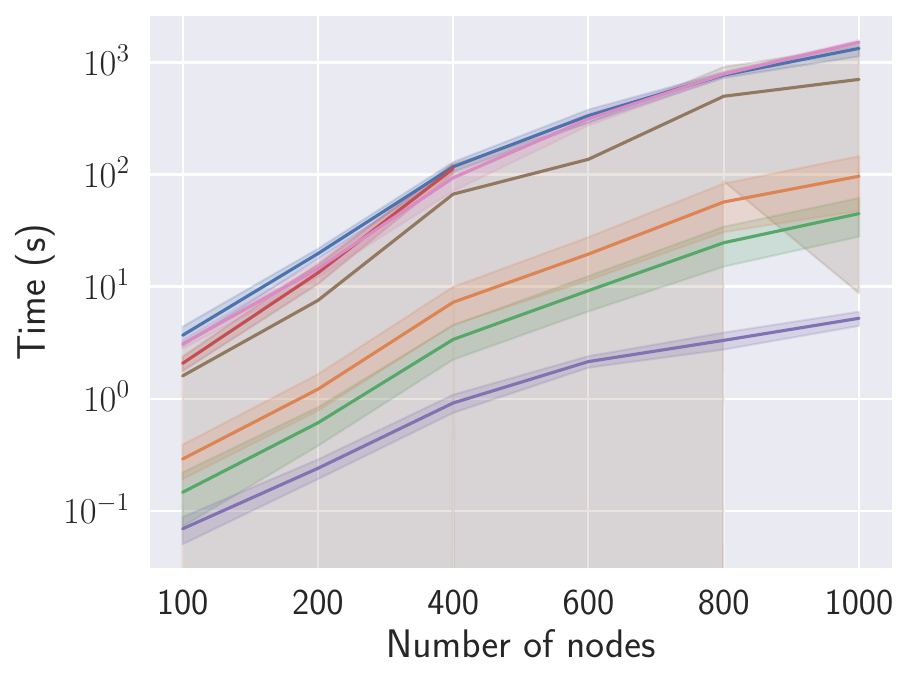}
            \includegraphics[width=\linewidth]{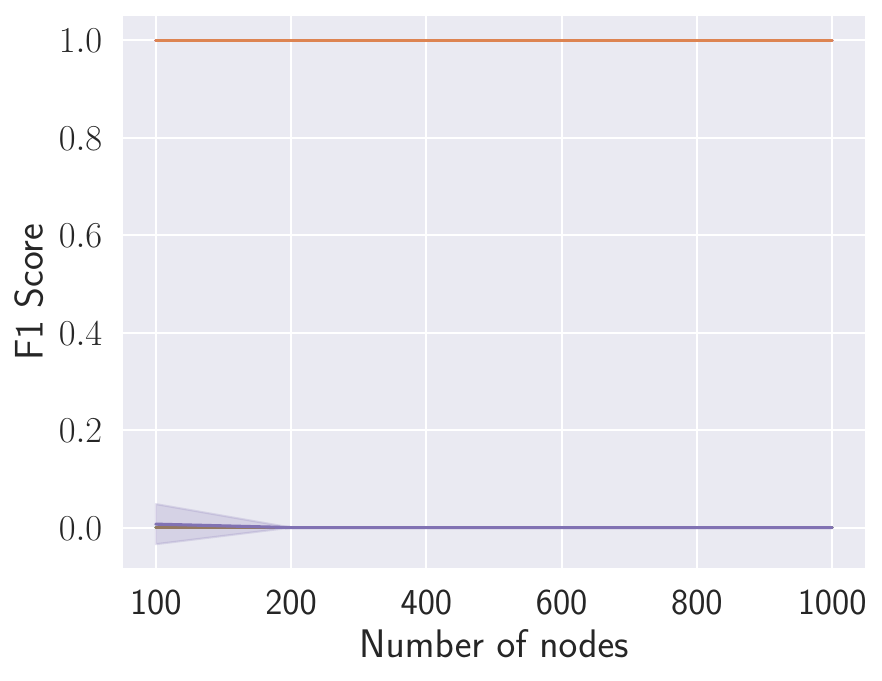}
            \includegraphics[width=\linewidth]{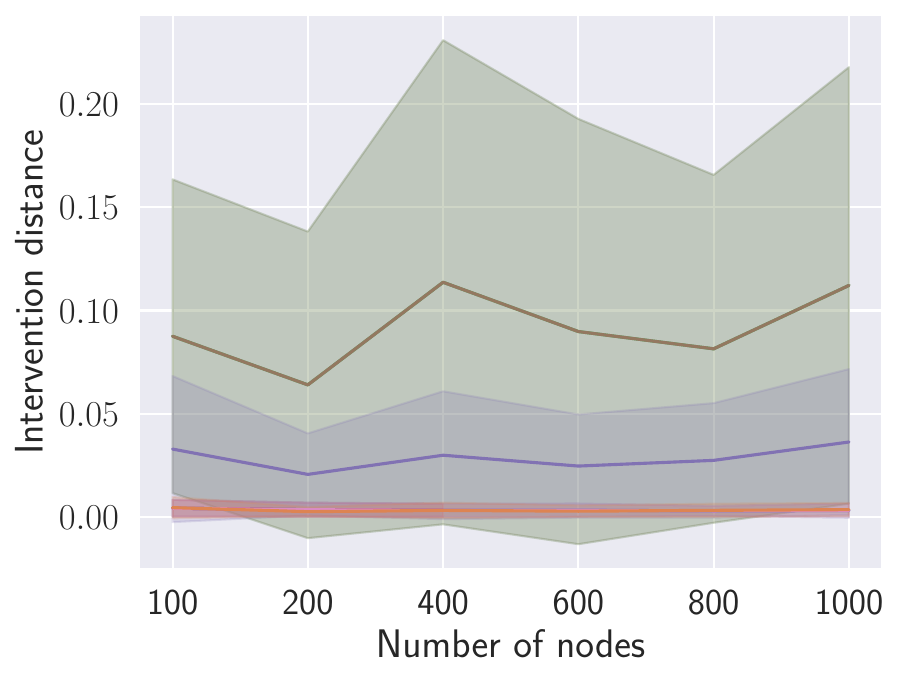}
        \end{subfigure}
        \begin{subfigure}[b]{0.32\linewidth}
            \centering
            \caption*{$\quad$ Fisher-Z tests}
            \includegraphics[width=\linewidth]{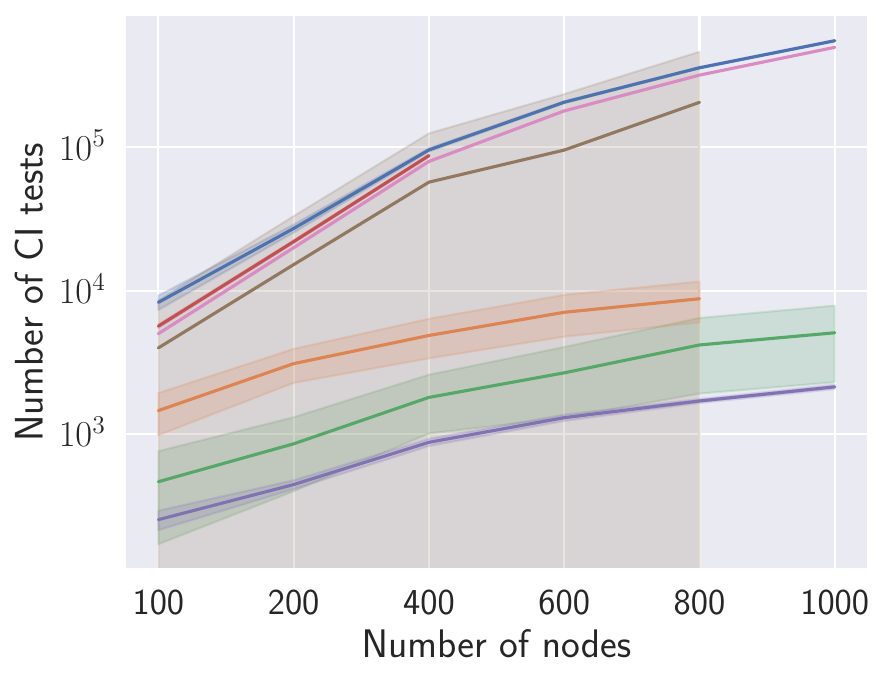}
            \includegraphics[width=\linewidth]{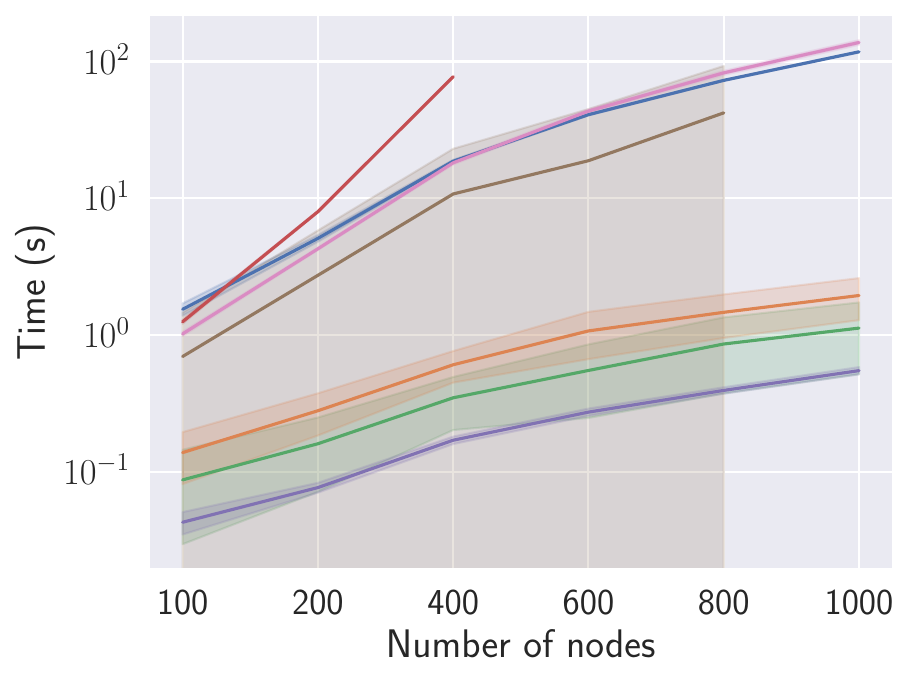}
            \includegraphics[width=\linewidth]{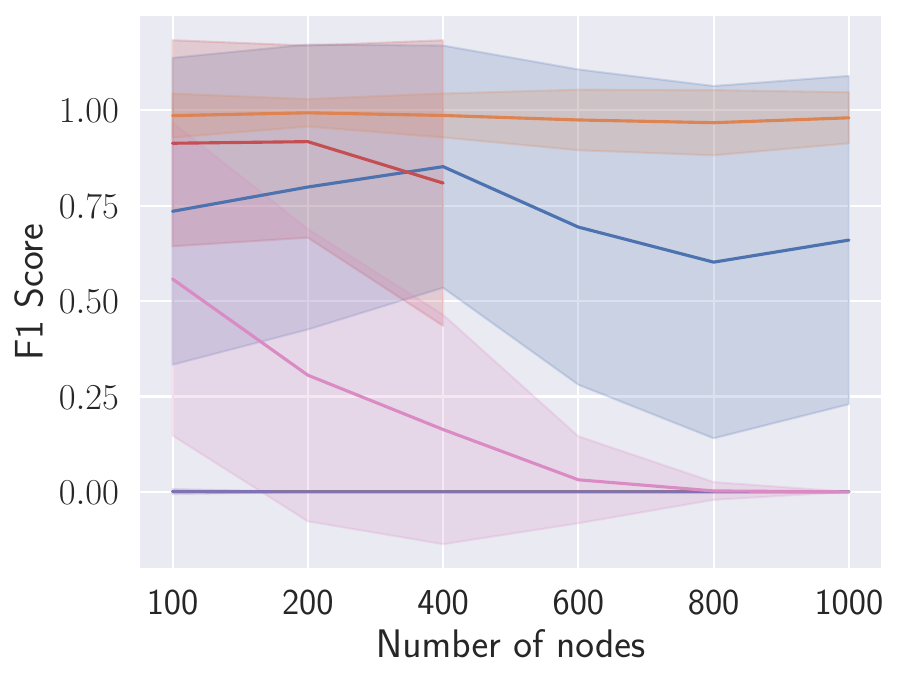}
            \includegraphics[width=\linewidth]{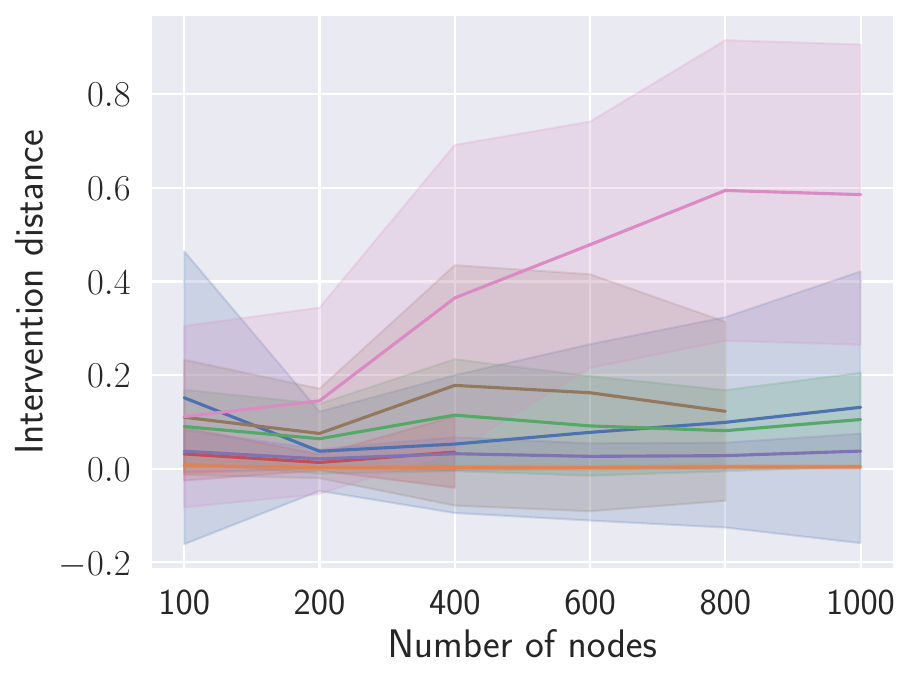}
        \end{subfigure}
        \begin{subfigure}[b]{0.32\linewidth}
            \centering
            \caption*{$G^2$ tests}
            \includegraphics[width=\linewidth]{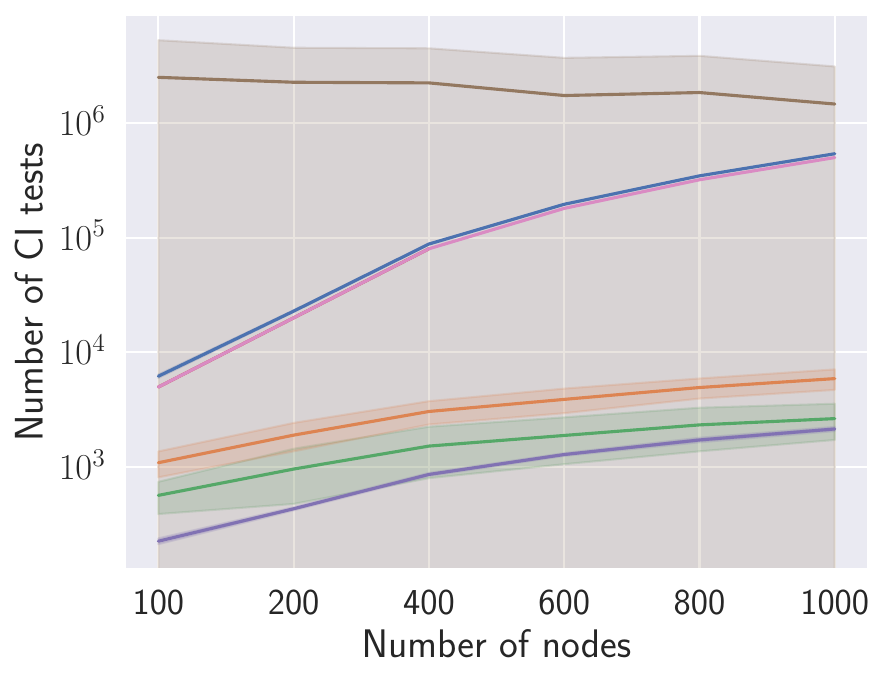}
            \includegraphics[width=\linewidth]{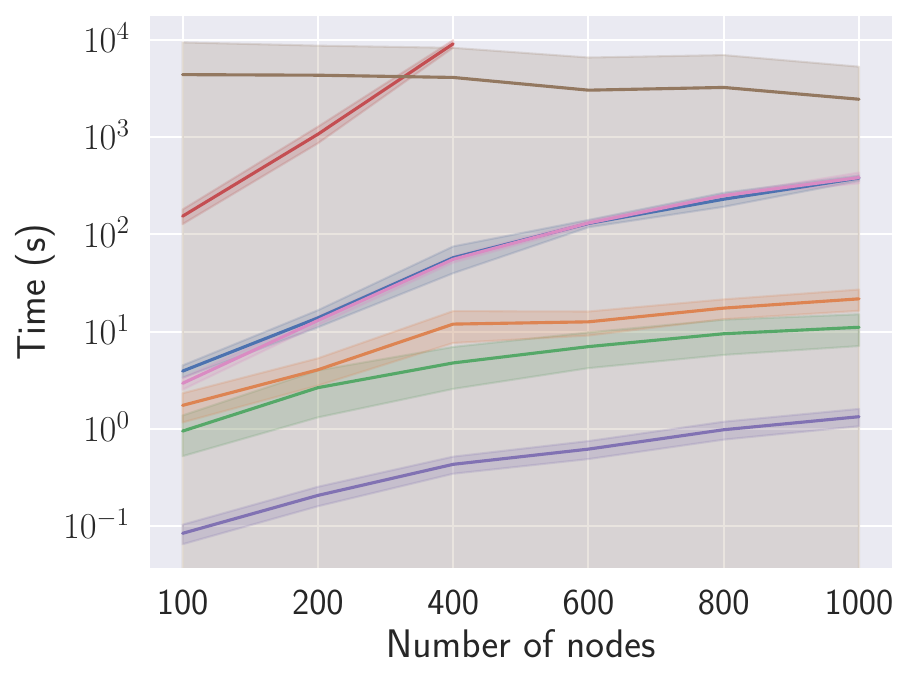}
            \includegraphics[width=\linewidth]{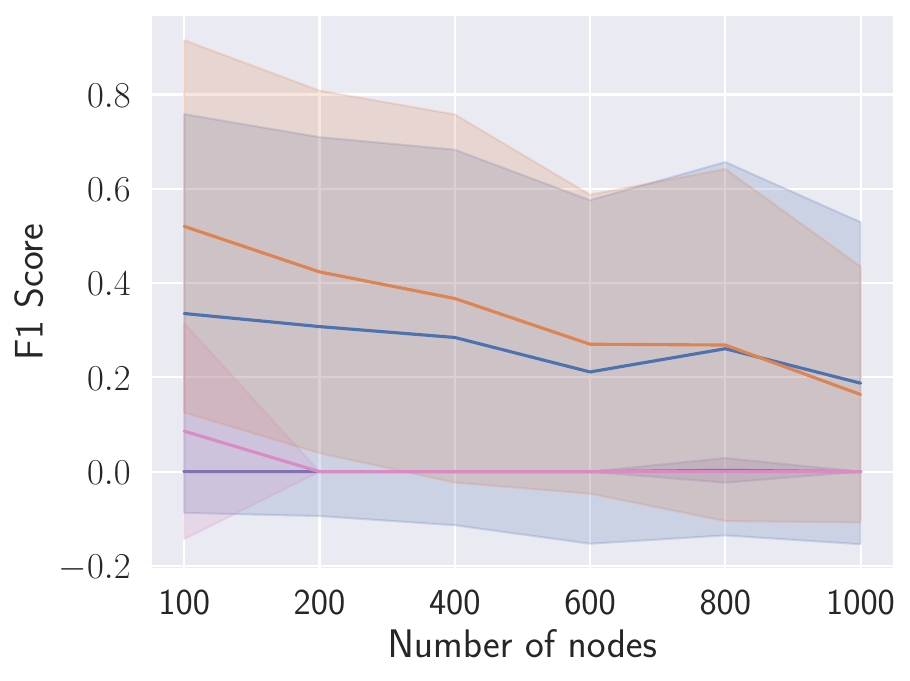}
            \includegraphics[width=\linewidth]{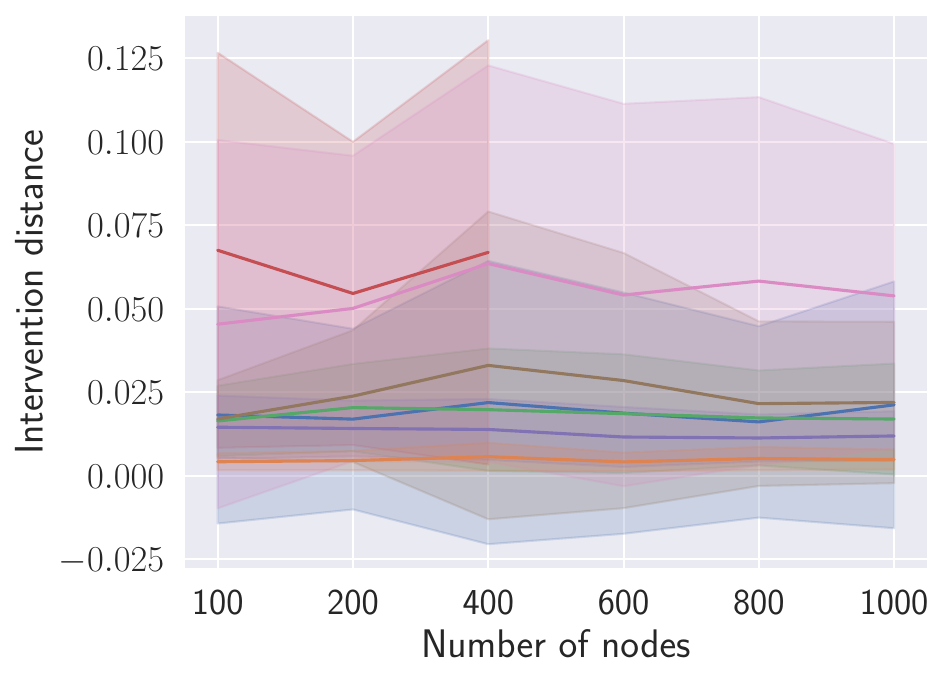}
        \end{subfigure}
    \end{subfigure}
    \caption{Results over number of nodes with $n_{\mathbf{D}} = 10000$, $\overline{d} = 2$ and $d_{\max} = 10$ for the scenario where the treatment-outcome relation is provided as background knowledge for MB-by-MB, LDECC, LDP and LOAD$^*$. The shadow area denotes the range of the standard deviation.}
    \label{fig:bk_results}
\end{figure*}

\subsection{Identifiable Causal Effects}
\label{sec:identifiable_experiments}

Sampling target pairs where one is an explicit ancestor of the other, but its causal effect is not necessarily identifiable often leads to unidentifiable causal effects, making LOAD return early after step 2.
To evaluate the scenario where LOAD should  execute all of its steps, we evaluate all algorithms on target pairs where we ensure that the causal effect is identifiable.

In this setting, we only experiment with node numbers $[100, 200, 300, 400, 500]$ because we encountered time and memory issues when generating the data for larger number of nodes.
In particular, without parallelization, each seed for 500 nodes take around one hour and 15 minutes, and we average results over a 100 seeds.
With larger nodes, parallelization becomes difficult due to the size of the datasets.
Otherwise, we use the same data setup as in the main paper, i.e., $n_{\mathbf{D}} = 10000$ data samples, expected degree of $\overline{d} = 2$ and maximum degree of $d_{\max} = 10$, and use a significance level of $\alpha = 0.01$ for all algorithms and CI tests.

Results for identifiable target pairs are shown in \Cref{fig:ident_results}.
Similarly to the main results report in \Cref{fig:main_results}, the number of CI tests performed by LOAD is similar to local methods, and its computation time is even better than LDP.
The F1 score for optimal adjustment set recovery and the intervention distance of LOAD is one of the best across all settings, while all other baselines fall behind in some setting for one of these metrics.

\begin{figure*}[ht]
    \centering
    \begin{subfigure}[b]{\linewidth}
        \centering
        \includegraphics[width=.8\linewidth]{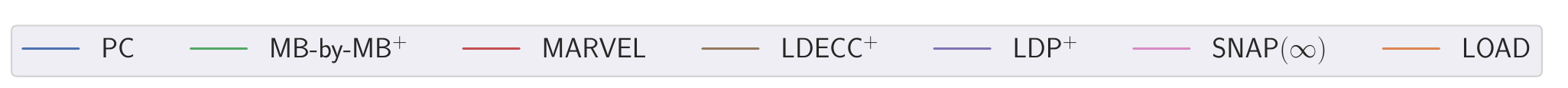}
    \end{subfigure}
    \begin{subfigure}[b]{.8\linewidth}
        \begin{subfigure}[b]{0.32\linewidth}
            \centering
            \caption*{ $\quad$ d-separation tests}
            \includegraphics[width=\linewidth]{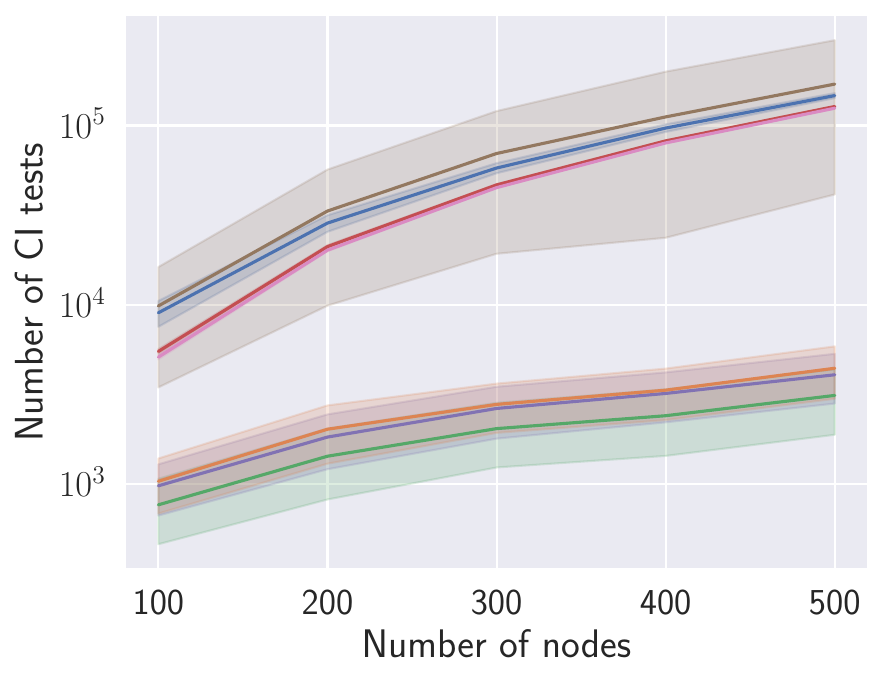}
            \includegraphics[width=\linewidth]{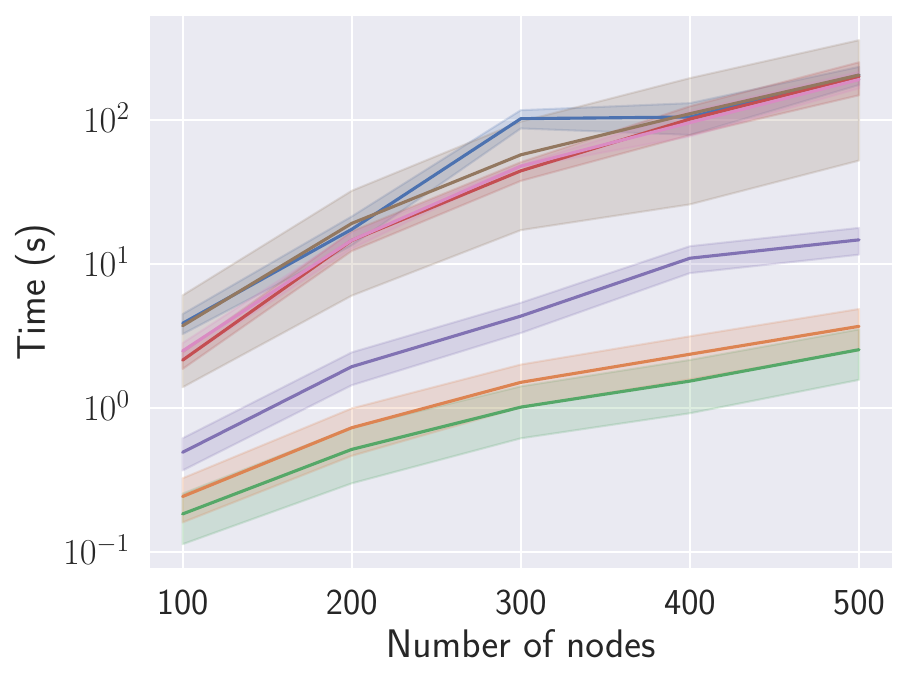}
            \includegraphics[width=\linewidth]{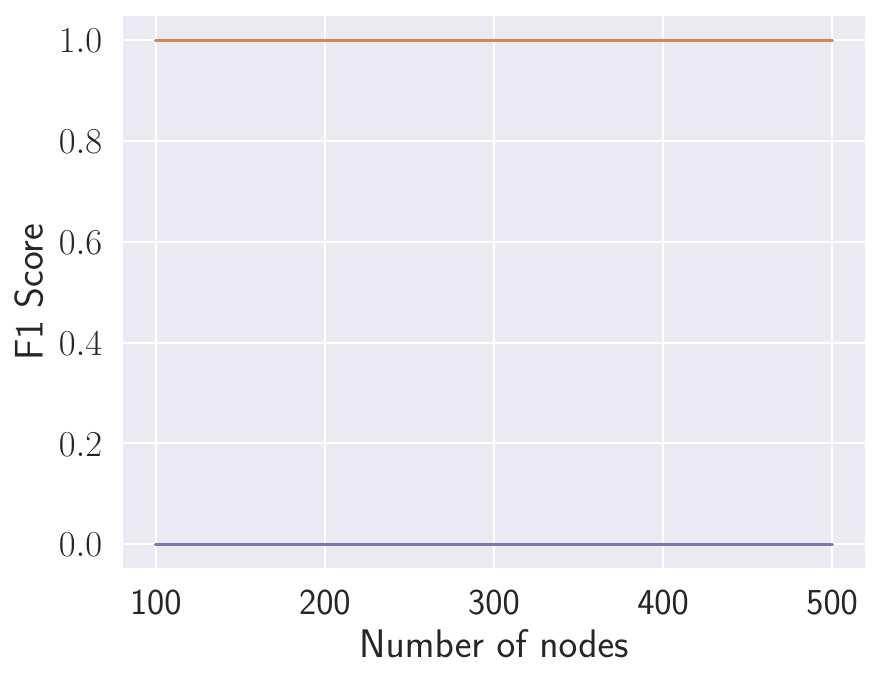}
            \includegraphics[width=\linewidth]{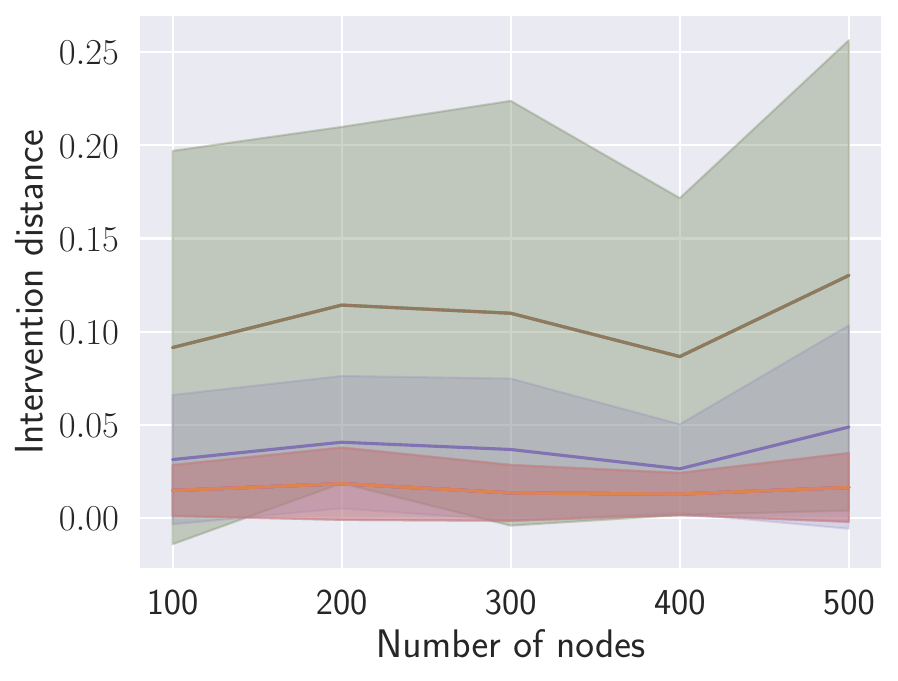}
        \end{subfigure}
        \begin{subfigure}[b]{0.32\linewidth}
            \centering
            \caption*{$\quad$ Fisher-Z tests}
            \includegraphics[width=\linewidth]{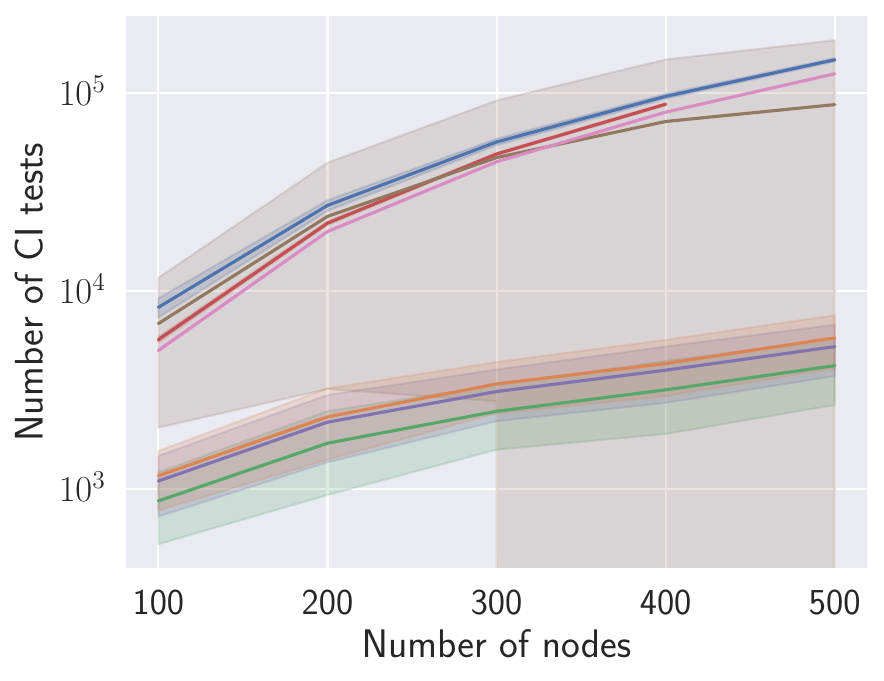}
            \includegraphics[width=\linewidth]{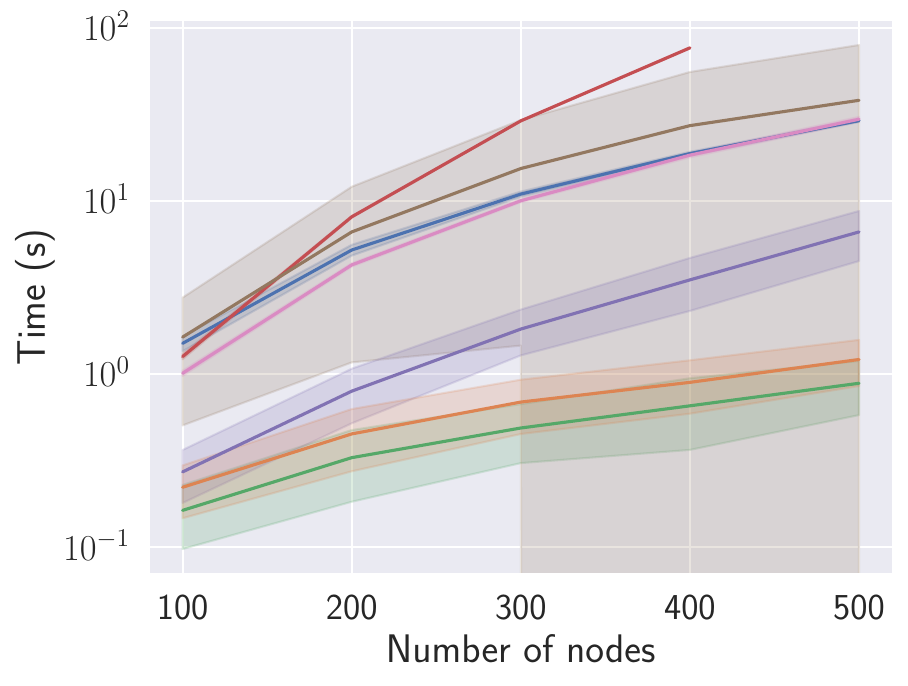}
            \includegraphics[width=\linewidth]{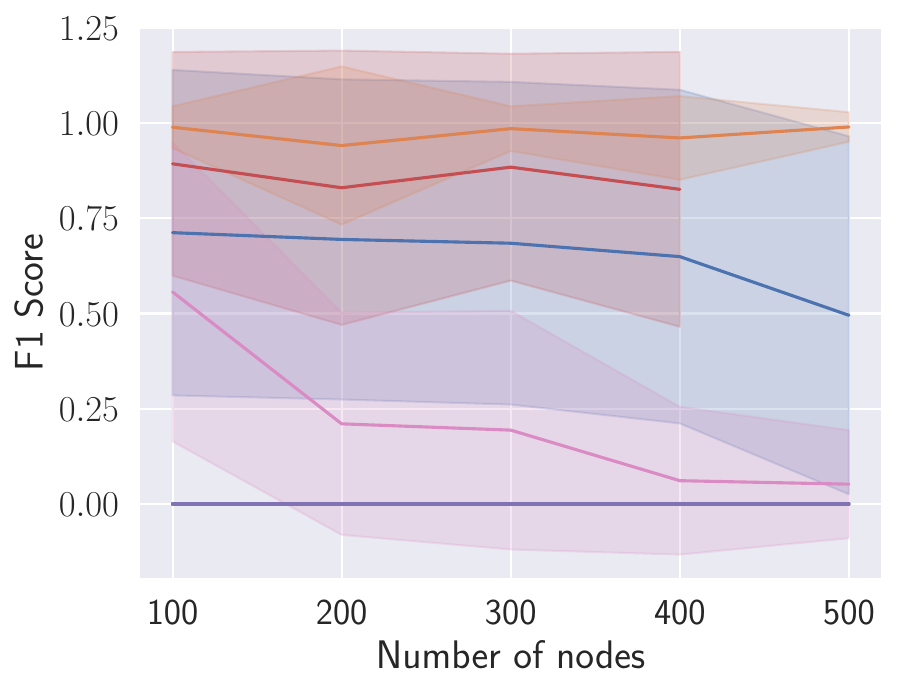}
            \includegraphics[width=\linewidth]{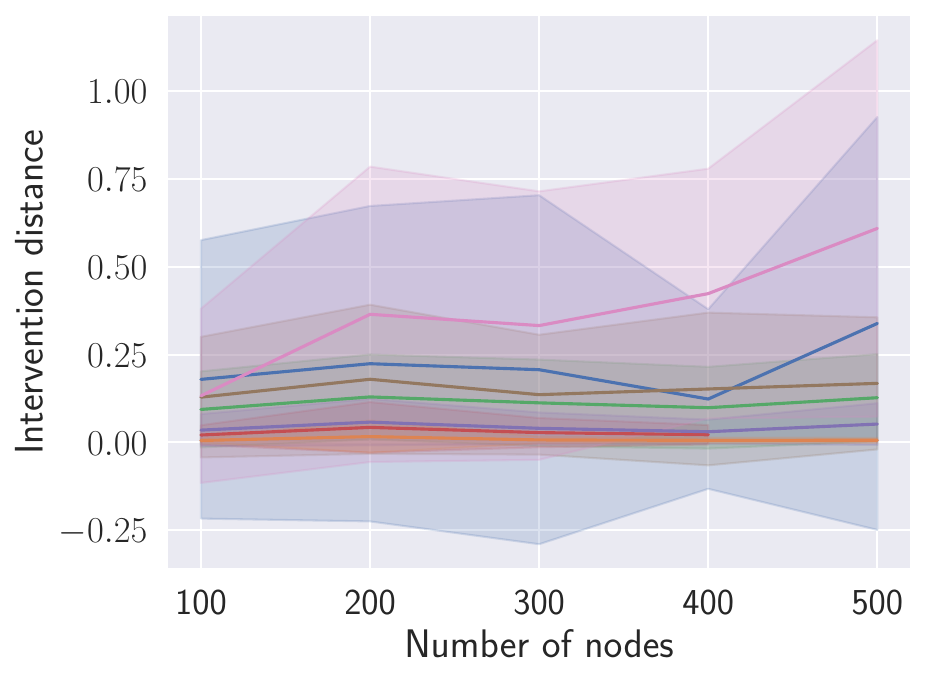}
        \end{subfigure}
        \begin{subfigure}[b]{0.32\linewidth}
            \centering
            \caption*{$G^2$ tests}
            \includegraphics[width=\linewidth]{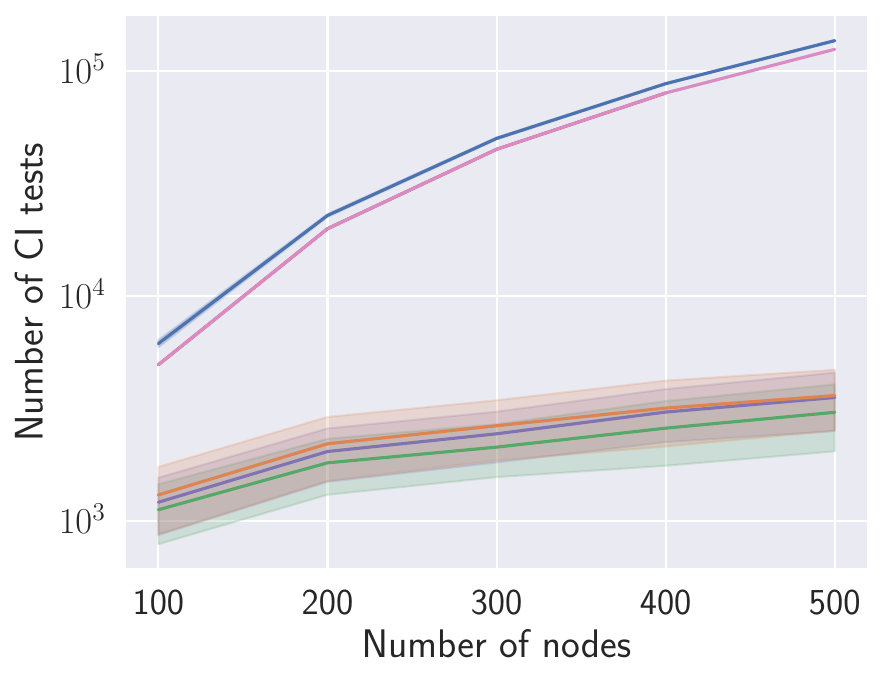}
            \includegraphics[width=\linewidth]{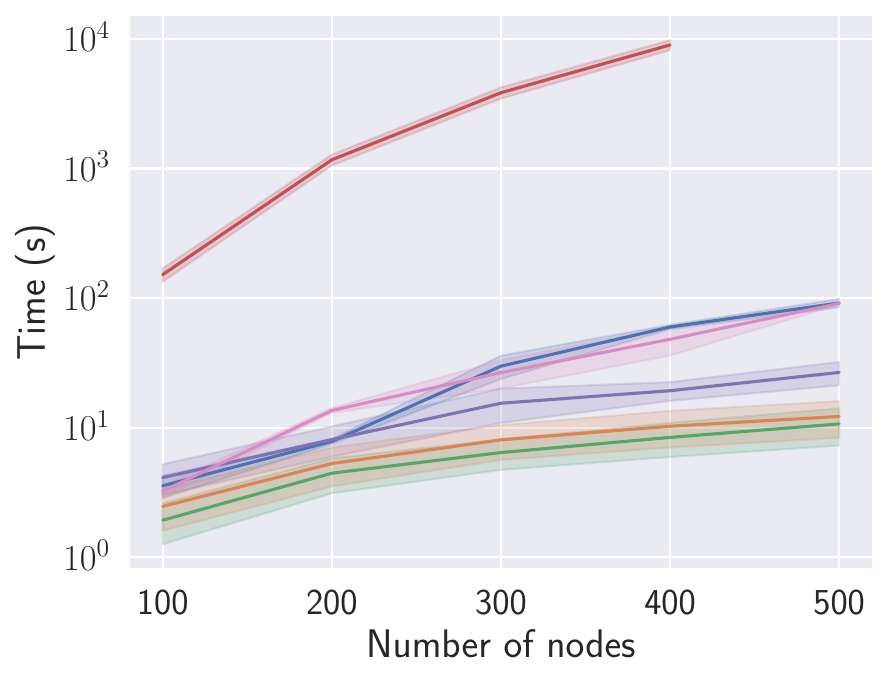}
            \includegraphics[width=\linewidth]{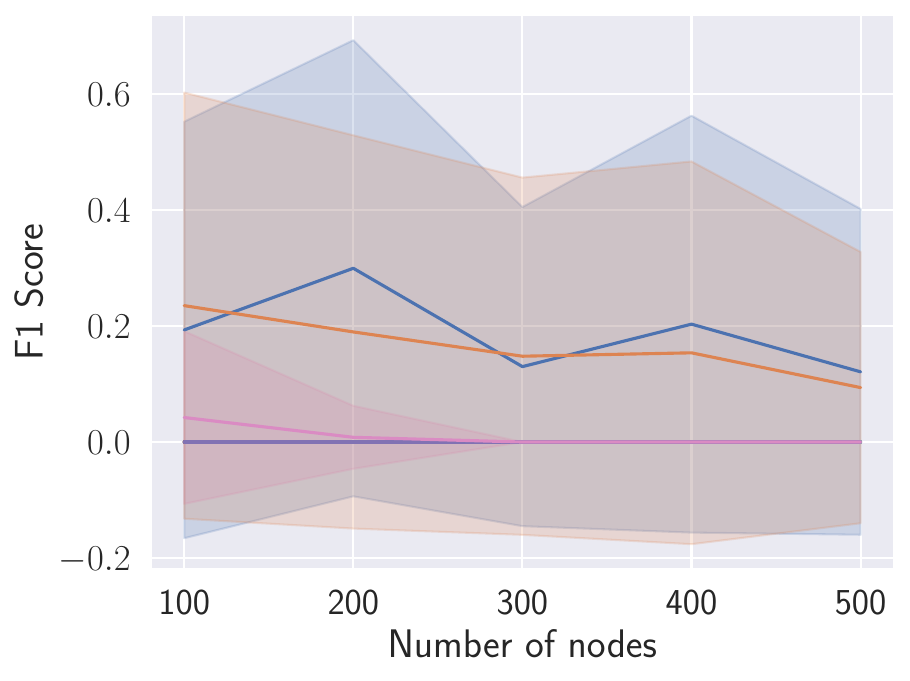}
            \includegraphics[width=\linewidth]{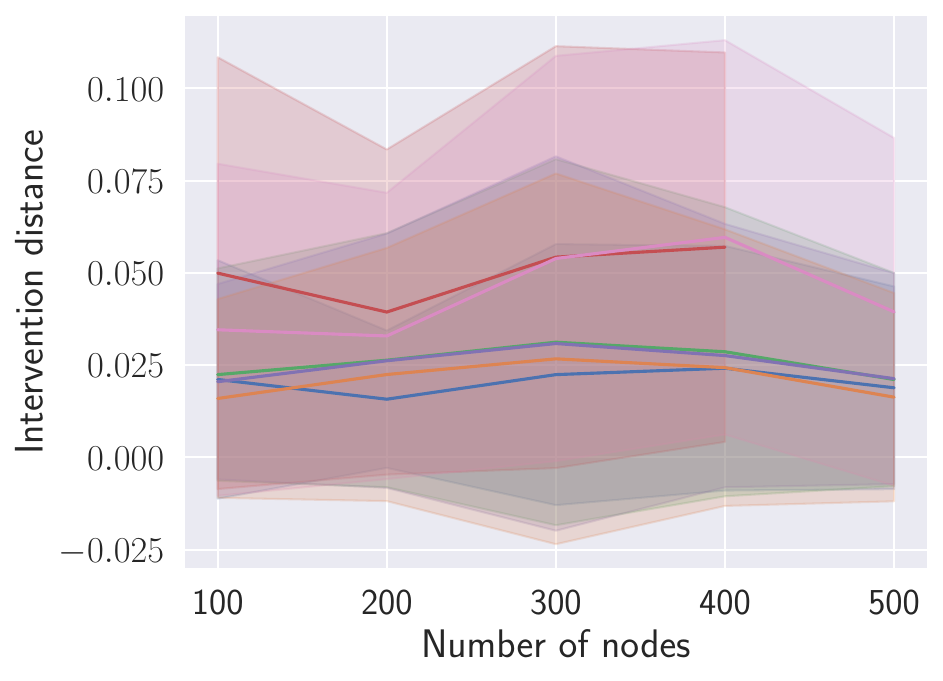}
        \end{subfigure}
    \end{subfigure}
    \caption{Results over number of nodes with $n_{\mathbf{D}} = 10000$, $\overline{d} = 2$ and $d_{\max} = 10$ with identifiable target pairs. The shadow area denotes the range of the standard deviation.}
    \label{fig:ident_results}
\end{figure*}

\subsection{Various Number of Samples}
\label{sec:samples}
In this section with experiment with a fixed number of nodes at $n_{\mathbf{V}} = 200$ and vary the number of synthetic data samples among $n_{\mathbf{D}} \in [500, 1000, 5000, 10000]$.
We maintain every other setting at expected degree of $\overline{d} = 2$ and maximum degree of $d_{\max} = 10$, and use a significance level of $\alpha = 0.01$ for all algorithms and CI tests.
We only evaluate the finite data cases of Fisher-Z CI tests on linear Gaussian data and $G^2$ CI tests on binary data, as oracle d-separation CI tests are not affected by the sample size.

We report our results in \Cref{fig:samples}.
The number of samples has no impact on the number of CI tests and computation time for Fisher-Z tests.
In the case of $G^2$ tests it slightly increases the number of CI tests and computation time of LOAD and other local methods, and drastically increases the computation time of MARVEL.

More samples help in recovering higher quality adjustment sets in both data domains, as LOAD achieves one of the best F1 scores and intervention distance over all samples.

\begin{figure*}[ht]
    \centering
    \begin{subfigure}[b]{\linewidth}
        \centering
        \includegraphics[width=.8\linewidth]{experiments/legend.pdf}
    \end{subfigure}
    \begin{subfigure}[b]{\linewidth}
        \centering
        \caption*{$\quad$ Fisher-Z tests}
        \includegraphics[width=.24\linewidth]{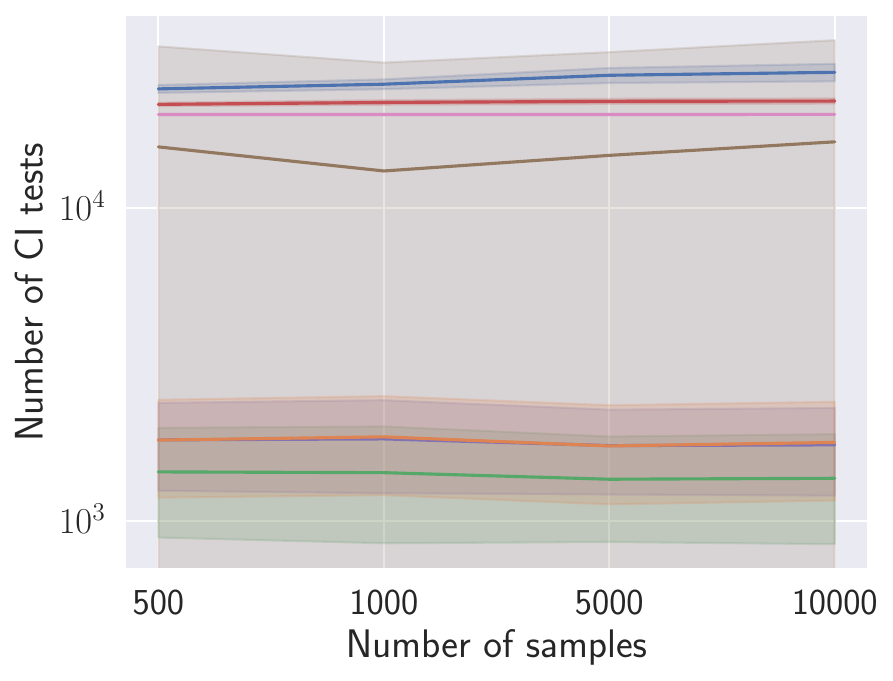}
        \includegraphics[width=.24\linewidth]{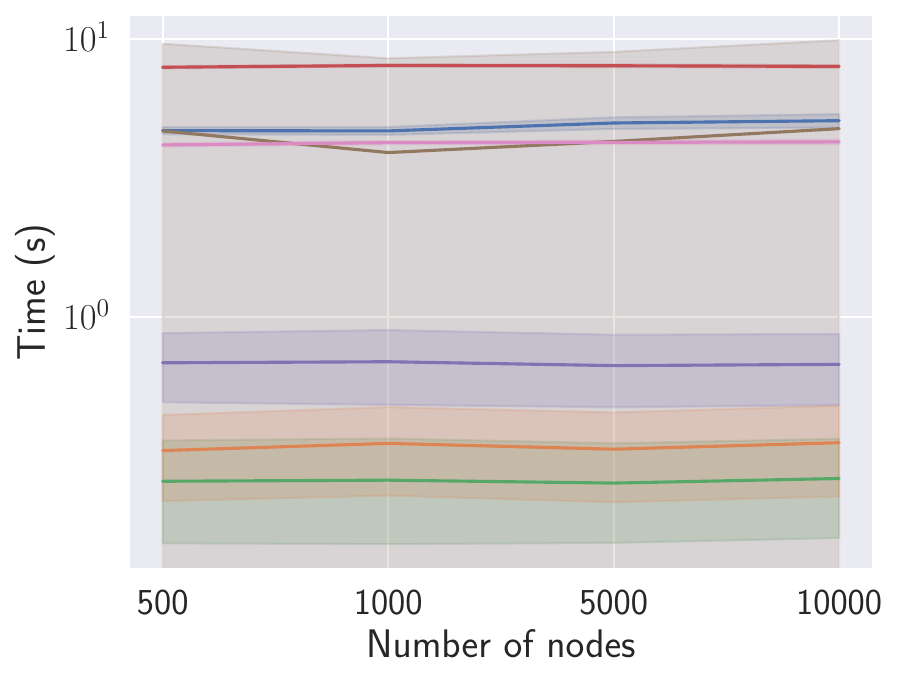}
        \includegraphics[width=.24\linewidth]{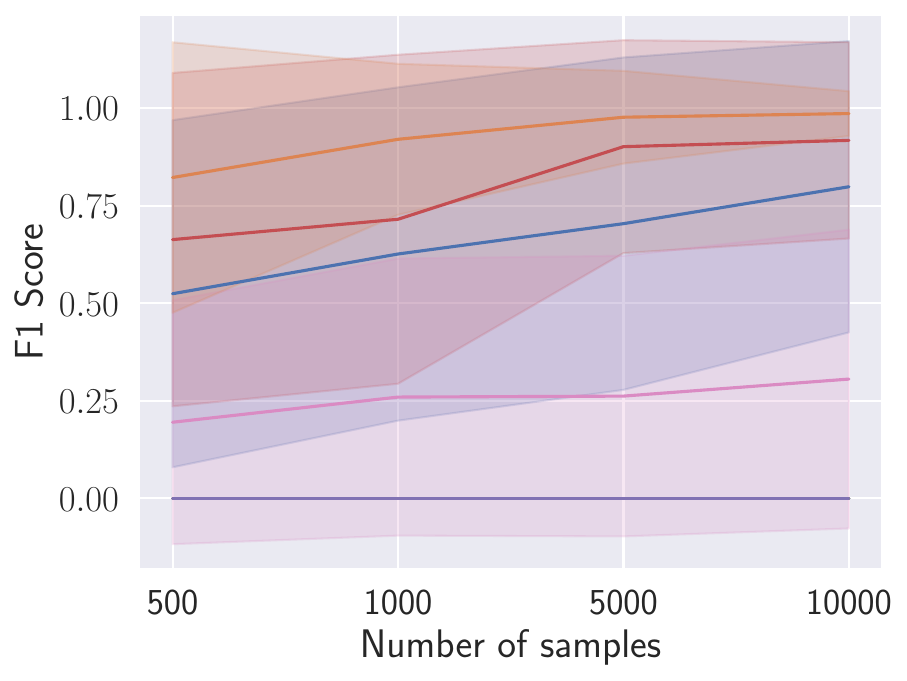}
        \includegraphics[width=.24\linewidth]{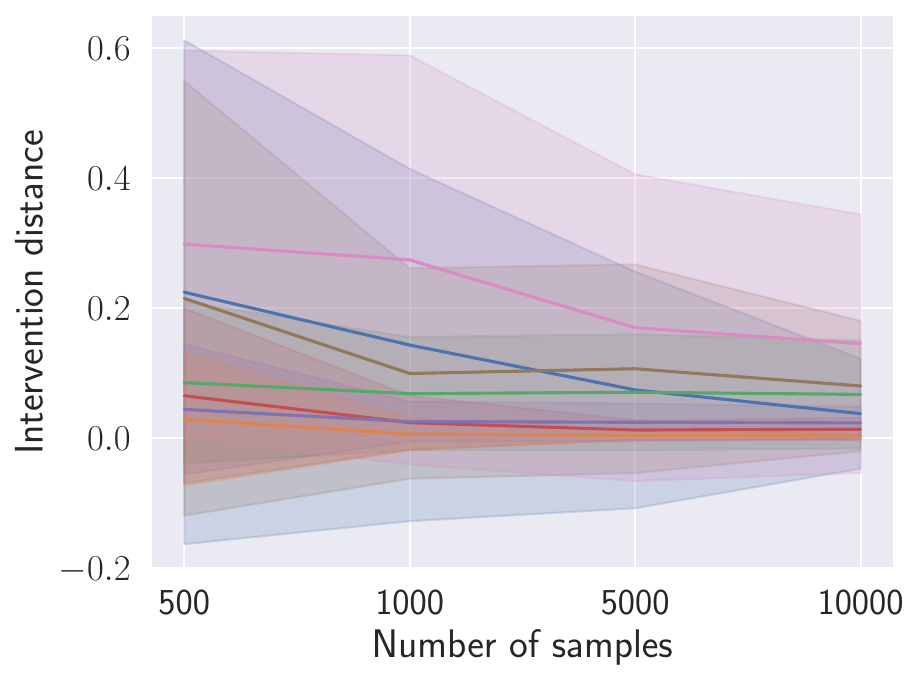}
    \end{subfigure}
    \begin{subfigure}[b]{\linewidth}
        \centering
        \caption*{$G^2$ tests}
        \includegraphics[width=.24\linewidth]{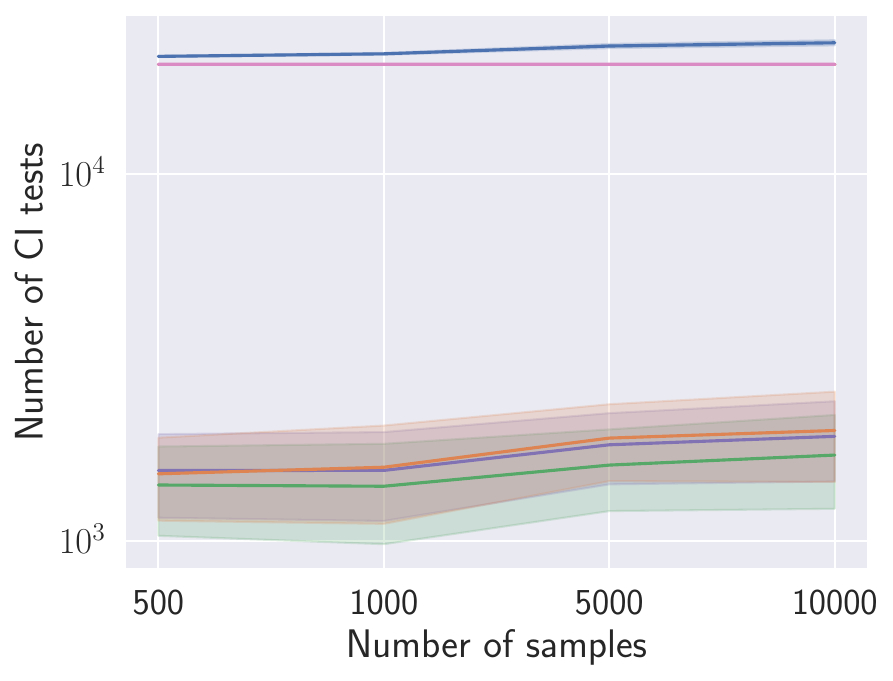}
        \includegraphics[width=.24\linewidth]{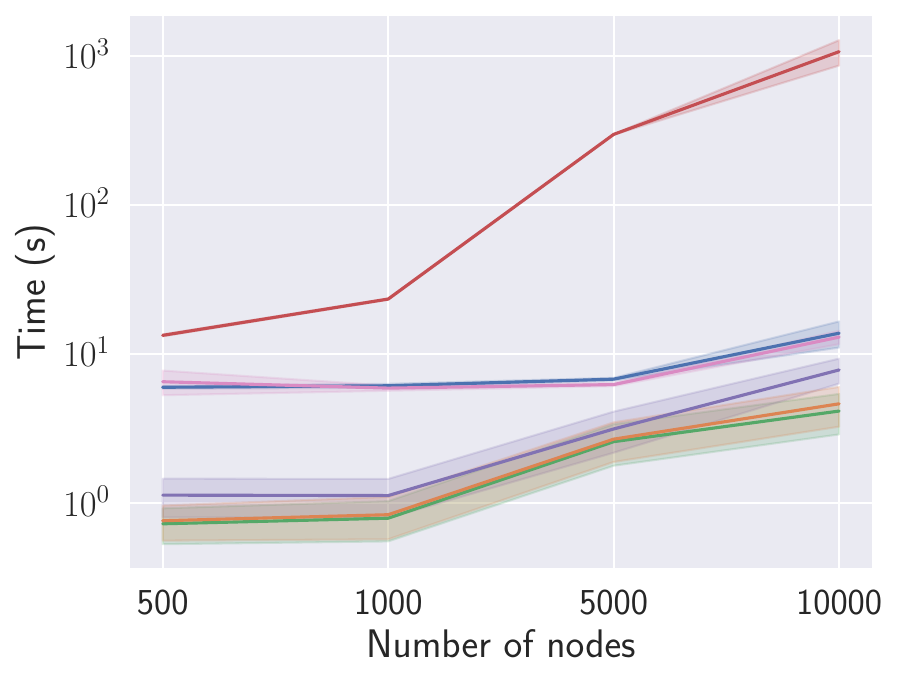}
        \includegraphics[width=.24\linewidth]{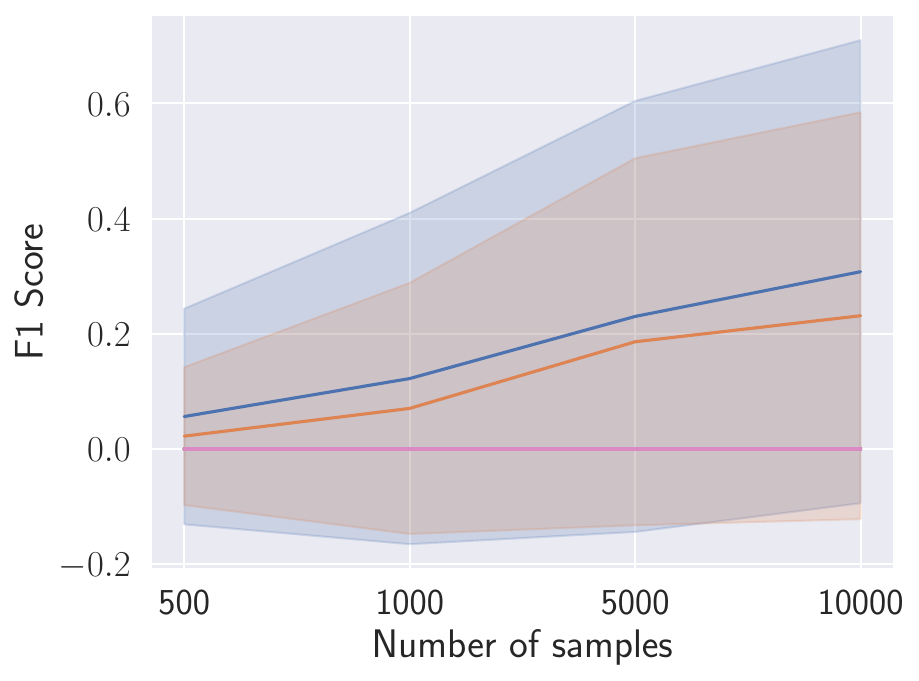}
        \includegraphics[width=.24\linewidth]{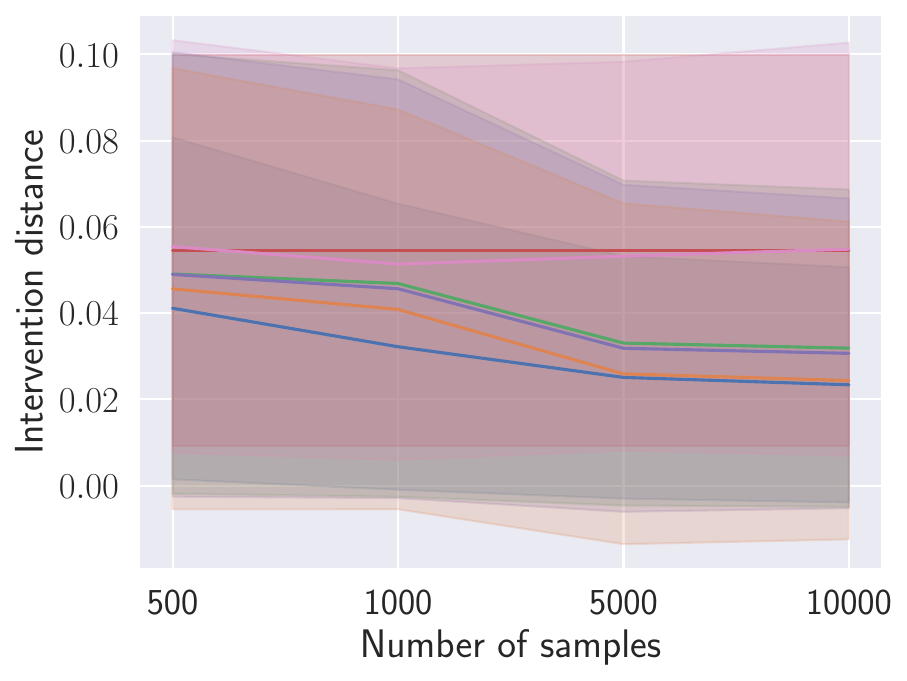}
    \end{subfigure}
    \caption{Results over number of samples with $n_{\mathbf{V}} = 200$, $\overline{d} = 2$ and $d_{\max} = 10$ and target pairs such that one is an explicit ancestor of the other. The shadow area denotes the range of the standard deviation.}
    \label{fig:samples}
\end{figure*}

\subsection{Various Expected Degrees}
\label{sec:degrees}

In \Cref{fig:degrees} we show results for various number of expected degrees $\overline{d} \in [2.0, 2.5, 3.0]$, and a fixed number of nodes at $n_{\mathbf{V}} = 200$,  maximum degree of $d_{\max} = 10$ and  $n_{\mathbf{D}} = 10000$.
We use a significance level of $\alpha = 0.01$ for all algorithms and CI tests.

\Cref{fig:degrees} shows that increasing the expected degree makes the problem setting much more difficult for all algorithms, as computational requirements generally increase, while the quality of the recovered adjustment sets become worse.
Nonetheless, the computational requirements of LOAD remain close to local methods, while still achieving one of the best F1 scores and intervention distances across data domains.

\begin{figure*}[ht!]
    \centering
    \begin{subfigure}[b]{\linewidth}
        \centering
        \includegraphics[width=.8\linewidth]{experiments/legend.pdf}
    \end{subfigure}
    \begin{subfigure}[b]{.8\linewidth}
        \begin{subfigure}[b]{0.32\linewidth}
            \centering
            \caption*{ $\quad$ d-separation tests}
            \includegraphics[width=\linewidth]{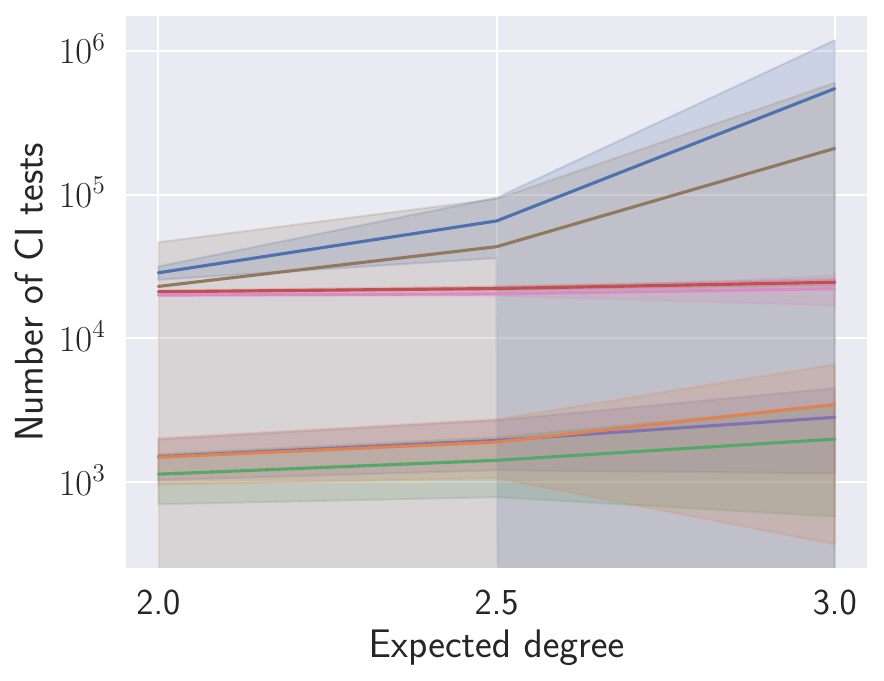}
            \includegraphics[width=\linewidth]{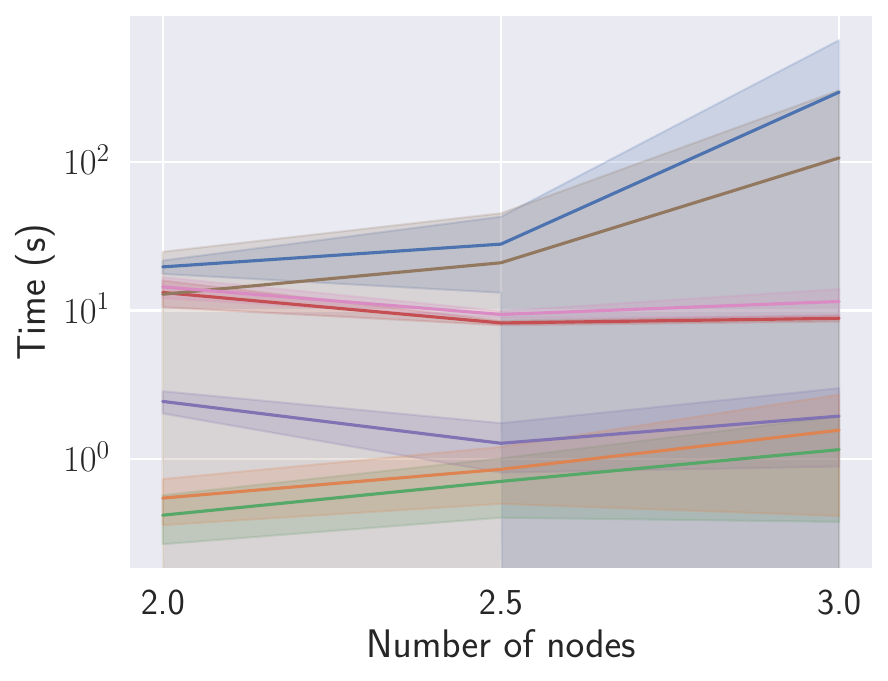}
            \includegraphics[width=\linewidth]{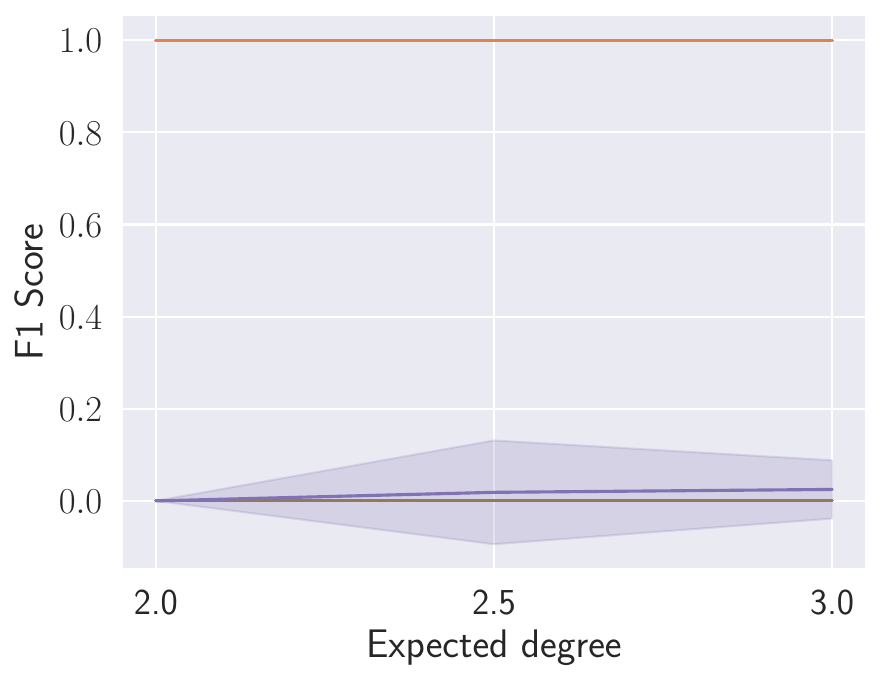}
            \includegraphics[width=\linewidth]{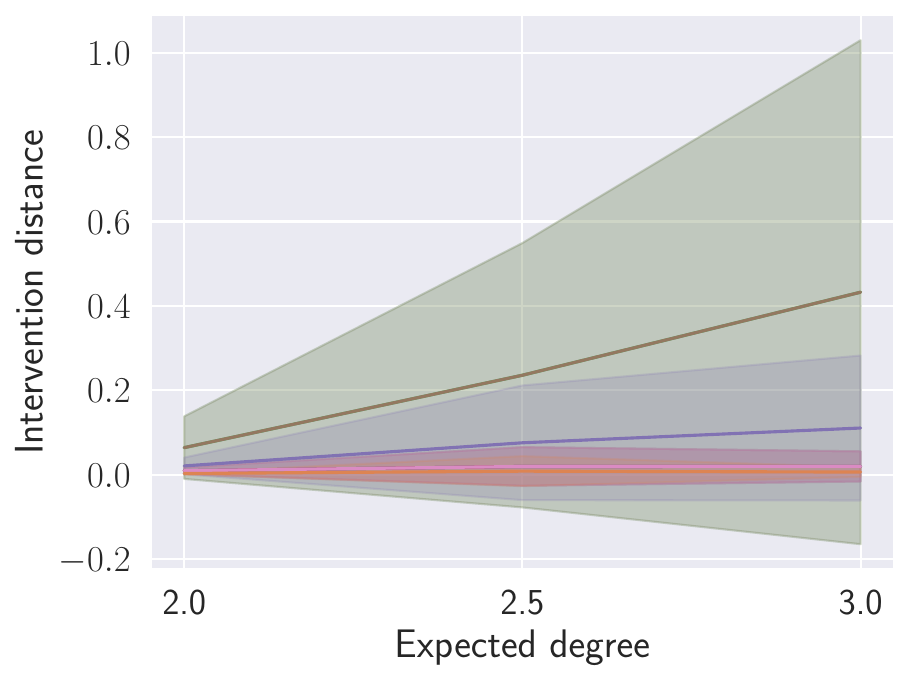}
        \end{subfigure}
        \begin{subfigure}[b]{0.32\linewidth}
            \centering
            \caption*{$\quad$ Fisher-Z tests}
            \includegraphics[width=\linewidth]{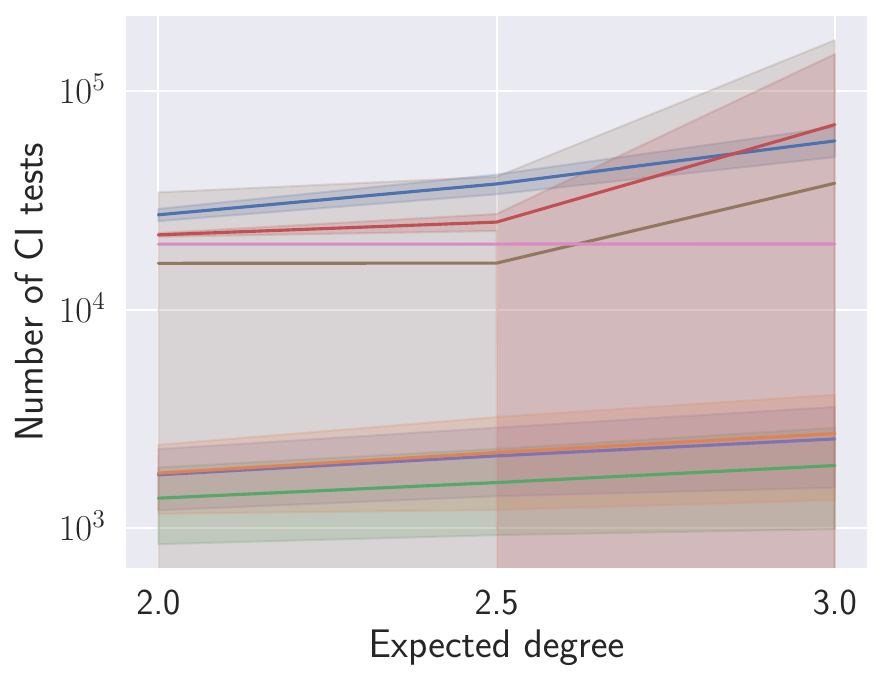}
            \includegraphics[width=\linewidth]{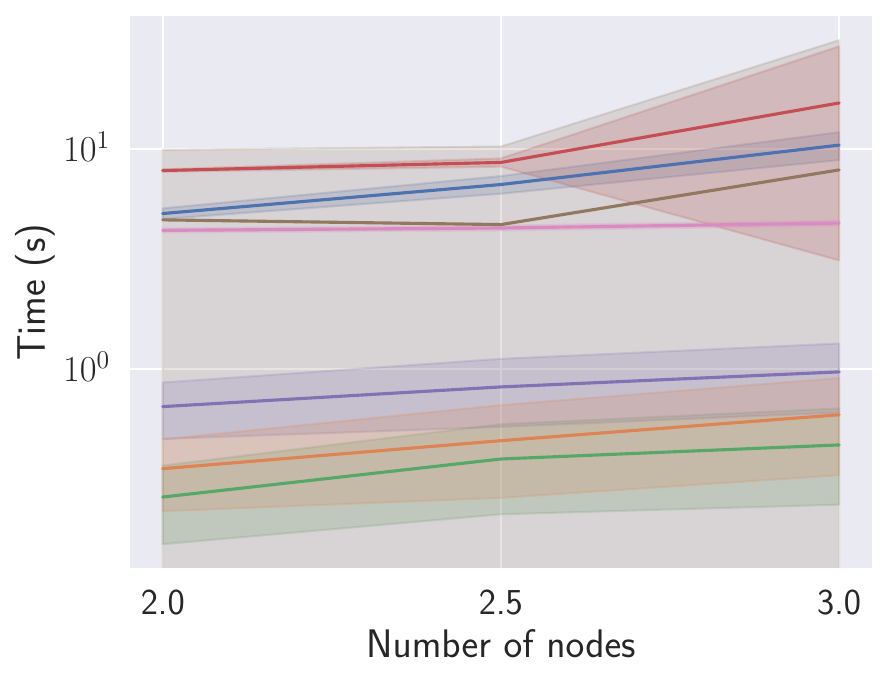}
            \includegraphics[width=\linewidth]{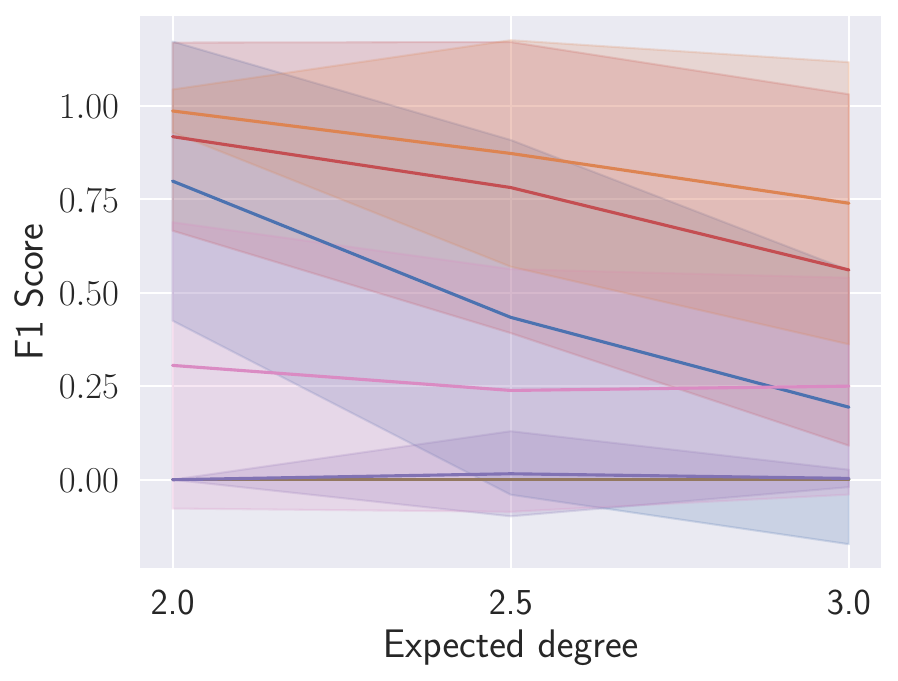}
            \includegraphics[width=\linewidth]{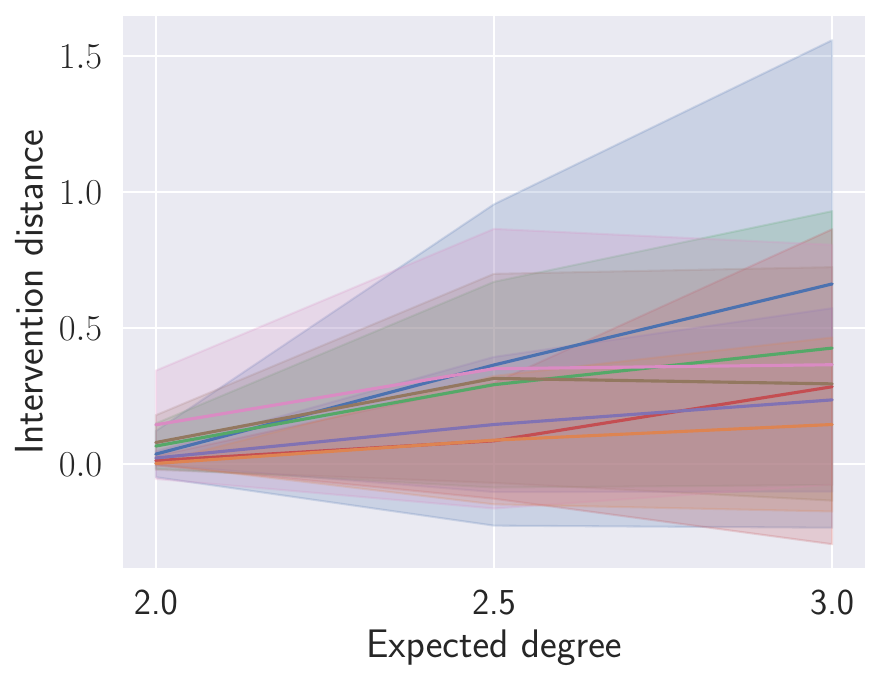}
        \end{subfigure}
        \begin{subfigure}[b]{0.32\linewidth}
            \centering
            \caption*{$G^2$ tests}
            \includegraphics[width=\linewidth]{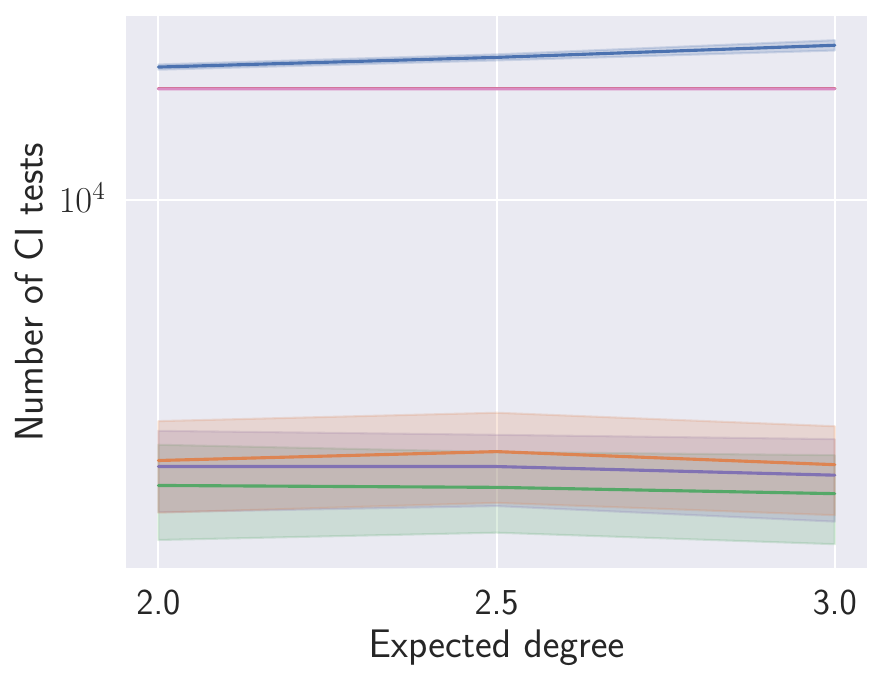}
            \includegraphics[width=\linewidth]{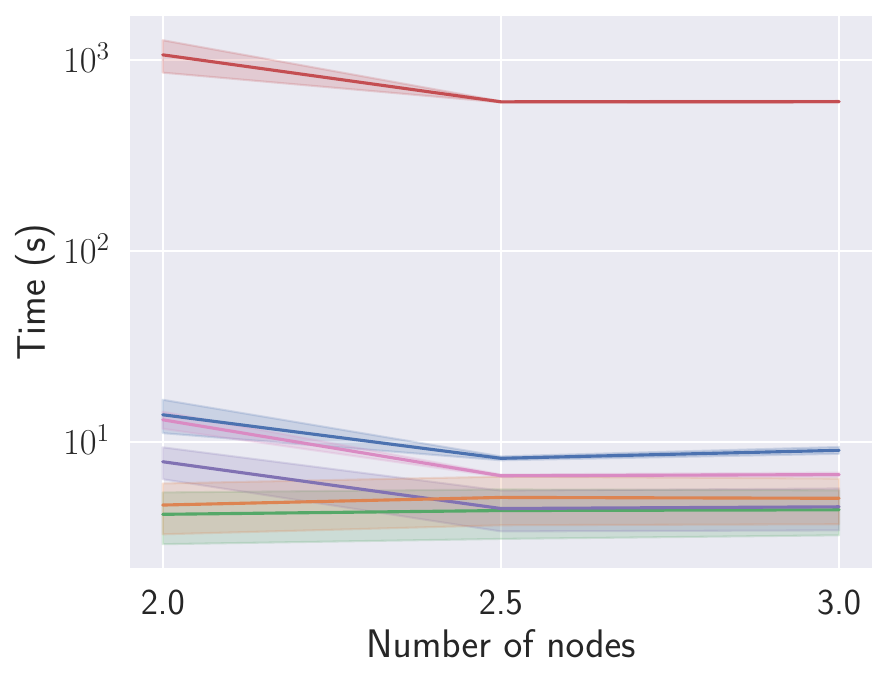}
            \includegraphics[width=\linewidth]{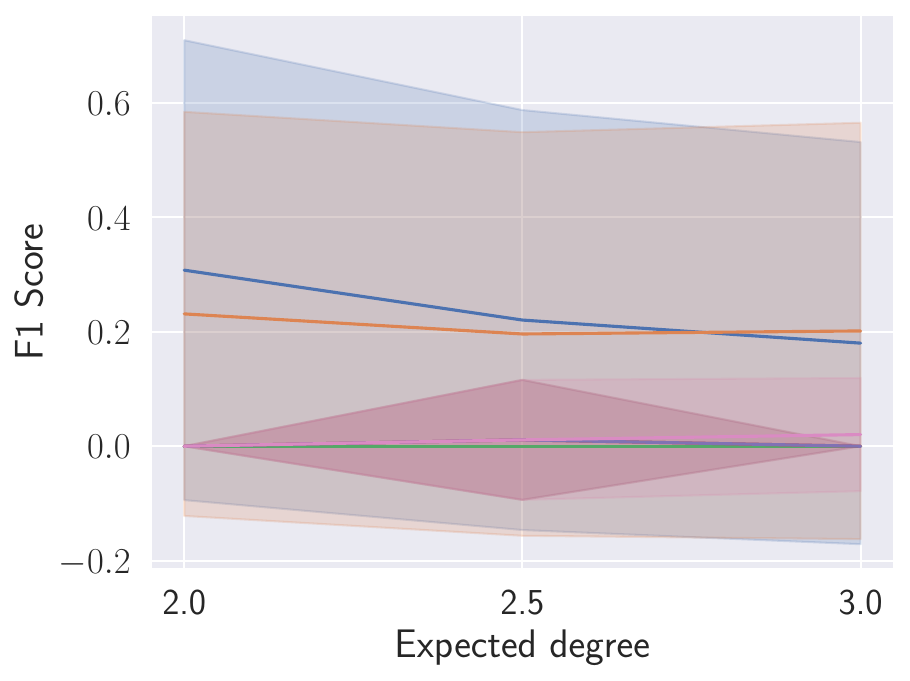}
            \includegraphics[width=\linewidth]{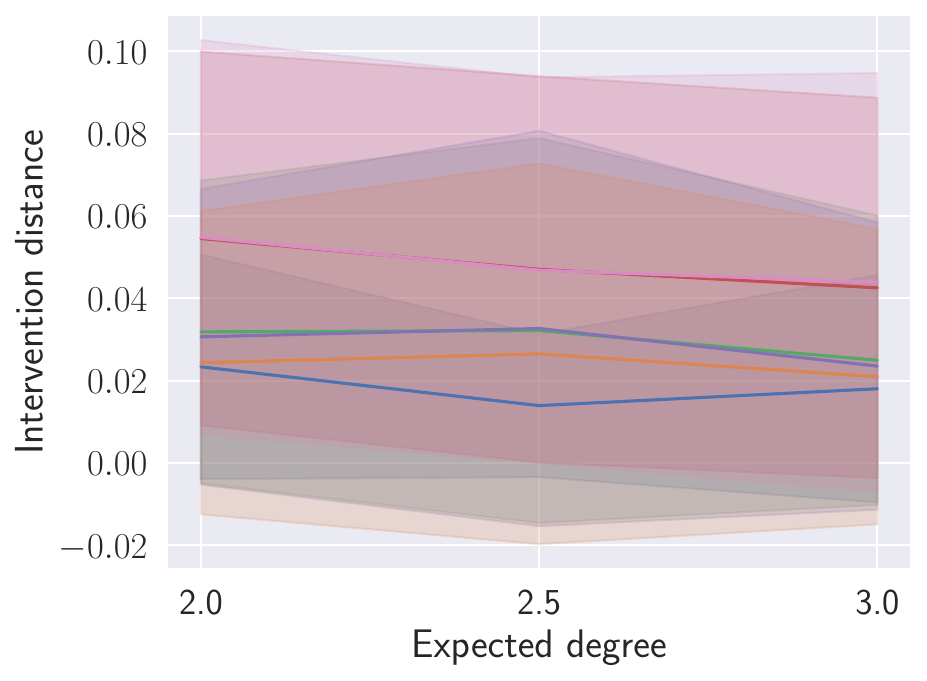}
        \end{subfigure}
    \end{subfigure}
    \caption{Results over number of expected degrees, with $n_{\mathbf{V}} = 200$, $n_{\mathbf{D}} = 10000$, and $d_{\max} = 10$ and target pairs such that one is an explicit ancestor of the other, for different settings in each column.
    The shadow area denotes the range of the standard deviation.}
    \label{fig:degrees}
\end{figure*}

\end{document}